\documentclass[final,3p,times,authoryear]{elsarticle}

\usepackage{booktabs}

\usepackage{amsmath,amssymb}
\usepackage{graphicx}
\usepackage{hyperref}
\usepackage{xcolor}
\usepackage{comment}
\usepackage{tikz}
\usepackage{subfig}
\usepackage{amsthm}
\usetikzlibrary{decorations.pathreplacing}
\usepackage{ulem}

\usepackage{array}
\newcolumntype{P}[1]{>{\centering\arraybackslash}p{#1}}

\theoremstyle{definition}
\newtheorem{definition}{Definition}[section]
\newcommand{\bx}{\mathbf{x}} 
\newcommand{\bob}{\mathbf{f}} 
\newcommand{\bigO}{\mathcal{O}}
\newcommand{\todo}[1]{\textcolor{red}{[~#1~]}}
\newcommand{\hide}[1]{}

\newcommand{\set}[1]{\left\{{#1}\right\}}

\providecommand{\R}{\ensuremath{\mathbb{R}}}
\newtheorem{corollary}{Corollary}[section]
\newtheorem{remark}{Remark}[section]

\usepackage{xcolor}
\let\svthefootnote\thefootnote
\newcommand\colorfootnote[2][black]{\def\thefootnote{\color{#1}\svthefootnote}%
  \footnote{\color{#1}#2}\def\thefootnote{\color{black}\svthefootnote}}
\newcommand{\RICHARD}[1]{\protect\colorfootnote[blue]{{\textbf{[RICHARD: #1]}}}}

\let\emph\relax
\DeclareTextFontCommand{\emph}{\bfseries}

\begin{document}

\begin{frontmatter}


\title{What if we Increase the Number of Objectives?\\
Theoretical and Empirical Implications for Many-objective Optimization}


\author[label1]{Richard Allmendinger}
\author[label2]{Andrzej Jaszkiewicz}
\author[label3]{Arnaud Liefooghe}
\author[label4]{Christiane Tammer}

\address[label1]{The University of Manchester, Alliance Manchester Business School, Booth Street West, Manchester M15 6PB, UK}
\address[label2]{Pozna\'{n} University of Technology, Institute of Computing Science, Piotrowo 2, 60-965 Pozna\'{n}, Poland}
\address[label3]{University of Lille, CNRS, Centrale Lille, Inria, UMR 9189 - CRIStAL, F-59000 Lille, France}
\address[label4]{Martin-Luther-Universit\"{a}t Halle-Wittenberg, Institut für Mathematik, 06099 Halle, Germany}

\begin{abstract}
The difficulty of solving a multi-objective optimization problem is impacted by the number of objectives to be optimized.
The presence of many objectives typically introduces a number of challenges that affect the choice/design of optimization algorithms. This paper investigates the drivers of these challenges from two angles: (i)~the influence of the number of objectives on problem characteristics and (ii)~the practical behavior of commonly used procedures and algorithms for coping with many objectives. In addition to reviewing various drivers, the paper makes theoretical contributions by quantifying some drivers and/or verifying these drivers empirically by carrying out experiments on multi-objective NK landscapes and other typical benchmarks. We then make use of our theoretical and empirical findings to derive practical recommendations to support algorithm design. Finally, we discuss remaining theoretical gaps and opportunities for future research in the area of multi- and many-objective optimization.

\medskip
\noindent
\textit{Contribution:} Survey, theory, empirical, implications on algorithm design. 
\end{abstract}


\end{frontmatter}

\section{Introduction}
Many real-world optimization problems require the decision maker (DM) to trade off two or more conflicting criteria (objectives), such as maximizing quality vs minimizing cost when buying a car~\citep{deb2001multi}, or maximizing potency vs minimizing side effects and manufacturing costs when developing a new drug~\citep{small2011efficient}. In practical applications, there can be significantly more than three objectives accounting for a mix of performance, resource, environmental and other goals; for example,~\cite{kollat2011many} optimize the position and frequency of tracer sampling of groundwater using six design objectives, ~\cite{eikelboom2015spatial} 
address land use planning problems using cost, benefit and spatial objectives for the individual parcels of land, and~\cite{fleming2005many} design optimal control systems for aircraft engines using eight objectives, each corresponding to a different design specifications. 

Multi-objective optimization (MO)~\citep{deb2001multi,miettinen2012nonlinear} is the area looking at the development and application of algorithms to problems with multiple conflicting objectives; problems with more than three objectives have also been termed as many-objective problems~\citep{kollat2011many} and are less studied. In the absence of any user preferences about desired ideal solutions, the goal of MO algorithms (MOAs) is not to identify a single optimal solution but to approximate the set of best trade-off solutions to a problem, also known as the Pareto (optimal) set. We say that a solution Pareto dominates, or simply dominates, another solution if it is not worse in any objective and if it is strictly better in at least one objective. Solutions in the Pareto set are non-dominated by any other solution from the feasible solution space.

The traditional approach to tackle a MO problem is to convert it into a single-objective problem using a scalarizing function~\citep{eich08}, and then solve the problem repeatedly using different `configurations' of the scalarizing function (e.g. different weights). A commonly used scalarizing method is to combine the objectives using a weighted sum, and then alter the weights to discover one 
solution per algorithmic run. However, in the non-convex case, non-linear scalarization methods are more useful. 
An alternative approach is to use a population-based approach, such as a multi-objective evolutionary algorithm (MOEA), to evolve several solutions (a population) in one algorithmic run with the hope to also discover multiple Pareto optimal solutions in one go. MOEAs have proven to be an efficient approach, especially for problems with two and three objectives~\citep{coello2007evolutionary,deb2001multi}. Traditional MOEAs rely on the Pareto dominance concept combined with diversity maintenance mechanisms to drive the search (e.g. NSGA-II~\citep{deb2002fast}). Two further prominent concepts are indicator-based and decomposition-based methods. The former concept replaces the objectives with a unary set performance metric (indicator) and then optimizes this metric (e.g. IBEA~\citep{zitzler2004indicator}). A commonly used indicator is the hypervolume indicator~\citep{Zitzler2003}, which measures the volume of the objective space that is dominated by the image of the Pareto set (also known as the Pareto front) and bounded by a reference point. Decomposition-based approaches (e.g. MOEA/D~\citep{zhang2007moea}) decompose a MO problem into several single-objective problems using a scalarizing function with different weights. Each solution in the population is then dedicated to optimize one scalarizing function. 

From the literature, it is apparent that solving MO problems is known to be difficult for some problem classes; moreover, we know from the No Free Lunch theorem~\citep{wolpert1997no} that there is no single best approach in the general case. The computational difficulty grows with an increase in the number of objectives affecting different problem characteristics, such as the number of Pareto optimal solutions, distances among solutions, and likelihood of objectives varying in complexity and evaluation times. Furthermore, it is anticipated that algorithms and procedures designed for MO problems will face difficulties as the number of objectives increases. For example, the complexity for routines such as dominance tests and updating of the Pareto archive affects Pareto dominance-based MOEAs, creating evenly distributed weight vectors impacts decomposition-based MOEAs and scalarizing methods, and computing and approximating the hypervolume can become an issue for indicator-based MOEAs and performance validation. An additional issue accompanied by an increase in the number of objectives is the limited availability of visualisable test problems. An overview of recent developments in the area of many-objective optimization can be found, for example, in~\citet{ishibuchi2008evolutionary,aguirre_many_survey,chand2015evolutionary}, and empirical work on the efficiency of MOEAs for many-objective optimization problems can be found, for example, in~\citet{purshouse2003evolutionary,hughes2005evolutionary,wagner2007pareto}.


We make several contributions 
to support the community in gaining a better understanding about the effect of increasing the number of objectives on problem characteristics and the complexity of MO procedures and algorithms:

\begin{itemize}
    \item We adopt a holistic approach linking theory with empirical analysis, and then we highlight how our results translate into practice. 
    
    \item We theoretically investigate the key drivers attributing to an increase in computational and algorithmic challenges. In particular, we derive probabilities for a solution to be non-dominated, and propose a general formulation for  scalarizing functions.  
    
    \item We conduct empirical experiments on multi-objective \mbox{NK-landscapes}~\citep{aguirre2007,verel2013} and other benchmarks to back up our theoretical contributions and verify various published theoretical results (as summarized in Table~\ref{contribution_table}).

    \item We derive practical recommendations from our analysis to support algorithm design. 
\end{itemize}
%
Let us highlight that our focus is on understanding the impact of the number of objectives on different problem characteristics and MO procedures. The review and evaluation of individual MOEAs is out of scope of this paper, and can be found, for example, in~\citet{ishibuchi2008evolutionary,chand2015evolutionary}. However, as pointed out above, we will use our theoretical and empirical findings to provide recommendations for MOEAs of different type.

This paper is organized as follows. The next section defines key MO concepts and motivates the choice of the main test function used for empirical evaluation. 
Section~\ref{problemChar} investigates the influence of the number of objectives on different problem characteristics, while Section~\ref{MOEAprocedures} investigates the complexity of commonly used procedures and algorithms for coping with multiple objectives. Section~\ref{conclusions} presents a summary of our theoretical and empirical findings, 
and then uses these to make recommendations for algorithmic setup choices 
for MOEAs, before finally  discussing directions for future research. 

\begin{table}[t]
\caption{Overview the topics to which this paper is making either a theoretical or empirical contribution, and the section, equation(s) and/or figure(s) presenting the actual contribution. Sections not listed in the table perform a mini review and are included in the paper to put the work in context.}%
    \label{contribution_table}%
    \small%
    \centering%
    \begin{tabular}{llll}
    \toprule
    \textbf{Topic} & \textbf{Presented in} & \textbf{Theoretical contribution} & \textbf{Empirical contribution} \\
    \midrule
    Number of Pareto optimal solutions ~~~~~~~~~~~~~~~~~~~~~~~~~~ & Section~\ref{sec:pos} &  & Fig.~\ref{fig:pos}  \\ 
    Probability for a solution to be non-dominated & Section~\ref{sec:dom_rel} & Eq.~(\ref{eq:nd_pairs}) & Figs.~\ref{fig:prob_nd_model}-\ref{fig:nb_nd_nk} \\
    Probability of heterogeneous objectives &  Section~\ref{HeterogObj} & & Fig.~\ref{fig:HetOb} \\
    Distances between solutions & Section~\ref{distanceBetweenSolutions} & & Fig.~\ref{fig:dist_sol} \\
    Dominance test and updating the Pareto archive & Section~\ref{sec:DomTestUpdate} & & Figs.~\ref{fig:NDAll}-\ref{fig:NDvsPoints}\\
    Computing and approximating hypervolume & Section~\ref{sec:hv} & & Figs.~\ref{fig:Exact}-\ref{fig:MC}\\
    Impact on scalarization methods & Section~\ref{s-furtherscal} & Eq.~(\ref{f-spec-scal-block-new}) & \\
    Distances between weight vectors & Section~\ref{sec:dist_weights} & & Fig.~\ref{fig:dist_weight} \\
    \bottomrule
    \end{tabular}
\end{table}

\section{Background}\label{background}
This section provides the relevant formal definitions to understand and avoid ambiguity of the various MO concepts, which will be used in the subsequent sections. The section will also provide details of multi-objective \mbox{NK-landscapes}, the model we use to empirically validate our theoretical contributions and/or existing theoretical results.  

\subsection{Basic Definitions}

\smallskip

\begin{definition}{(Multi-objective optimization (MO) problem)}\label{d-MO}
The general formulation of a MO problem is to ``maximize'' an $m$-dimensional objective function vector: $\mathbf{\bob}(\mathbf{\bx}) = (f_{1}(\mathbf{\bx}),\ldots,f_{i}(\mathbf{\bx}),\ldots,f_{m}(\mathbf{\bx}))$, where each objective depends upon a vector $\mathbf{\bx} = (x_{1},\ldots,x_j,\ldots,x_{n})$ of
$n$ design (or decision) variables; we refer to $\mathbf{\bx}$ also as the candidate solution vector or simply as a solution. A MO problem may also include equality and inequality constraints, and these constraints define a feasible design space, $X$ (which can be, for example, real-valued, binary, or of combinatorial nature). The objective space image of $X$ is denoted $Y$, and termed the feasible objective space. The term ``maximize'' is written in quotes to indicate that there are no unique maxima to a MO problem in general, and a further definition is needed to define an ordering on candidate solutions (see below).
\end{definition}

\begin{definition}{(Pareto dominance)}
Consider two solutions $\bx^1\in X$ and $\bx^2\in X$. We say that $\bx^1$ is dominated by~$\bx^2$, also written as $\bx^1 \prec \bx^2$, if and only if 	$\exists i$ such that $f_{i}(\bx^1) < f_{i}(\bx^2)$ and $\forall j, f_{j}(\bx^1) \leq f_{j}(\bx^2)$. This relation is also sometimes called strict Pareto dominance in contrast to the weak Pareto dominance defined below.
\end{definition}

\begin{definition}{(Weak Pareto dominance)}
Consider two solutions $\bx^1\in X$ and $\bx^2\in X$. We say that $\bx^1$ is weakly dominated by~$\bx^2$, also written as $\bx^1 \preceq \bx^2$, if and only if $\forall j, f_{j}(\bx^1) \leq f_{j}(\bx^2)$.
\end{definition}

\begin{definition}{(Pareto optimal)}\label{d-MOmax}
A solution $\bx^1\in X$ is called Pareto optimal if there does not exist a solution $\bx^2\in X$ that dominates it. 
\end{definition}

\begin{remark}In Section \ref{s-furtherscal}, where we will discuss a general scalarization technique for a characterization of solutions to MO problems, we will assume a more general definition of Pareto optimality: there we will consider Pareto (weakly) optimal solutions defined by a nontrivial, closed, convex and pointed cone $C\subset \mathbb{R}^{m }$ (see Definition \ref{d-fCmax}).
\end{remark}

\begin{definition}{(Pareto set)}
The set of all Pareto optimal solutions is said to form the Pareto set.
\end{definition}

\begin{definition}{(Pareto front)}
The image of the Pareto set in the objective space is known as the Pareto front. 
\end{definition}

\begin{definition}\label{d-fCmax}
Consider a MO problem where the objective function $\mathbf{\bob}$ is to ''maximize'' with respect to a nontrivial, closed, convex and pointed cone $C\subset \mathbb{R}^{m }$. A solution $\bx^1 \in X$ is called a Pareto optimal solution with respect to $\mathbf{\bob}$ and $C$ if
\[ \mathbf{\bob}(\bx^1) \in \{ \mathbf{\bob}(\bx)\in Y \ : \ Y \bigcap (\mathbf{\bob}(\bx) + (C \setminus \{0\})) = \emptyset \}.\]
Furthermore, $\bx^1 \in X$ is called a Pareto weakly optimal solution with respect to $C$ if
\[ \mathbf{\bob}(\bx^1) \in \{ \mathbf{\bob}(\bx)\in Y\ : \ Y \bigcap (\mathbf{\bob}(\bx) + \operatorname*{int} C ) = \emptyset \},\]
where $\operatorname*{int} A$ denotes the interior of a set $A \subset \mathbb{R}^m$.
\end{definition}
{\begin{remark}
Of course, for the special case $C= \mathbb{R}^m_+ := \{y \in \mathbb{R}^m \mid \forall i\in \{1, \ldots , m\}: \; y_i \geq 0 \}$, a Pareto optimal solution in the sense of Definition \ref{d-fCmax} coincides with a Pareto optimal solution in the sense of Definition \ref{d-MOmax}.
\end{remark}}

\begin{remark}
Pareto (weakly) optimal solutions of a multi-objective optimization problem where the objective function $\mathbf{\bob}$ is to ''minimize'' are analogously defined, i.e., we replace $C$ by $-C$ and $\operatorname*{int} C $ by $-\operatorname*{int} C $ in Definition \ref{d-fCmax}.
\end{remark}

{\begin{remark}\label{r-maxmin}
Obviously, $\bx^1$  is a Pareto optimal (maximal) solution with respect to $\mathbf{\bob}$ and $C$ in the sense of Definition \ref{d-fCmax} if and only if $\bx^1$ is a Pareto optimal (minimal) solution with respect to $-\mathbf{\bob}$ and $-C$.
\end{remark}}

\begin{definition}{(Hypervolume indicator)}
Given a set of points in the objective space $S \subset \mathbb{R}^m$ and a reference point $r \in \mathbb{R}^m$,
the hypervolume indicator of S is the measure of the region weakly dominated by $S$ and weakly dominating $r$ , i.e.:
\begin{equation*}
    H(S,r) = \Lambda({q \in \mathbb{R}^m \:|\: r \preceq q \; \land \; \exists p \in S : q \preceq p })    
\end{equation*}
where $\Lambda(\text{.})$ denotes the Lebesgue measure. Alternatively, it may be interpreted as the measure of the union of boxes:
\begin{equation*}
    H(S) = \Lambda \Big(\bigcup_{p \in S, r \preceq p}[r,p]\Big)
\end{equation*}
where $[r, p] = \{q \in \mathbb{R}^m \:|\: q \preceq p \;  \land \; r \preceq q\}$ denotes the hypercuboid delimited by $p$ and $r$. The advantage of hypervolume indicator is its compliance with the Pareto dominance relation \citep{Zitzler2003}.
\end{definition}

\begin{definition}{(Pareto archive)}
A Pareto archive is a set of mutually non-dominated solutions. 
In the context of multi-objective 
algorithms, a Pareto archive is used to store potentially Pareto optimal solutions, i.e. solutions that are not dominated by any solution generated so far.
The Pareto archive may be unbounded, or bounded in cardinality, i.e. it may contain only a limited number of $N$ solutions.
\end{definition}


\subsection{Multi-objective NK-Landscapes}
\label{sec:rmnk}

\newcommand{\N}[0]{\textit{n}}
\newcommand{\K}[0]{\textit{k}}
\newcommand{\M}[0]{m}

Multi-objective \mbox{NK-landscapes}~\citep{aguirre2007,verel2013} 
are a  problem-inde\-pen\-dent model used for constructing multi-objective multi-modal combinatorial problems. 
They extend single-objective NK-land\-scapes~\citep{kauffman1993}.
Candidate solutions are binary strings of size~\N. 
The objective function vector $\mathbf{\bob}$ is defined as $\mathbf{\bob}\colon \lbrace 0, 1 \rbrace^{\N} \mapsto [0,1]^\M$
such that each objective $f_i$ is to be maximized.
As in the single-objective case, the objective value $f_i(\bx)$ of a solution~$\bx$
is the average value of the individual contributions associated with each design variable~$x_j$.
Given objective~$f_i$, $i \in \set{1, \ldots, \M}$, and each variable~$x_j$, $j \in \set{1,\ldots,\N}$,
a component function $f_{ij}\colon \lbrace 0, 1 \rbrace^{\K+1} \mapsto [0,1]$ assigns a real-valued contribution for every combination
of~$x_j$ and its \K~\textit{epistatic interactions} $\set{x_{j_1}, \ldots, x_{j_\K}}$. 
These $f_{ij}$-values are uniformly distributed in~$[0,1]$.
Thus, the individual contribution of a variable~$x_j$ depends on its value and on the values of $\K < \N$ variables~$\set{x_{j_1}, \ldots, x_{j_\K}}$ other than $x_j$.
The problem can be formalized as follows:
\begin{equation*}
\begin{array}{rll}
\label{eq:nk}
\max			&	\displaystyle f_i(x) = \frac{1}{\N} \sum_{j=1}^{\N} f_{ij}(x_j, x_{j_1}, \ldots, x_{j_\K})	&	i \in \set{1, \ldots, \M}	\\
\mbox{s.t.}	&	\displaystyle x_j \in \{0,1\}											&	j \in \set{1, \ldots, \N}
\end{array}
\end{equation*}
The epistatic interactions, i.e.~the \K\ variables that influence the contribution of~$x_j$, are typically set uniformly at random among the $(\N-1)$ variables other than~$x_j$,
following the random neighborhood model from~\citep{kauffman1993}.
By increasing the number of epistatic interactions~$\K$ from $0$ to~$(\N-1)$, problem instances can be gradually tuned from smooth to rugged.
%
Interestingly, multi-objective \mbox{NK-landscapes} exhibit different characteristics and different degrees of difficulty for multi-objective optimization methods%
\footnote{The source code of the multi-objective \mbox{NK-landscapes} generator is available at the following URL: \url{http://mocobench.sf.net}.}%
~\citep{daolio2017,8832171}. In the following two sections, if not otherwise stated we have considered 30 randomly-generated instances of multi-objective \mbox{NK-landscapes} with $n=10$ decision variables, and $k=0$ (i.e. linear problems),
but, of course, varied the number of objectives $m$ as the impact of $m$ is the focus of this study.



\section{Effect of the Number of Objectives on Problem Characteristics}\label{problemChar}

In this section, we study the influence of the number of objectives on different problem characteristics as outlined in Table~\ref{contribution_table}.

\subsection{Number of Pareto Optimal Solutions (Combinatorial~Case)}
\label{sec:pos}

\begin{figure}[!t]
\centering%
\includegraphics[height=170pt]{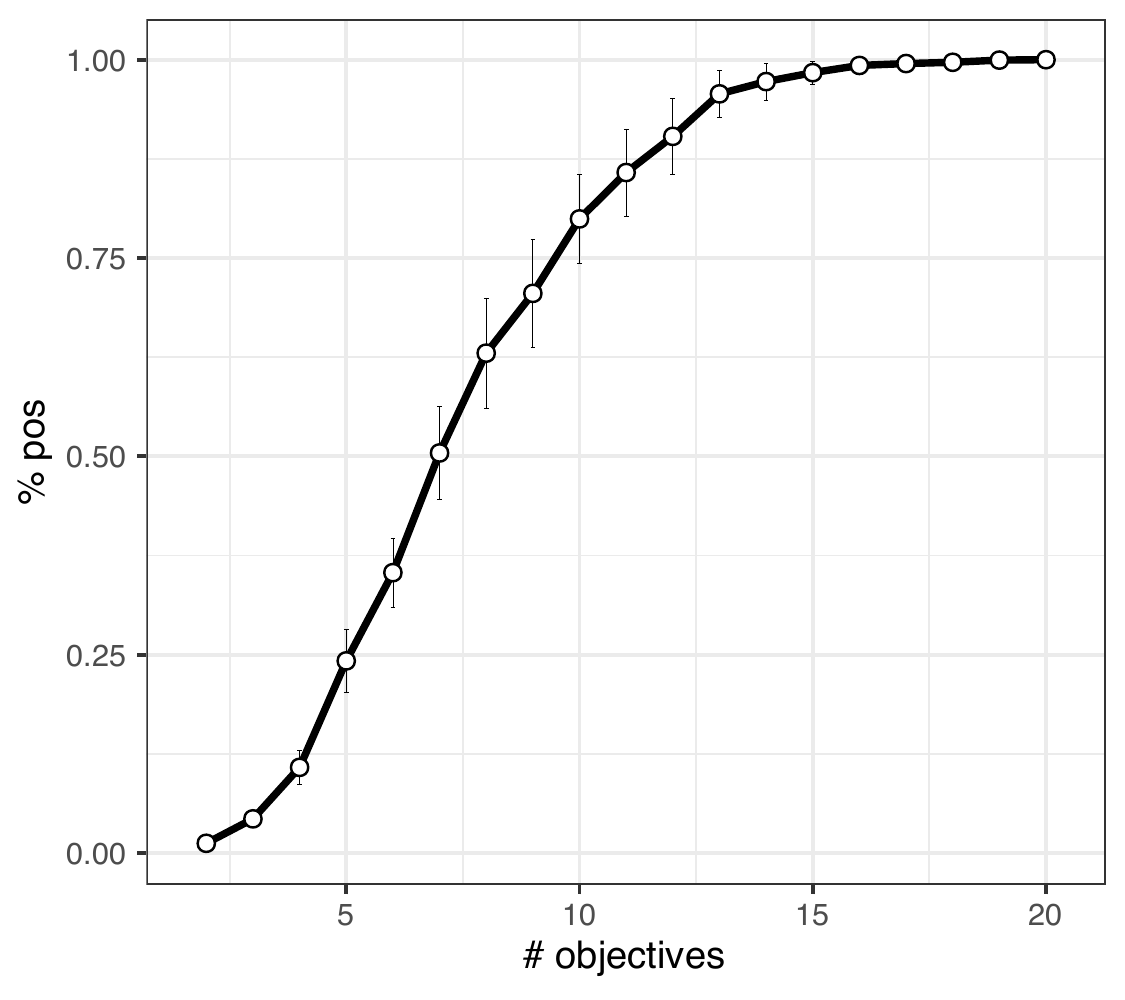}%
\caption{%
Proportional number of Pareto optimal solutions with respect to the number of objectives $m$ for multi-objective \mbox{NK-landscapes}.%
}
\label{fig:pos}
\end{figure}

In the combinatorial case, the number of Pareto optimal solutions is known to grow exponentially with the number of objectives in the worst case, that is $\bigO(c^{m-1})$, where $c$ is a constant value~\citep{BAZGAN2013}. Furthermore, as shown by~\citet{BAZGAN2013}, this bound is tight for many classical multi-objective combinatorial optimization problems, such as selection, knapsack, shortest path, spanning tree, traveling salesperson, and s–t cut problems. Obviously, the number of Pareto optimal solutions is also bounded by the size of the whole feasible set. 

\citet{verel2013} experimentally investigate the number of Pareto optimal solutions for multi-objective NK-land\-scapes with $2$, $3$ and $5$ objectives. They report that it grows exponentially with the problem size, the number of objectives and the degree of conflict among the objectives, and that it slightly decreases with the number of variable interactions. 
%
%
Given the aim of the current paper, we focus on the number of objectives by considering multi-objective \mbox{NK-landscapes} with $2$ to $20$ objectives.
In Fig.~\ref{fig:pos}, we report the proportion of Pareto optimal solutions in the solution space with respect to the number of objectives.
%
%
We see that less than $5\%$ of solutions are Pareto optimal for bi-objective problems ($m=2$), whereas this proportion grows to about $50\%$ for $m=7$ objectives.
For $m=20$ objectives, more than $99\%$ of solutions are Pareto optimal for all considered instances. 
%

\hide{\textbf{[Arnaud]} 
Following our discussions, I added a curve for $\rho=0.5, k=2$ (remind that we cannot have a high negative correlation, since $\rho > \frac{-1}{\M-1}$)\\
However, I feel that adding this information is a bit out-of-focus and might raise even more questions from the reviewers: if we study the impact of $\rho$ and $k$ for the proportion of Pareto optimal solutions, why not doing the same for all other Sections?
(In fact, we could even do the same for the number of variables $n$)\\
Therefore, I would suggest to keep $n$, $k$ and $\rho$ fixed throughout the paper in order to maintain the focus of the paper as clear as possible: what is the impact of the number of objectives $m$?\\
Of course, I'm not trying to say that studying the objective correlation ($\rho$) is not interesting, and the same for $n$ and $k$ (this is actually the point of our previous EJOR paper \citep{verel2013} and one of the reason we proposed $\rho m n k$-landscapes, I added comments about those previous results). But I feel it would be too much for the paper, and I don't see any reason to do so for some parts and not for others. We might also mention this in the future work section.\\
Anyway, I added both figures, and I would be happy to go with the one you like most :)
}

\subsection{Probability for a Solution to be Non-Dominated}
\label{sec:dom_rel}

With a growing number of objectives, the dominance relation becomes less discriminative~\citep{aguirre_many_survey}. Let us consider the comparison between two arbitrary solutions $\bx_1$ and $\bx_2$ on $m$ objectives. Assume that the probability of equal objective values 
can be neglected, and that the comparison with respect to each objective is independent. For each objective there is a $1/2$ probability that $\bx_1$ has a better value for this objective than $\bx_2$, and the same probability applies for the opposite situation. As such, given a problem with $m$ objectives, the probability that one solution dominates the other one is
$1/2^{(m-1)}$.
Thus, as $m$ increases, it becomes more likely that two arbitrary solutions are mutually non-dominated. If the objectives are positively correlated, this probability increases, and if they are negatively correlated this probability decreases~\citep{verel2013}.

As a consequence of the reduced discriminative power of the dominance relation, the probability that a given solution is Pareto optimal increases with the number of objectives.
The probability that all $\mu$ randomly selected pairs of solutions  
are mutually non-dominated is:%
\begin{equation}
\label{eq:nd_pairs}
\prod_{i=1}^{\mu} \left( 1-\frac{1}{2^{m-1}} \right) \; = \; \left( 1-\frac{1}{2^{m-1}} \right)^{\mu}
\end{equation}
%
This theoretical probability is reported in Fig.~\ref{fig:prob_nd_model} for different number of objectives, $m$, and number of pairs, $\mu$.
As such, the probability that all $\mu$ randomly selected pairs of solutions are mutually non-dominated increases with the number of objectives, and decreases with the number of pairs. It becomes very likely that all pairs are mutually non-dominated for problems with $m>15$ objectives, even for large numbers of pairs.
In Fig.~\ref{fig:nk_prob_nd_pairs_m}, we measure this proportion for multi-objective \mbox{NK-landscapes} using $m \in \{2, 3, \ldots, 20\}$. 
In particular, we perform $30$ independent samples per instance for each setting.
Comparing Fig.~\ref{fig:prob_nd_model} with Fig.~\ref{fig:nk_prob_nd_pairs_m}, we observe that the theoretical model fits the experimental data well, although there are some small but significant differences. Through an additional analysis we found that this is caused by small positive or negative correlations between objectives resulting from random (diss)similarities of instance parameters. Such correlations modify the probabilities that two randomly selected solutions are non-dominated w.r.t. the assumed independent objectives model. Since Eq. (1) is exponential, even very small differences of the probabilities that two randomly selected solutions are non-dominated may results in noticeable differences in the probability that all pairs are non-dominated. 

\begin{figure}[!t]
\centering%
\includegraphics[width=0.5\textwidth]{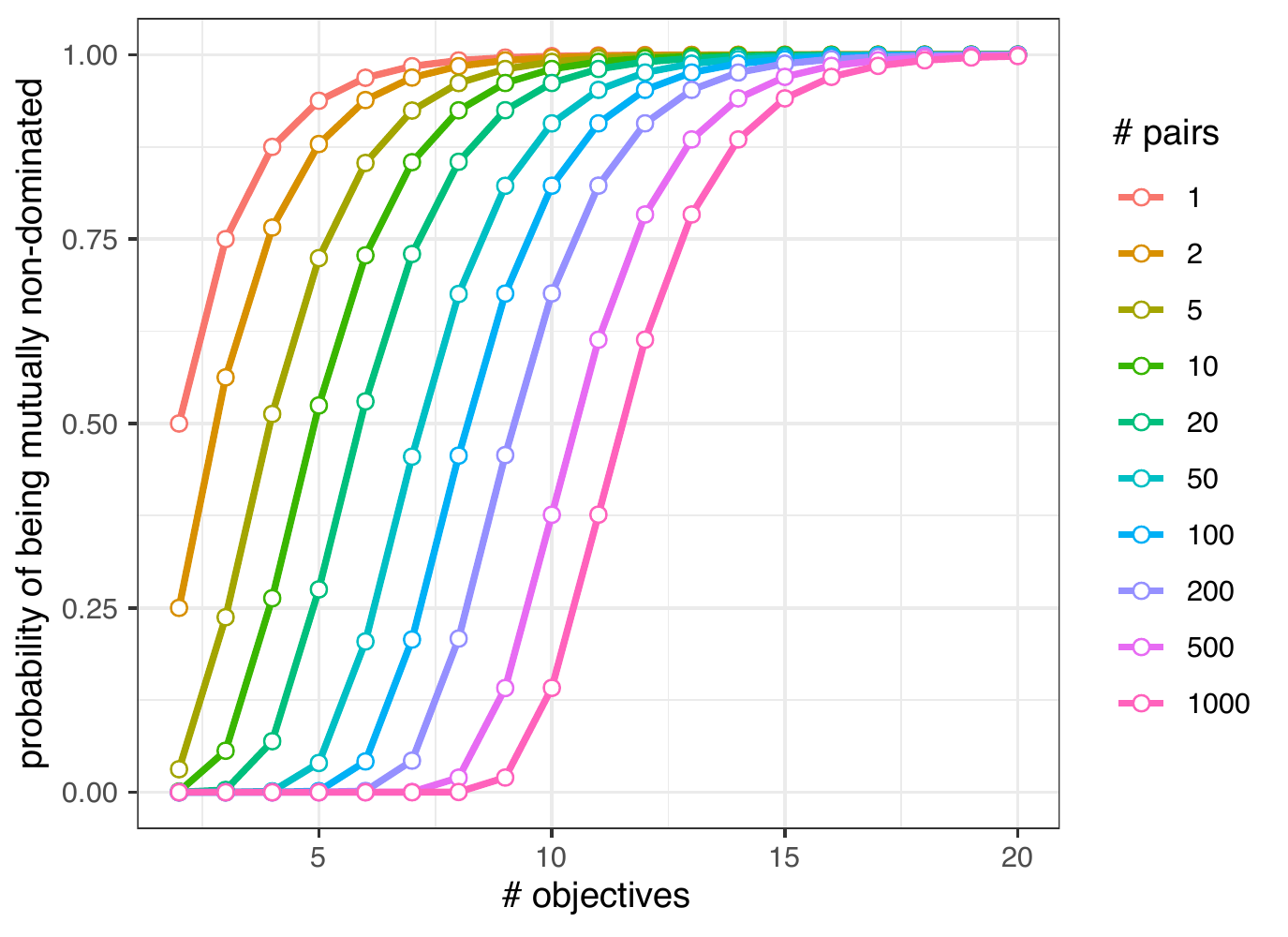}%
\includegraphics[width=0.5\textwidth]{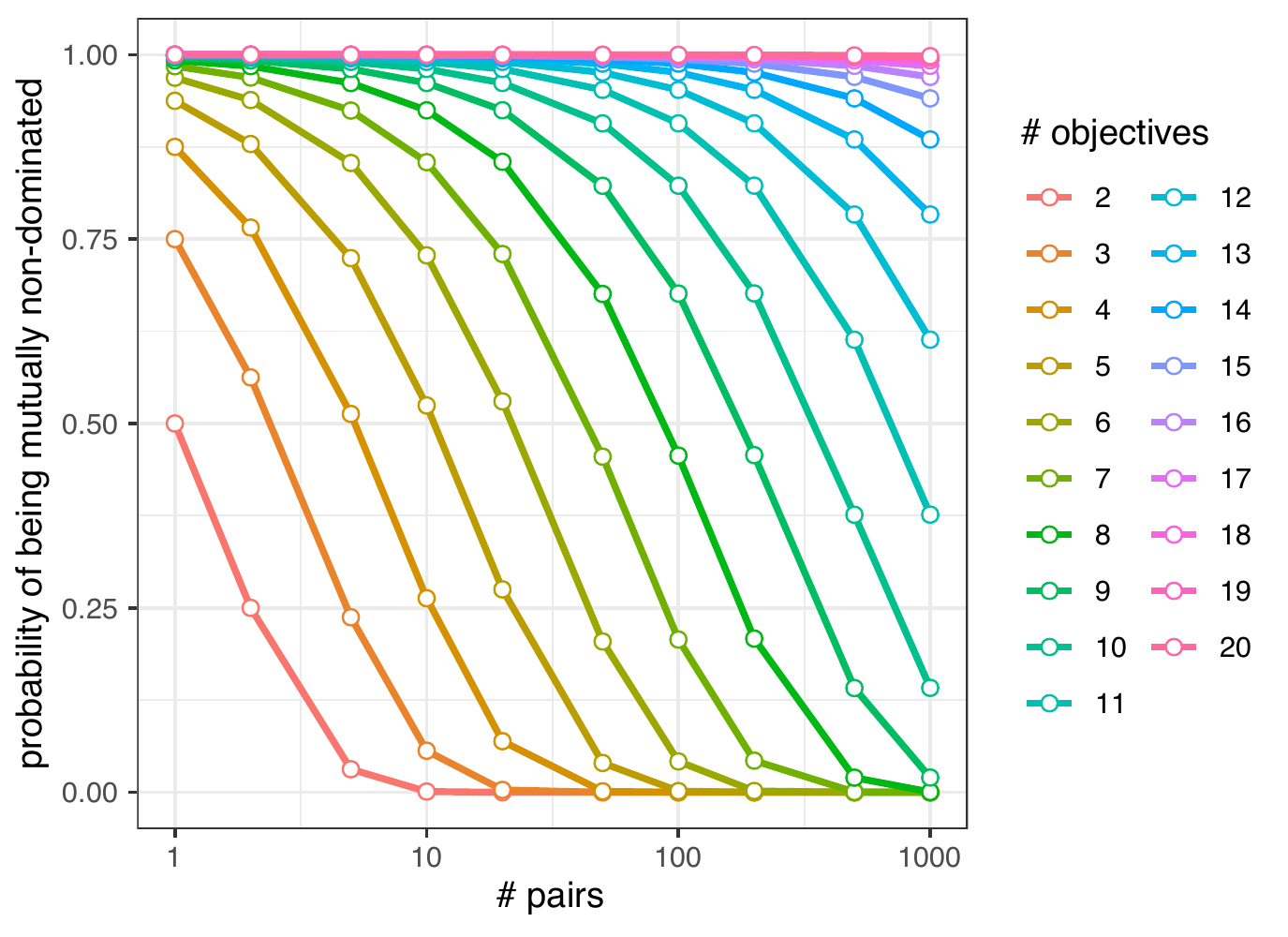}%
\caption{
Probability that all $\mu$ randomly selected pairs of solutions are mutually non-dominated
from the theoretical model given in Eq.~(\ref{eq:nd_pairs}), with respect to the number of objectives $m$ (left), and to the number of pairs $\mu$ (right).}
\label{fig:prob_nd_model}
\end{figure}

\begin{figure}[!t]
\centering%
\includegraphics[width=0.5\textwidth]{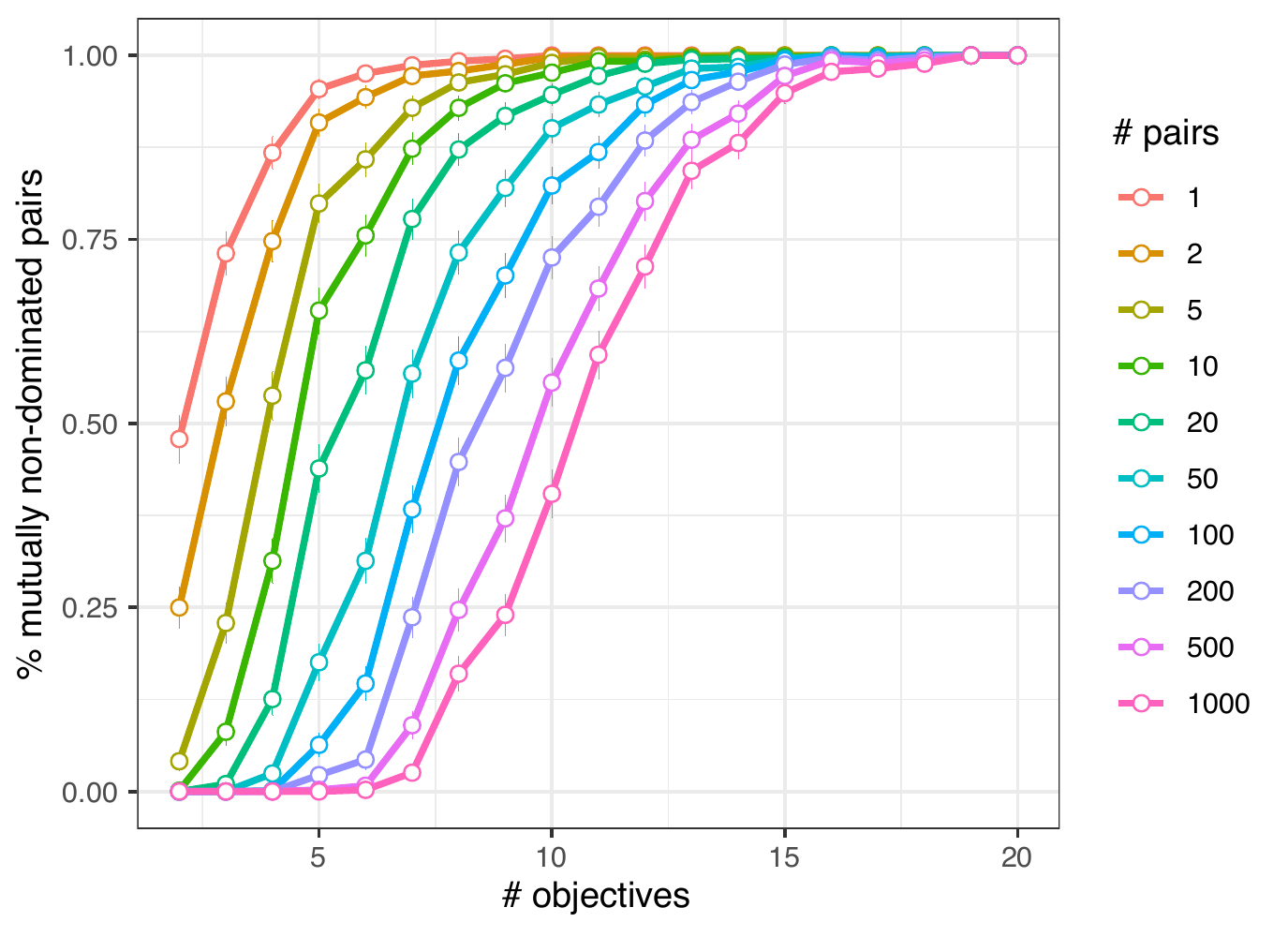}%
\includegraphics[width=0.5\textwidth]{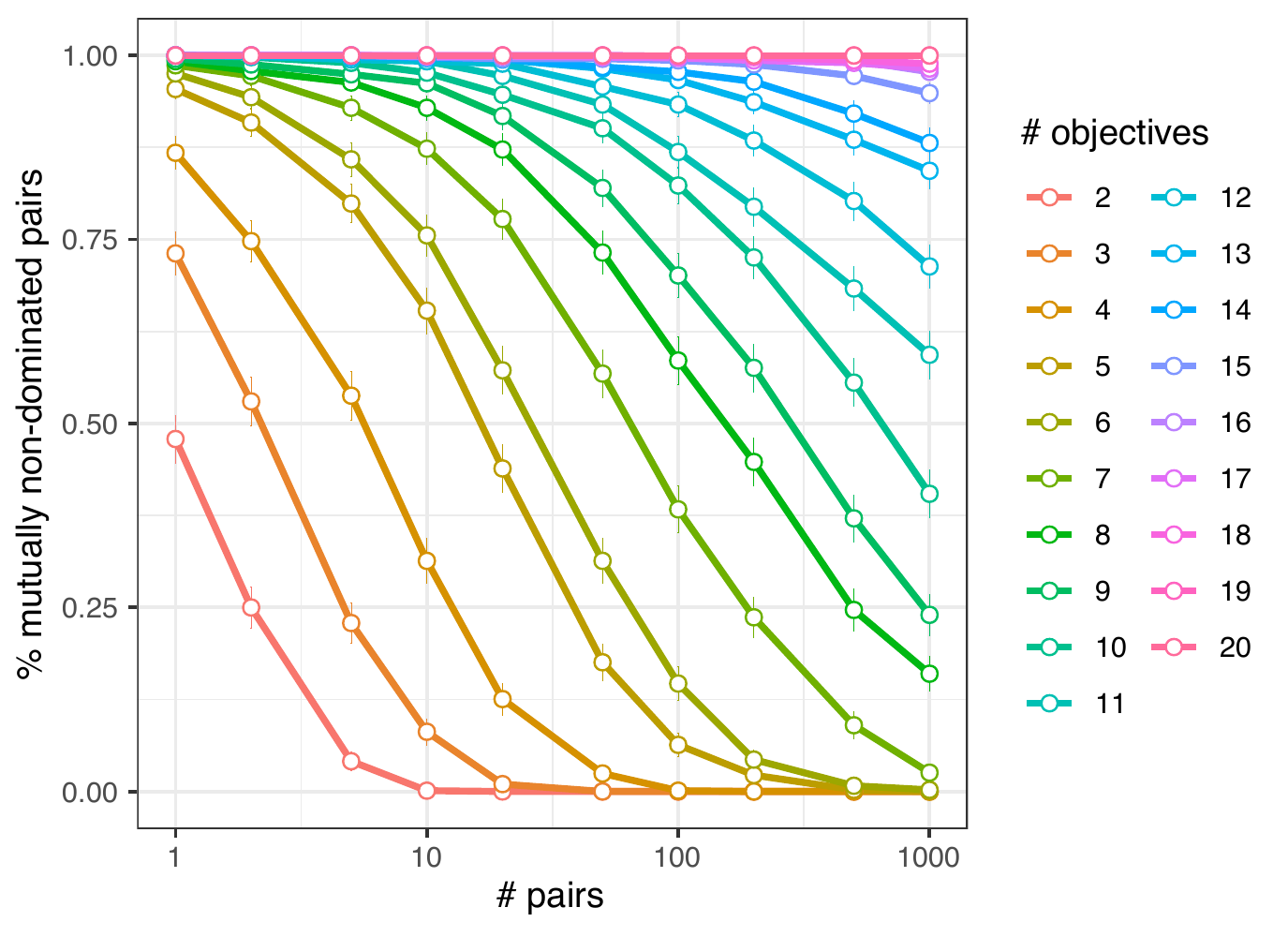}
\caption{
Proportion 
from $\mu$ randomly selected pairs of solutions
that are mutually non-dominated for multi-objective \mbox{NK-landscapes}, with respect to 
the number of objectives $m$ (left), and to the number of solutions $\mu$ (right).}
\label{fig:nk_prob_nd_pairs_m}
\end{figure}

Although the probability for a solution to be non-dominated
in random populations 
might be difficult to derive theoretically due to the dependencies between solutions, we report in Fig.~\ref{fig:prob_nd_nk} empirical results on the proportional number of times a given solution is \textit{not} dominated by any of $\mu$ other solutions, all drawn at random.
%
We consider the same set of multi-objective \mbox{NK-landscapes}. 
The probability for a solution to be non-dominated grows with the number of objectives, but the increase is steeper than that of Fig.~\ref{fig:nk_prob_nd_pairs_m}.
It also reduces with the population size, although its effect is lower than that of the number of objectives.
For $2$ objectives, it ranges from about $70\%$ when a single pair of solutions is considered, to less than $2\%$ when the solution is compared against $1\,000$ others.
By contrast, for problems with more than $12$ objectives, there is over $90\%$ chance that the solution is not dominated by any other, independently of the population size. 

Additionally, we report in Fig.~\ref{fig:nb_nd_nk} the proportional number of non-dominated solutions in populations containing $\mu$ random solutions.
The curves follow a similar trend than that of Fig.~\ref{fig:prob_nd_nk}.
For small populations ($\mu=10$), there is about $30\%$ of non-dominated solutions for $2$-objective problems, $50\%$ for $3$-objective problems, to more than $70\%$ for problems with $4$ objectives and more,
and even more than $90\%$ for problems with $6$ objectives and more.
For larger populations, this proportion decreases but remains higher than $90\%$ for problems with $12$ objectives or more, independently of the population size. 


\begin{figure}[!t]
\centering%
\includegraphics[width=0.5\textwidth]{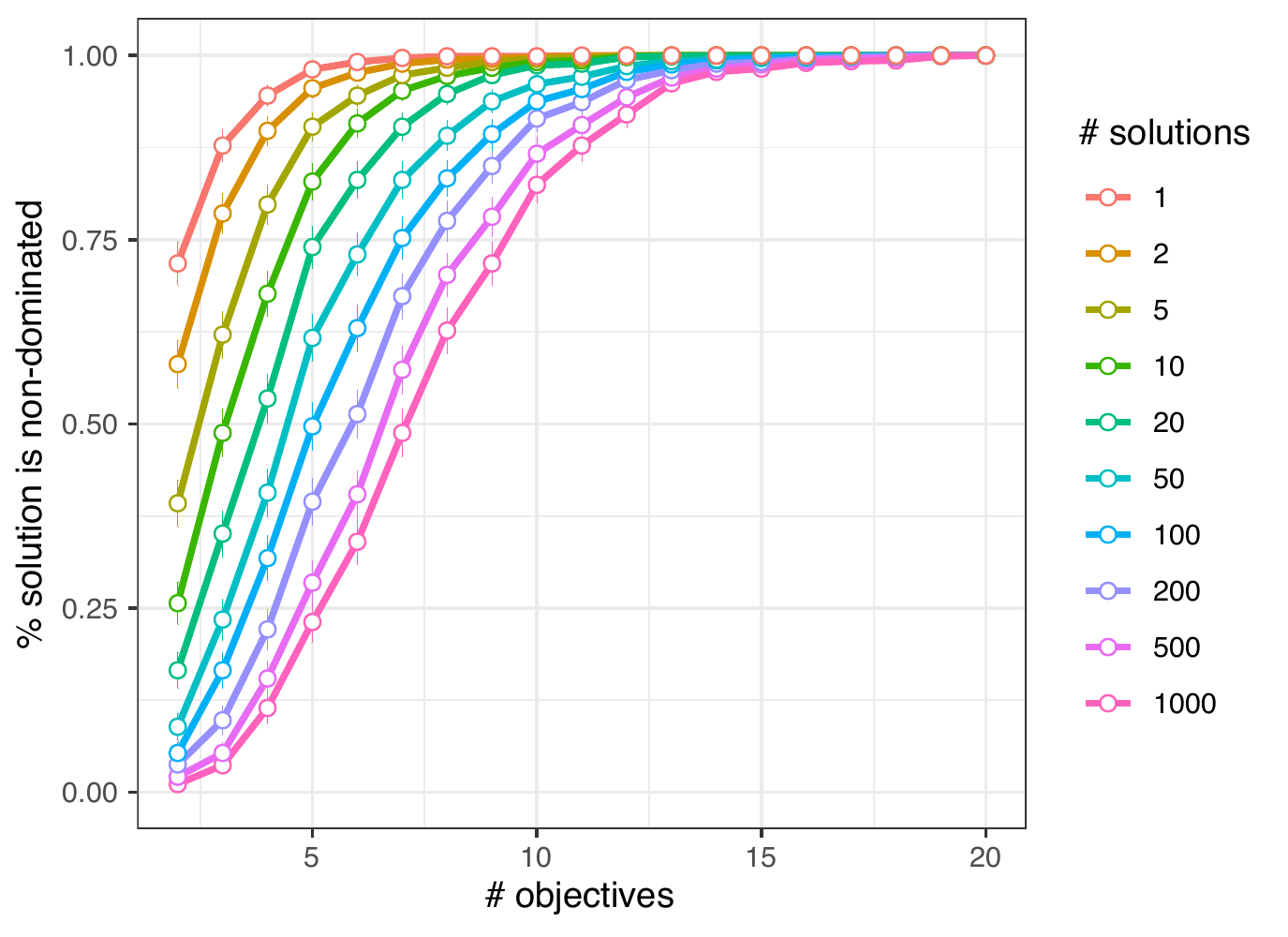}%
\includegraphics[width=0.5\textwidth]{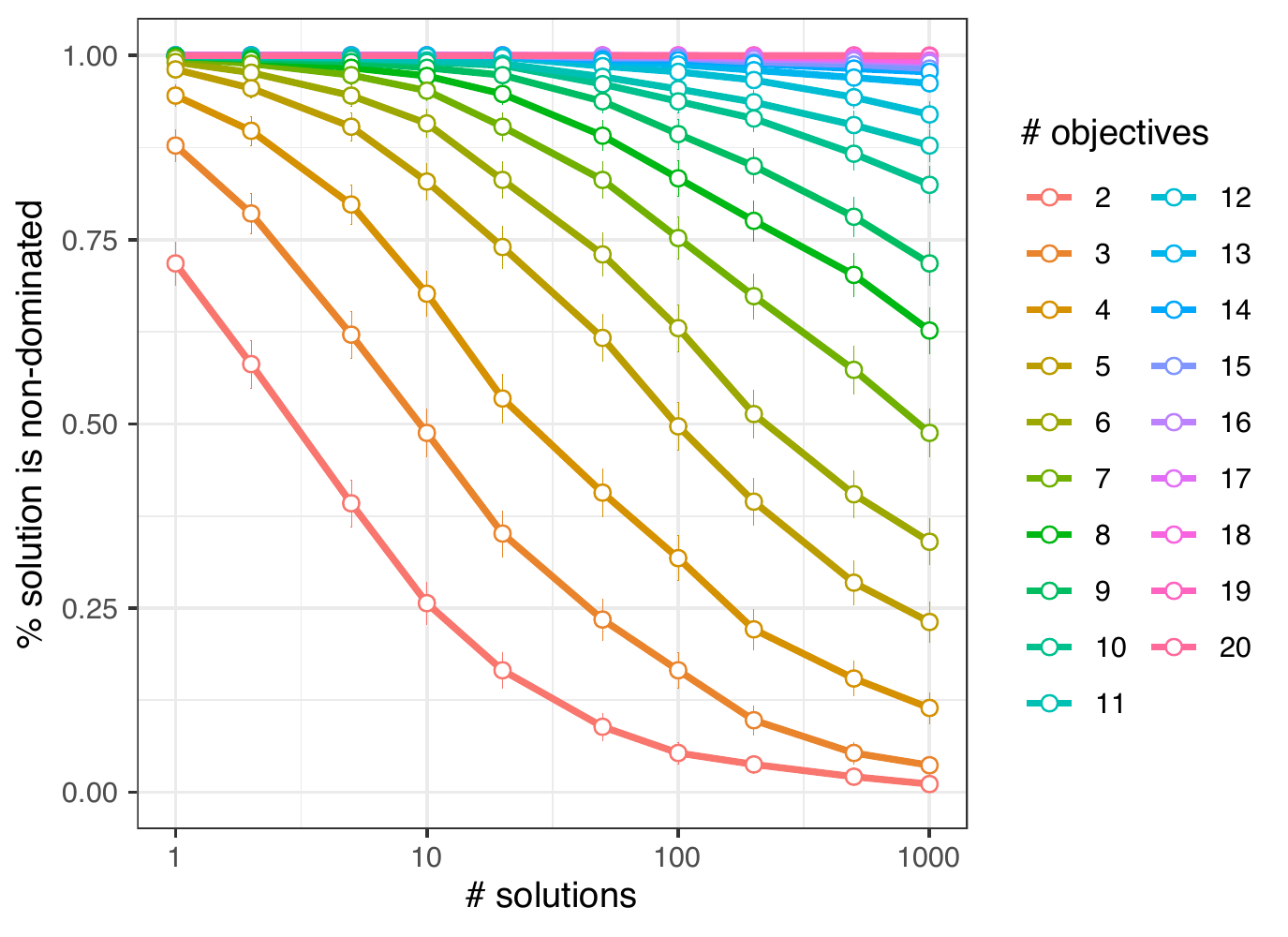}%
\caption{
Proportion of solutions that are not dominated by any of $\mu$ random solutions for multi-objective \mbox{NK-landscapes}, with respect to the number of objectives~$m$~(left), and to the number of solutions $\mu$ (right).
}
\label{fig:prob_nd_nk}
\end{figure}

\hide{[remove, probably more complicated than this]
and the expected number of non-dominated solutions in this population is then
$$
\mu \; \Big( 1-\frac{1}{2^m} \Big)^{\mu-1}
$$
\RICHARD{Maybe remove here but then elaborate on  in the future work section.}
Although the expected number of non-dominated solutions in random populations might be difficult to derive theoretically due to the dependencies between solutions, we report in Fig. 5 the measured number of non-dominated solutions in random populations for multi-objective \mbox{NK-landscapes}.
}

\begin{figure}[!t]
\centering%
\includegraphics[width=0.5\textwidth]{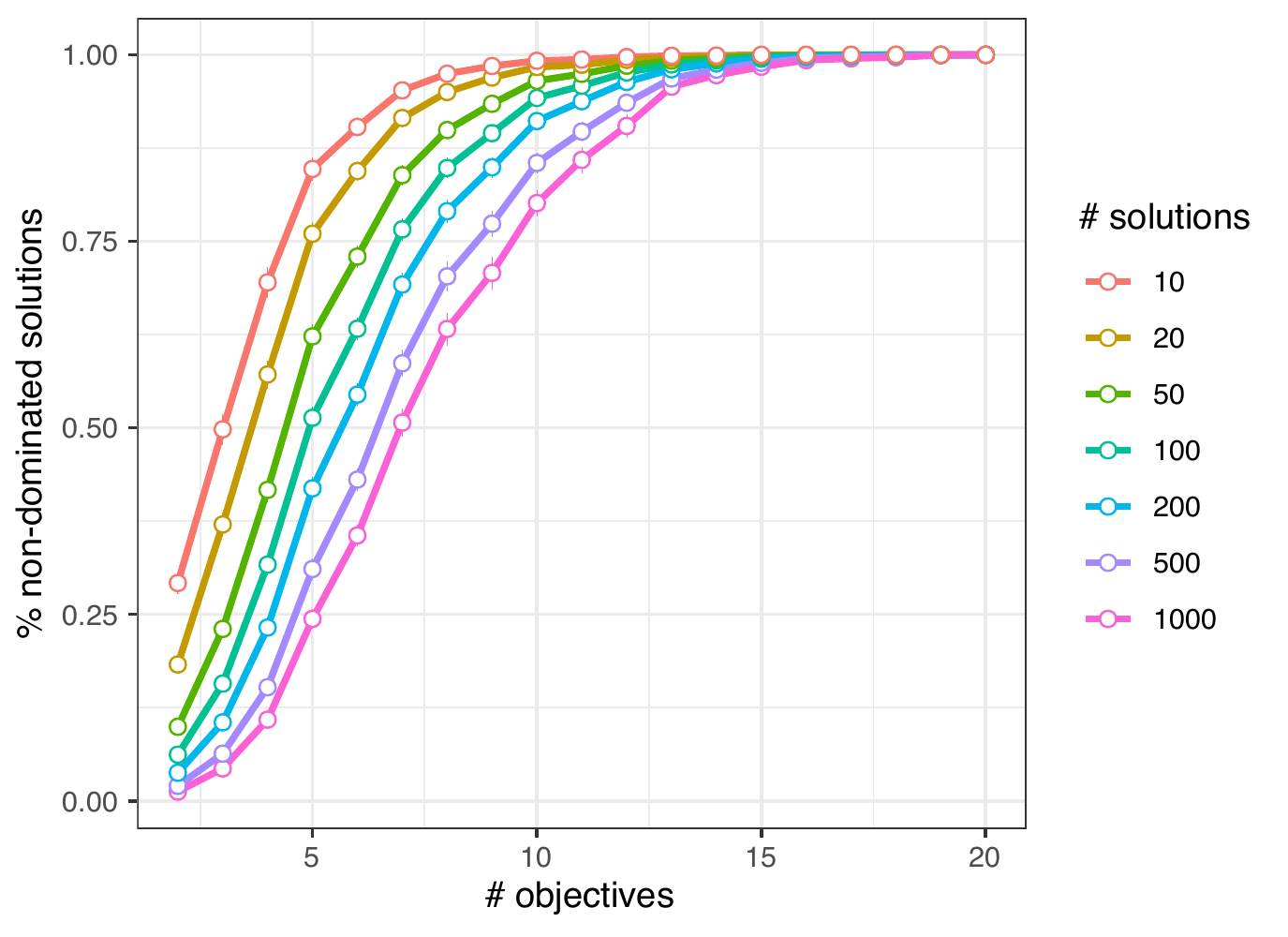}%
\includegraphics[width=0.5\textwidth]{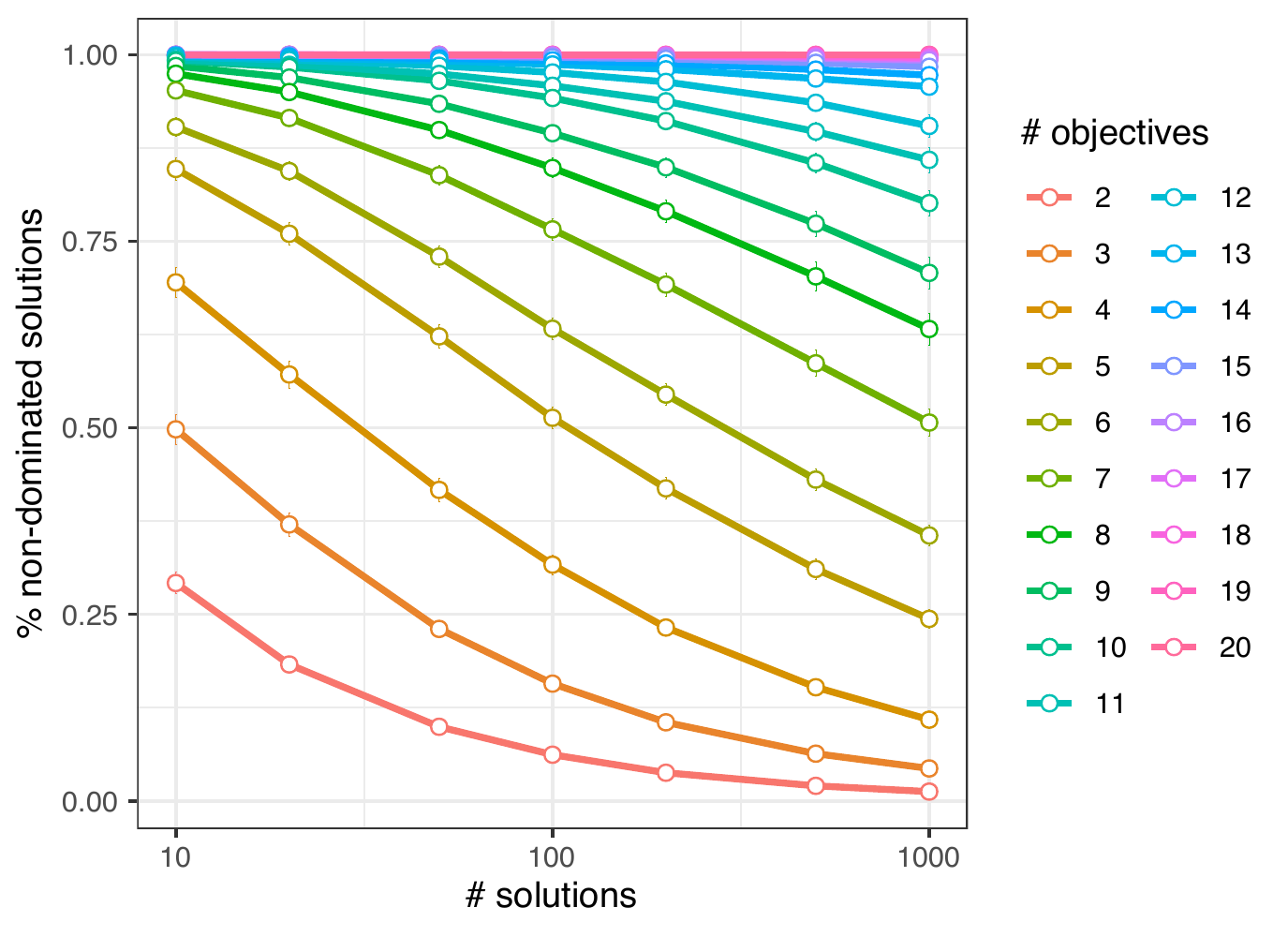}%
\caption{Proportion of non-dominated solutions in random populations 
for multi-objective \mbox{NK-landscapes}, with respect to the number of objectives~$m$ (left), and to the number of solutions $\mu$ (right).
}
\label{fig:nb_nd_nk}
\end{figure}

\subsection{Visualizing the Objective Space}\label{dimensionalityOS}

In the continuous case, the size of the Pareto set does not grow formally, because it is infinite
already for two objectives. However, the dimensionality of the objective space and the Pareto front grows. This means that more points or directions are typically required to obtain an approximation of the Pareto front or the possible search directions, respectively. 

An additional challenge arises as the dimensionality of the objective space increases, namely the visualization of an optimizer's search performance~\citep{tuvsar2015visualization}. Scatter plots are a popular choice for problems with up to $m=3$ objectives but, beyond that, lead to a loss of Pareto dominance relation information between solutions because approaches rely on mapping a higher-dimensional space to a 2 or 3-dimensional space. Alternatively, the use of pair-wise plots becomes quickly overwhelming (requiring $(m^2-m)/2$ plots). Visualization techniques designed for higher-dimensional spaces include parallel coordinate plots and heatmap-based visualization techniques, potentially coupled with dimensionality reduction methods. An alternative is to use inherently visualisable test problems, such as distance-based multi/many-objective problems~\citep{koppenetal2005,distanceBasedProblem2019}, to visualize the search behavior in its native domain. If visualization of the Pareto set/front is not sufficient to select a single final solution, then the decision maker (DM) can use a decision support system to interactively learn about the problem and finally arrive at a solution. Such methods are also referred to as navigation methods~\citep{allmendinger2017navigation}.


\subsection{Number of Preference Parameters}\label{NrOfPreferenceParam}

In numerous multi-objective optimization methods, the DM is expected to express his/her preferences, e.g. in the form of weighting coefficients or reference levels (aspiration levels/goals) specified for each objective~\citep{miettinen2012nonlinear}. The number of such preference parameters grows just linearly with the number of objectives. However, for certain methods, such as AHP~\citep{saaty1987analytic}, preference parameters are expressed with respect to each pair of objectives; in this case, their number grows quadratically. 

In the non-convex case, especially in multi-objective combinatorial optimization, non-linear scalarization techniques are more useful. 
\citet{PasSer84} introduce a scalar optimization problem that can be formulated using a general non-linear scalarizing functional that we will discuss in Section \ref{s-furtherscal}. Corresponding non-linear scalarization methods, including applications in economics, are studied by \citet{BonCor88,BonCre07} and \citet{luen92b,luen92}. The parameters in the scalar optimization problem by \citet{PasSer84} are elements $k \in \mathbb{R}^m$ and $a \in \mathbb{R}^m$, see Eq.~(\ref{PS}) in Section~\ref{s-furtherscal}. Furthermore, in the $\epsilon$-constraint problem by 
\citet{Hai71,ChaHai83}, parameters $\epsilon = (\epsilon_1, \ldots , \epsilon_m) \in \mathbb{R}^m$ are involved. The scalarization method by \citet{PasSer84} and the $\epsilon$-constraint problem are used for deriving parameter-based adaptive algorithms in 
\citet{eich08}, compare also \citet{Pol76}. By varying the parameters $k \in \mathbb{R}^m$ and $a \in \mathbb{R}^m$ in Eq.~(\ref{PS}), adaptive algorithms generate an equidistant approximation of the set of Pareto optimal elements to the MO problem by solving scalarized problems.

If the number of objective functions $m$ is increasing, then the parameter vectors $k$ and $a$ belong to higher dimensional spaces, i.e., the number of preference parameters increases linearly.

\subsection {Probability of Having Heterogeneous Objectives}\label{HeterogObj}
As the number of objectives increases and one is faced with a many-objective problem, the assumption that all objective functions share the same properties becomes less valid. In other words, it becomes more likely that the objectives are heterogeneous meaning the objectives may differ in, for example, complexity (e.g. linear vs non-linear, unimodal vs multimodal), evaluation efforts in terms of time, costs, or resources, available information (expensive black box vs given closed-form function), or determinism (stochastic vs deterministic). 

Examples of problems with heterogeneous objectives include real-world problems where some objectives must be evaluated using time-consuming physical experiments in the laboratory, while other objectives are evaluated using relatively quicker calculations on the computer. Problems that fall into this category can be found, for instance, in drug design~\citep{small2011efficient} and engineering~\citep{terzijska2014}. In the former,
the objectives related to potency and side effects of drugs need to be validated experimentally, while the cost of manufacturing a drug or number of drugs in a drug combinations can be obtained quickly. 

\begin{figure}[!t]
\centering
\includegraphics[width=0.5\textwidth]{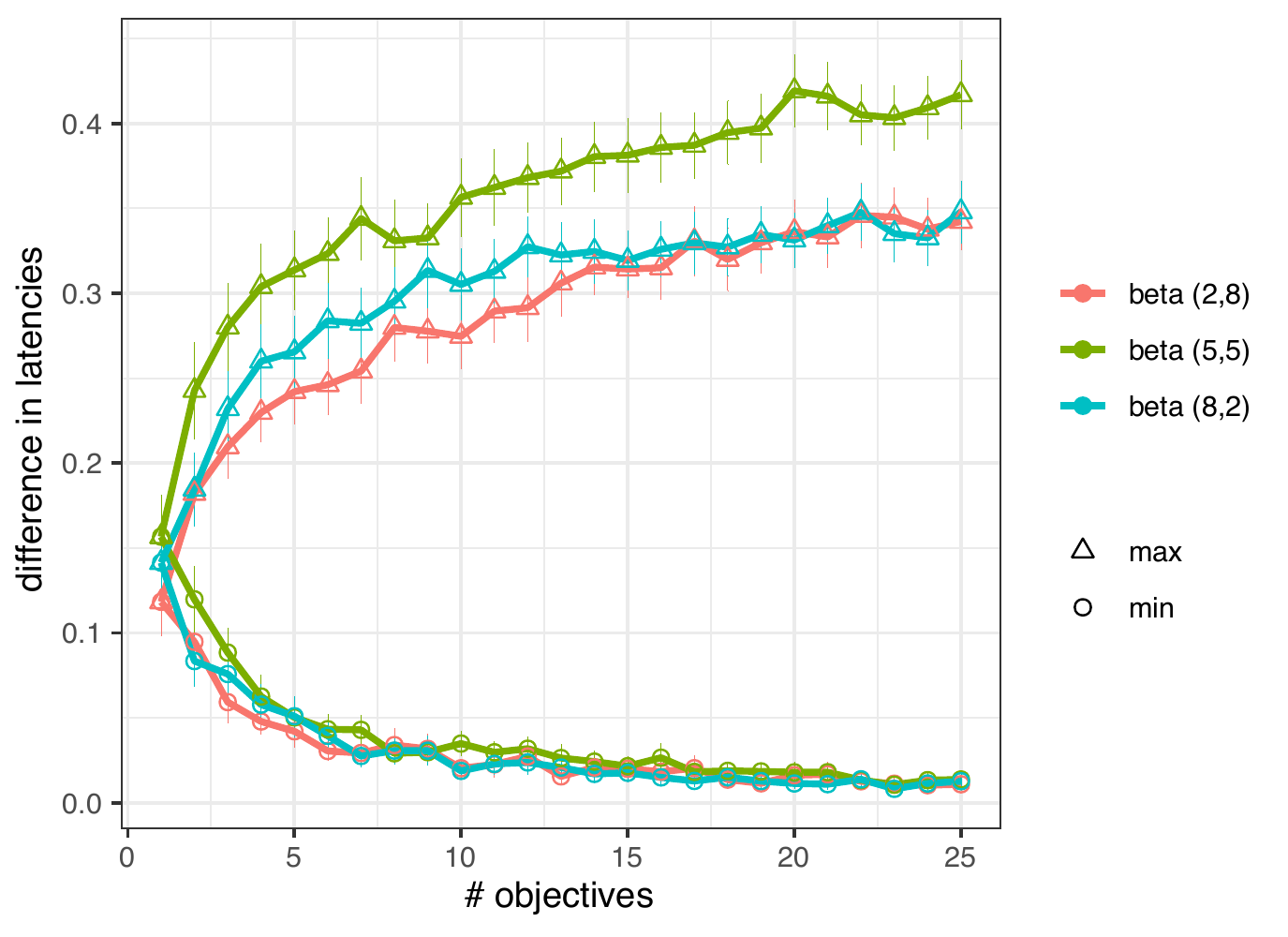}%
\includegraphics[width=0.5\textwidth]{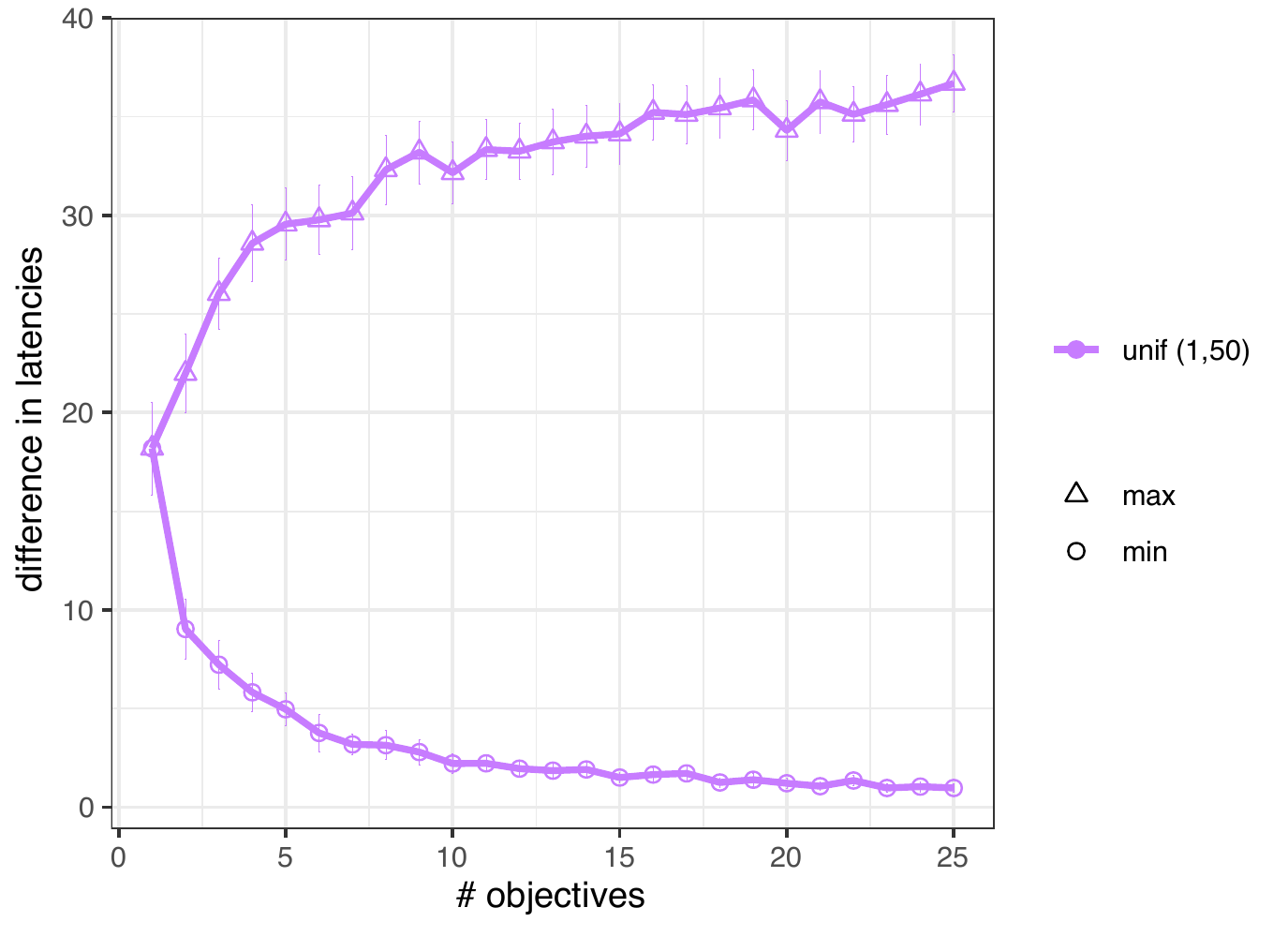}%
\caption{Mean and standard error of minimum (circles) and maximum differences (triangle) in objective evaluation duration (y-axis) as a function of the number of objectives, $m$ (x-axis). The left plot shows this data for three Beta distributions, and the right plot for a uniform distribution.}
\label{fig:HetOb}
\end{figure}

The first work we are aware of on the topic of heterogeneous objectives was by~\citet{allmendinger2013hang}. This work was motivated by a real-world scenario where the objective function components were of different ``latency'' (evaluation times). This initial study was then extended in~\citep{allmendinger2015delays}, where different strategies for coping with latencies in a bi-objective problem were proposes and analyzed. Recently, more research has been carried out on the topic of heterogeneous objectives including the application of surrogate-assisted methods~\citep{chugh2018surrogate,xilu2020}, and non-evolutionary methods~\citep{thomann2019trust_PhD}. The topic on heterogeneous functions was also the focus of a working group at a Dagstuhl seminar in 2015~\citep{eichfelder4}. 

To motivate and further raise awareness of the heterogeneous objectives, we will investigate the relationship between the number of objectives in a problem and the degree of heterogeneity induced by them. For the purpose of this experiment, we limit ourselves to heterogeneity in form of latencies (evaluation times) across the objectives. Assuming a problem with a certain number of objectives, each being associated with an evaluation duration drawn from a given distribution, we want to know what is the mean minimum and maximum difference in evaluation times among these objectives. We repeated this experiment for different numbers of objectives (1-25 objectives as indicated on the x-axis) and four distributions to drawn the evaluation durations from. Each configuration was repeated 100 times, and the mean and standard error of the minimum and maximum differences in evaluation durations are plotted in Figure~\ref{fig:HetOb}. As the considered distributions, we considered (i)~three different Beta distributions --- $\texttt{beta}(2,8)$ (skewed to the right), $\texttt{beta}(8,2)$ (skewed to the left), and $\texttt{beta}(5,5)$ (symmetric), each defined on the interval [0,1] --- allowing us to simulate conveniently skewness in the evaluation of durations of the objectives, and (ii)~one uniform distribution defined on the interval [1,50]. The Beta distribution has also been used extensively in the literature to quantify the duration of tasks (see seminal paper of~\citet{PERT_1959}) in different contexts. The purpose of using the uniform distribution is to have a baseline to compare against.  

It can be seen from the left plot of Figure~\ref{fig:HetOb} that, regardless of the skewness and probability distribution, the mean minimum and maximum difference in evaluation durations of objectives starts roughly from the same level for a bi-objective problem ($m=2$) with the two mean differences becoming exponentially more and more distinct as the problem becomes many-objective. This pattern is expected since it becomes more likely that the evaluation duration of a new objective (as sampled from the Beta distribution) is either more similar or distinct to the evaluation duration of an existing objective. Perhaps more surprisingly is the asymmetry between the minimum and maximum difference with increasing $m$: While the mean minimum difference is hardly affected by the choice of probability distribution (with the difference flattening quickly from around $m>15$ objectives), there is a statistical difference between the Beta distributions when considering the mean maximum difference; in particular, the mean maximum difference associated with the symmetric distribution ($\texttt{beta}(5,5)$) increases faster than with the two skewed distributions. This pattern stems from the fact that its more likely to sample extreme values for the evaluation duration with the symmetric distribution. Consequently, the gap between the mean minimum and maximum difference in evaluation times is even greater in the case of the uniform distribution (right plot of Figure~\ref{fig:HetOb}), where the probability of sampling any possible evaluation time is equal within the interval [1,50].

Having observed that the algorithm choice is affected amongst others by latency (see e.g.~\citep{allmendinger2013hang,allmendinger2015delays,chugh2018surrogate,xilu2020}), this experiment indicates that knowing about the distributions of latency (evaluation times) can be used in the selection and design of novel algorithms to cope with heterogeneous objectives.

\begin{figure}
\centering
\begin{minipage}{0.65\textwidth}
  \centering
  \includegraphics[width=\textwidth]{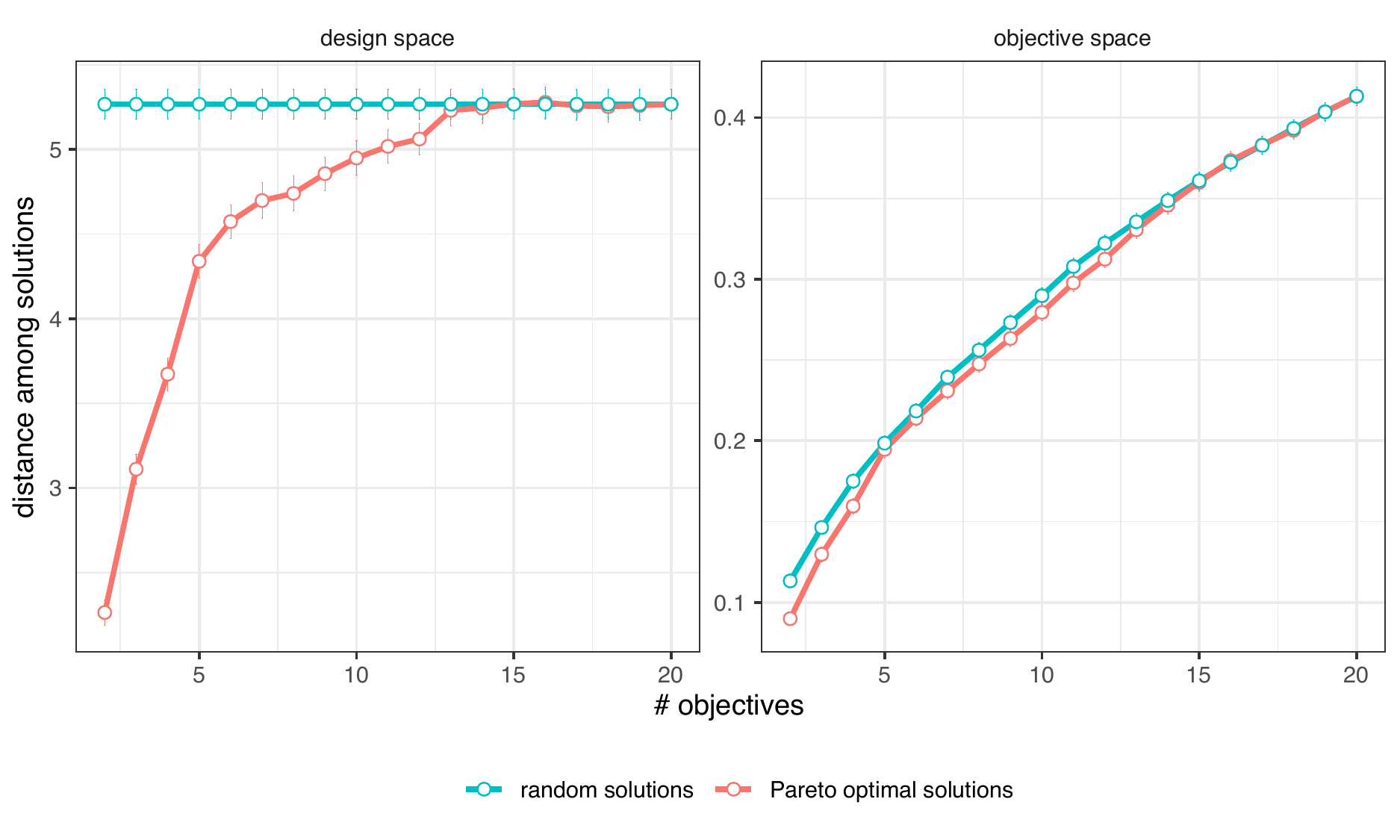}
  \caption{Distance among solutions from the whole search space and from the Pareto set for multi-objective \mbox{NK-landscapes} with respect to the number of objectives.}
  \label{fig:dist_sol}
\end{minipage}%
\hfill%
\begin{minipage}{0.325\textwidth}
  \centering
  \includegraphics[width=\textwidth]{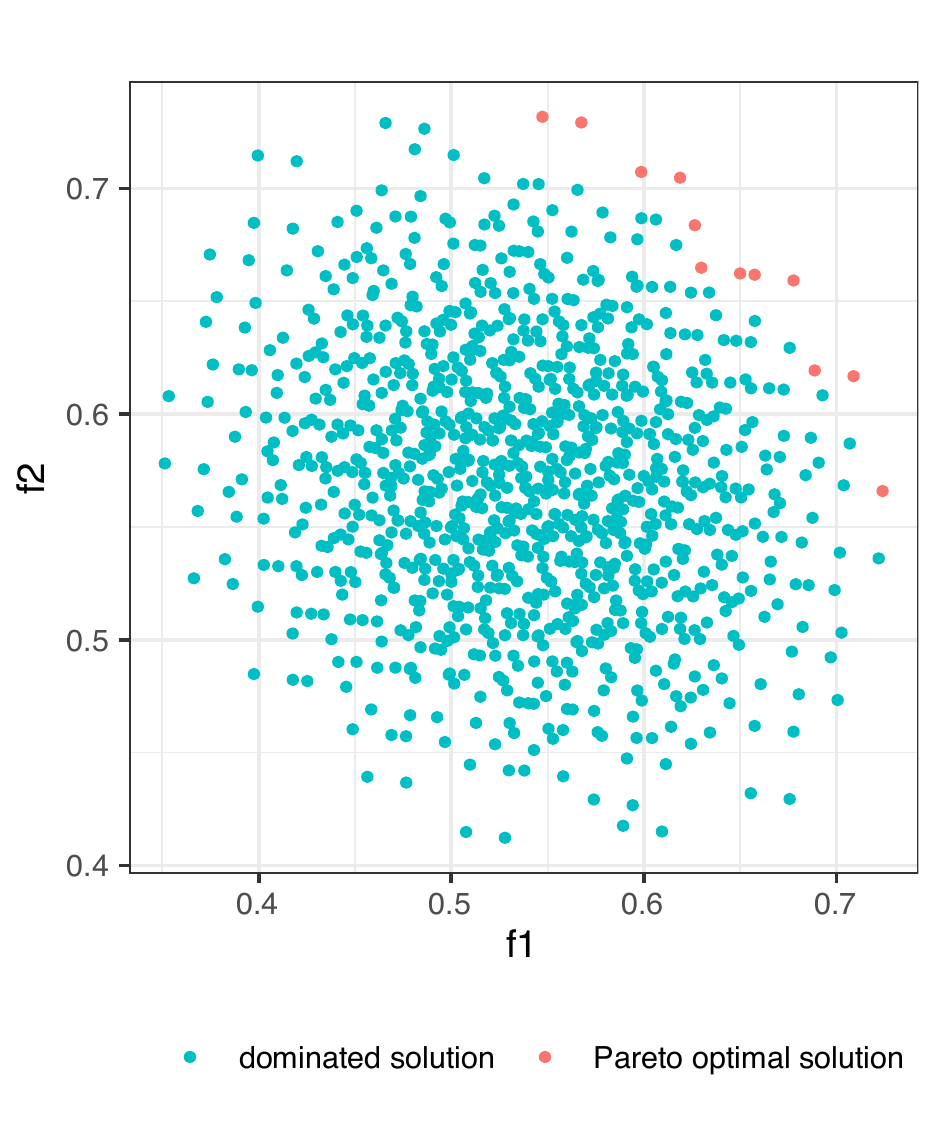}%
  \caption{Objective space of a 
  multi-objective NK-landscape instance with $m=2$ objectives.}
  \label{fig:obj_space}
\end{minipage}
\end{figure}

\subsection{Distance between Solutions}\label{distanceBetweenSolutions}

The distance between pairs of solutions randomly selected from the whole search space and from the Pareto set is reported in Fig.~\ref{fig:dist_sol}.
We consider multi-objective \mbox{NK-landscapes} with $m \in \set{2, 3, \ldots, 20}$.
For each instance, we select $30$ random pairs of solutions, and we report their Hamming distance in the design space and  their Euclidean distance in the objective space.
The Hamming distance among random solutions does not depend on the number of objectives, of course. However, the Hamming distance among Pareto optimal solutions increases significantly with the number of objectives, $m$. For $m>13$ objectives, the expected distance among Pareto optimal solutions matches the one from random solutions. This comes as no surprise since most solutions are Pareto optimal in such high-dimensional objective spaces, as already mentioned in Section~\ref{sec:pos}. However, when few objectives are considered, Pareto optimal solutions are much closer than random solutions.
%
As such, for many-objective problems, good-quality solutions spreading along the Pareto front are far away from each other, so that we argue that few building blocks (if any) might actually be exploited by (blind) recombination.

When looking at the objective space, the Euclidean distance among objective vectors is increasing almost linearly with the number of objectives, with the distance among Pareto optimal solutions being only slightly lower than the distance among random solutions for $m<15$ objectives.
As can be visualized in the illustrative two-objective \mbox{NK-landscapes} provided in Fig.~\ref{fig:obj_space}, reproduced from~\citet{verel2013}, the objective space resembles a multidimensional ``ball'', so that non-dominated objective vectors are not particularly closer than random ones.
Given that the distance between solutions in the objective space increases with the number of objectives, we expect that identifying a high-quality representation of the Pareto front, in terms of coverage, gets more difficult for many-objective optimization problems.

\section{Effect of the Number of Objectives on the Complexity of Multi-objective Procedures and Algorithms}\label{MOEAprocedures}

In this section we report a number of computational experiments, both performed for the purpose of this paper and analyzing results reported in the literature. Note that, although a number of computational studies with the methods considered in this paper have been reported in the literature, they are usually focused on performance comparison, in most cases w.r.t. a newly-proposed method. By contrast, our goal is to more systematically study the influence of the number of objectives on the behavior of some representative algorithms. Furthermore, usually smaller numbers of objectives are used in reported experiments. 

Although most studies are based on a number of pairwise comparisons of solutions, it is important to notice that the elementary operation for complexity results reported below is a pairwise comparison per objective. This choice is motivated by the fact that we want to highlight the effect of the number of objectives ($m$) on different multi-objective optimization tools and methods.

\subsection{Dominance Test and Updating the Pareto Archive}
\label{sec:DomTestUpdate}

In this section, we consider the processes of testing if a solution $\mathbf{\bx}$ is dominated or not by a Pareto archive and of updating this archive. Updating the Pareto archive $A$ with a new solution $\mathbf{\bx}$ means that all solutions dominated by $\mathbf{\bx}$ are removed from $A$ and $\mathbf{\bx}$ is added to $A$ if it is not dominated by any solution in $A$. The complexity analysis of the two processes is the same, since the dominance test is the bottleneck part of the updating process.  

The simplest data structure for the Pareto archive is a simple, unorderd list of solutions with linear time complexity of update. Several methods and related data structures aiming at efficient realization of the Pareto archive update have been proposed, e.g., Quad~Tree~\citep{Sun1996,Mostaghim2005,Sun2006,Sun2011, Fieldsend2020}, MFront~II~\citep{Drozdik2015}, BSP~Tree~\citep{Glasmachers2017}, and ND-Tree~\citep{Jaszkiewicz2018}.
\citet{Jaszkiewicz2018} reported some complexity results for ND-Tree:
\begin{itemize}
\item Worst case: $\bigO(m \; N)$;
\item Best case: $\Theta(m \; \log(N))$;
\item Average case: $\Theta(m \; N^b)$, where $N$ is the size of the archive, and $b \in [0,1]$ is the probability of branching (i.e. the probability of testing more than one subnode in ND-Tree). 
\end{itemize}
Note that linear time complexity in the worst case holds also for Quad~Tree~\citep{Sun1996} and the sublinear time complexity in average case could also be obtained with BSP~Tree~\citep{Glasmachers2017}.


The Pareto archive may be either bounded in size or unbounded, i.e. contain only some or all non-dominated solutions generated so far \citep{Fieldsend2003}. In the latter case, the size of the Pareto archive may, in general, grow exponentially with the number of objectives (see Section~\ref{sec:pos}). Assuming that $N = \bigO(c^{m-1})$ (see Section~\ref{sec:pos}), the complexity of the update process of an unbounded archive becomes:

\begin{itemize}
\item Worst case:	$\bigO(m \; N) = \bigO(m \; c^{m-1})$;
\item Best case:	$\Theta(m \; \log(N)) = \Theta(m \; \log(c^{m-1}) ) = \Theta(m^2 \; \log(c) )$;
\item Average case:	$\Theta(m \; N^b) = \Theta(m \; c^{(m-1) b})$. 
\end{itemize}
In other words, in the worst case, the time grows exponentially with $m$, and, in the average case, the time may grow exponentially with $m$ if $c^b > 1$.

To illustrate the above analysis, we perform the following computational experiment. We consider 
multi-objective \mbox{NK-landscapes}
with $n=16$, $k=0$, and $m \in \{3, 4, \ldots, 20\}$. %
For each setting, we generate $30$ random instances. %
All $2^{16}=65\,536$ solutions are processed in random order. Note that in this experiment we use instances with $n=16$ instead of $n=10$ used in other places, because with $n=10$ the number of solutions was too low to show significant differences between the evaluated methods. The methods used in this experiments are simple list, ND-Tree~\citep{Jaszkiewicz2018}, Quad~Tree (precisely Quad~Tree 2 algorithm described by~\citet{Mostaghim2005} with the corrections described by~\citet{Fieldsend2020}), and MFront~II~\citep{Drozdik2015} with the modifications proposed by~\citet{Jaszkiewicz2018}. We used C++ implementations of these methods described by~\citet{Jaszkiewicz2018}, however, implementation of the Quad~Tree has been improved on both technical level and using the corrections proposed by~\citet{Fieldsend2020}. Fig.~\ref{fig:NDAll} (left) presents average running times of the four methods needed to process all solutions. As could be observed, the ranking of the methods depends on the number of objectives in a non-trivial way. Overall, the two best methods are Quad~Tree and ND-Tree, with Quad~Tree being better for $m=\{12,\dots,17\}$ objectives, and ND-Tree being better in other cases. Somehow surprisingly, simple list performs very well for small numbers of objectives and is the best method for $m=3$ and $m=4$ objectives. This is related to the design of the experiment in which all solutions are used. For small numbers of objectives, there are relatively few non-dominated solutions (see Fig.~\ref{fig:pos}), so the list is relatively short and a new solution has a high chance to be dominated by many solutions in the list. Thus, a dominating solution is often quickly found and the update process is finished. Note also, that MFront~II is the worst performing method. 

It could also be observed that the running time grows exponentially with the number of objectives. This could be however, caused just by the growing size of the Pareto archive (see Section~\ref{sec:pos}, Fig.~\ref{fig:pos}). Thus, in Fig.~\ref{fig:NDAll} (right) we present running times divided by the number of Pareto optimal solutions. The relative running time changes only slightly for Quad~Tree and ND-Tree again with a non-trivial pattern.

\begin{figure}[!t]
\centering
	\subfloat{
      \includegraphics[width=0.5\columnwidth]{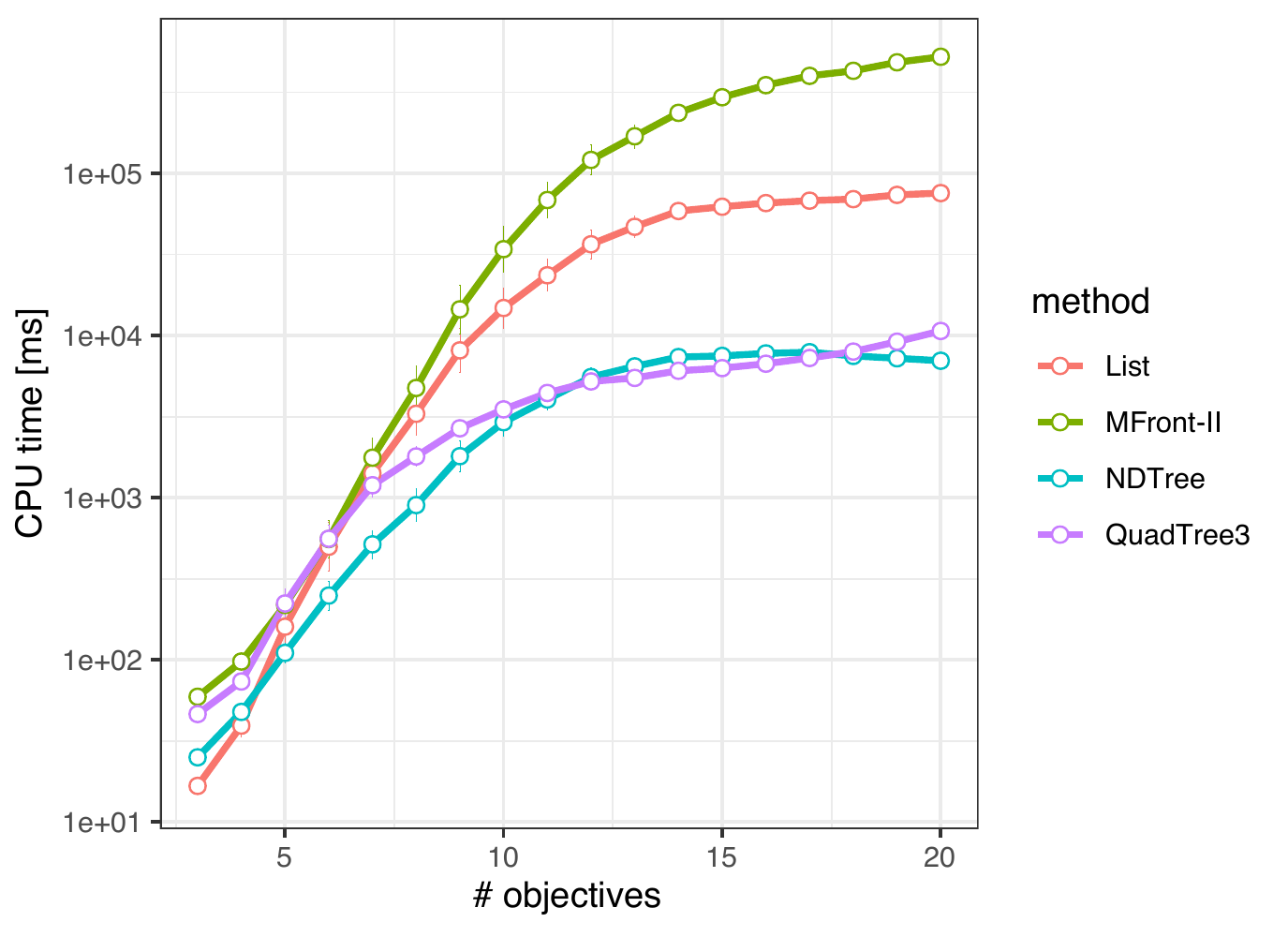}
	}
	\subfloat{
      \includegraphics[width=0.5\columnwidth]{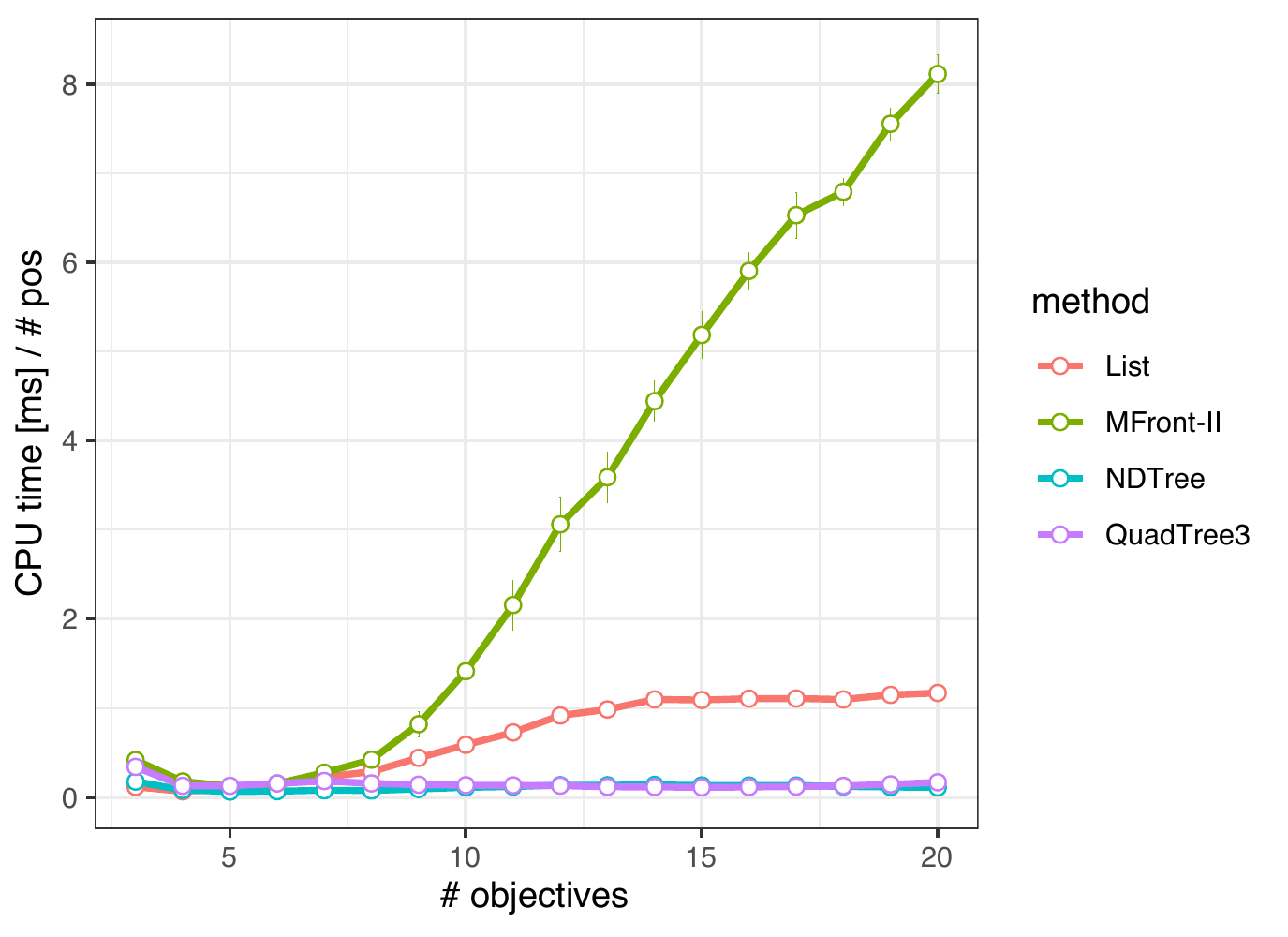}
	}
\caption{Running time of four methods for updating the Pareto archive (left, logarithmic scale), 
and running time divided by the number of Pareto optimal solutions~(right).
}
\label{fig:NDAll}
\end{figure}


Furthermore, as we discussed above in this Section, methods like ND-Tree aim at ensuring sublinear time complexity w.r.t. the size of the Pareto archive. To verify this, we measure running times needed to achieve the Pareto archive of $10\%,20\%,\dots,100\%$ of the final size. To be precise, the running time reported e.g. for $20\%$ is the running time from the moment of achieving $10\%$ of the final size till the moment of achieving $20\%$ of the final size divided by the number of processed solutions (in other words, we report average time of processing a single solution in a given interval of the Pareto archive size). Ideally, this running time should be constant or grow sublinearly with the size of the Pareto archive. The results for 20 objectives are presented in Fig.~\ref{fig:NDAllvsPos20}. Since the running times of the four methods differ significantly (see Fig.~\ref{fig:NDAll}), the running times of each method are normalized such that the maximum time is $1$. As it can be observed, running time needed to process a single solutions for 20 objectives grows linearly with the size of the Pareto archive for all methods. Of course, the speed of this growth differs for different methods. On the other hand, for some smaller numbers of objectives the running times are indeed almost constant or clearly sublinear (see Fig.~\ref{fig:NDvsPoints}\footnote{The 'jump' or running time observed for Quad tree with 20\% of the Pareto archive size and small number of objectives ($m=3$ in particular) is probably because the initial tree must be built before the method achieves full performance.}). This means that for each method, for some number of objectives, we observe a switch of its behavior from sublinear to approximately linear dependence on the size of the Pareto archive. The switching point is different for different methods and is between $5-7$ objectives for the simple list and MFront~II, between $6-8$ objectives for \mbox{ND-tree}, and between $4-6$ objectives for Quad tree. In other words, since the number of Pareto optimal solutions grows in general exponentially with the number of objectives (see Section~\ref{sec:pos}), according to the presented experiment, for at least $8$ objectives we can expect exponential growth (w.r.t. $m$) of updating an unbounded Pareto archive independent on the method used. In our opinion, it is caused by the reduced discriminative power of the dominance relation (see Section~\ref{sec:dom_rel}) for higher numbers of objectives. Although different methods are based on different ideas, all of them use some properties of dominance relation to speed-up the update process, thus all of them are affected by its reduced discriminative power.

\begin{figure}[!t]
\centering
      \includegraphics[width=0.5\textwidth]{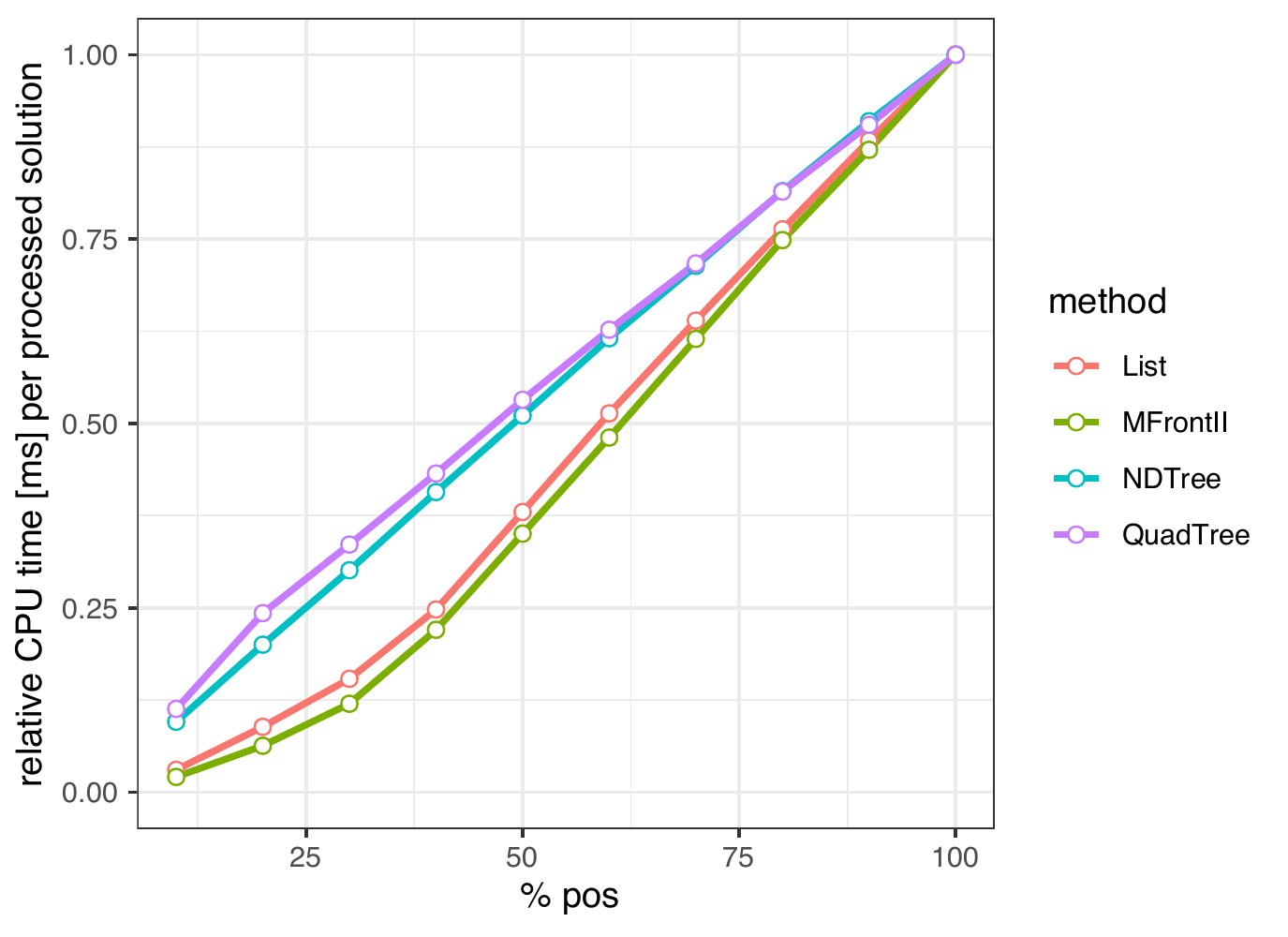}
\caption{Relative running time of Pareto archive for processing a single solution w.r.t. to the Pareto archive size for $m=20$ objectives.
}
\label{fig:NDAllvsPos20}
\end{figure}

\begin{figure}[!t]
\centering%
\includegraphics[width=\textwidth]{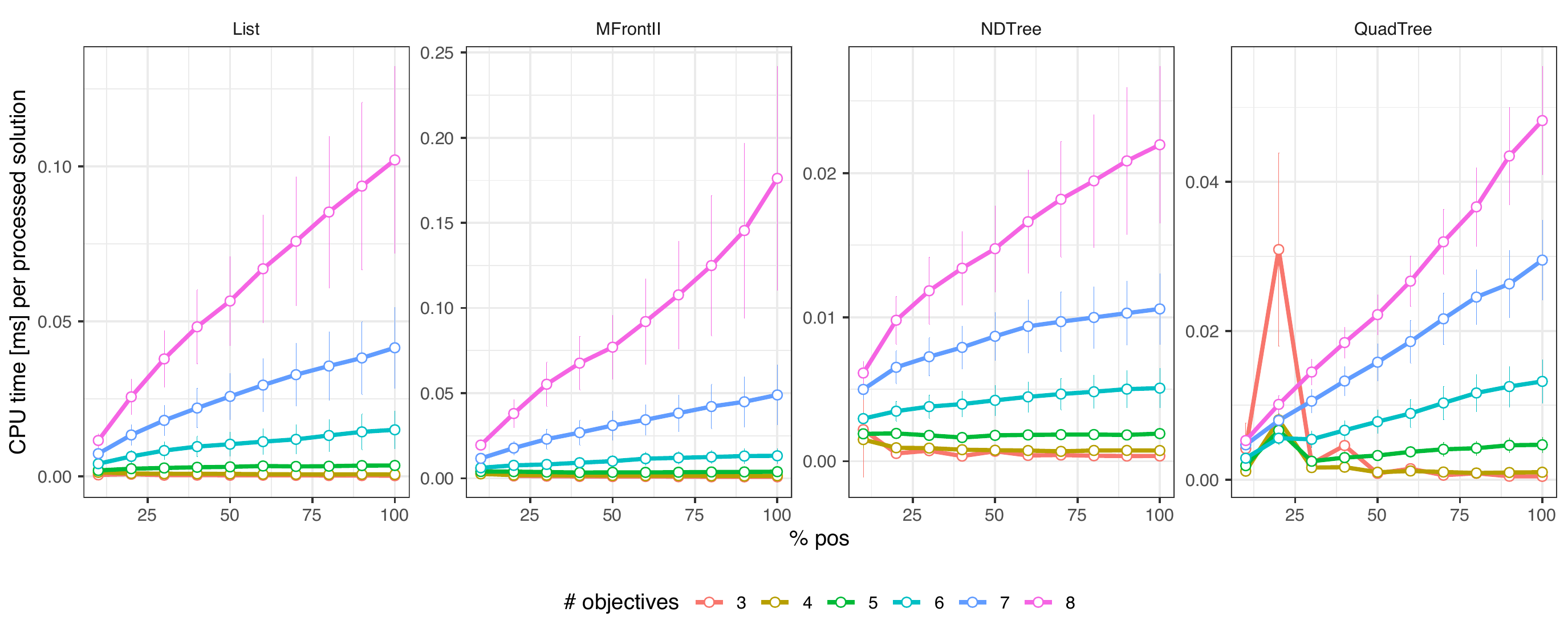}
\caption{Running time of Pareto archive for processing a single solution w.r.t. to the relative Pareto archive size for smaller numbers of objectives $m \in \{3, 4, \ldots, 8\}$.
}
\label{fig:NDvsPoints}
\end{figure}

The methods tested in our experiment together with several other methods were also recently experimentally evaluated by~\citet{Fieldsend2020}. The conclusion of that experiment was that ND-Tree generally performs best in long runs, but this was not a universal finding. Furthermore, the performance of the different methods can vary considerably between hardware architectures, and in a non-constant fashion. Also, our additional preliminary experiments (not reported in this paper) indicate that performance may strongly depend on the data sets used and the order in which the solutions are processed. In particular, the running time may depend on whether the new solution is dominated, non-dominated or dominating, and the pattern may be different for different methods. Thus, we encourage further experiments testing practical behavior of the methods for updating the Pareto archive, in particular using different orders of solutions, which is however, out of the scope of this paper.


\subsection{Computing and Approximating Hypervolume}
\label{sec:hv}
When assessing the performance of multi-objective optimization algorithms, or in indicator-based evolutionary multi-objective optimization, the indicator-value of a set of solutions is to be computed multiple times. One of the recommended and  most-often used indicator is the hypervolume~\citep{Zitzler2003}, because of its compliance with the comparison of sets of points based on the dominance relation. Note that, since dominated solutions do not influence the hypervolume, only non-dominated solutions, forming a Pareto archive of size $N$, need to be taken into account. Unfortunately, the exact hypervolume computation time is known to grow exponentially with the number of objectives. 
Although efficient algorithms exist for $m=3$~\citep{Beume2009} and $m=4$~\citep{Guerreiro2018}, the best known algorithm for the general case has a time complexity of $\mathcal{O}(N^{m/3} \: \textrm{polylog} \: N)$~\citep{Chan2013}. 

To analyze the above theoretical results experimentally, we test the behavior of two state-of-the-art exact methods for exact computation of hypervolume, namely the non-incremental version of the Hypervolume Box Decomposition Algorithm (HBDA-NI)~\citep{LACOUR2017347} and the Improved Quick Hypervolume algorithm (QHV-II)~\citep{JASZKIEWICZ2018b}. The results published in these papers are used, and the running times of HBDA-NI have been divided by 2.5 to compensate processor differences as suggested in~\citep{JASZKIEWICZ2018b}. The results are reported for data sets proposed in~\citep{LACOUR2017347} (concave, convex, and linear) composed of $1\,000$ points in $m \in \{4, 6, 8, 10\}$ dimensional objective spaces, with objective values (and thus hypervolume) normalized to the range $[0,1]$. We use 10 instances for each combination of data set type and number of objectives. The results are presented in Fig.~\ref{fig:Exact}. Note that for the running time w.r.t. the number of objectives, a logarithmic scale is used. The running times of the exact methods grow indeed exponentially with the number of objectives, which suggests that the exponential time complexity holds not only in the worst case, but in the typical case as well. At the same time the running times grow relatively slowly with the number of points (see Fig.~\ref{fig:Exact}, right). 

\begin{figure}[!t]
\centering
	\subfloat{
      \includegraphics[width=0.5\textwidth]{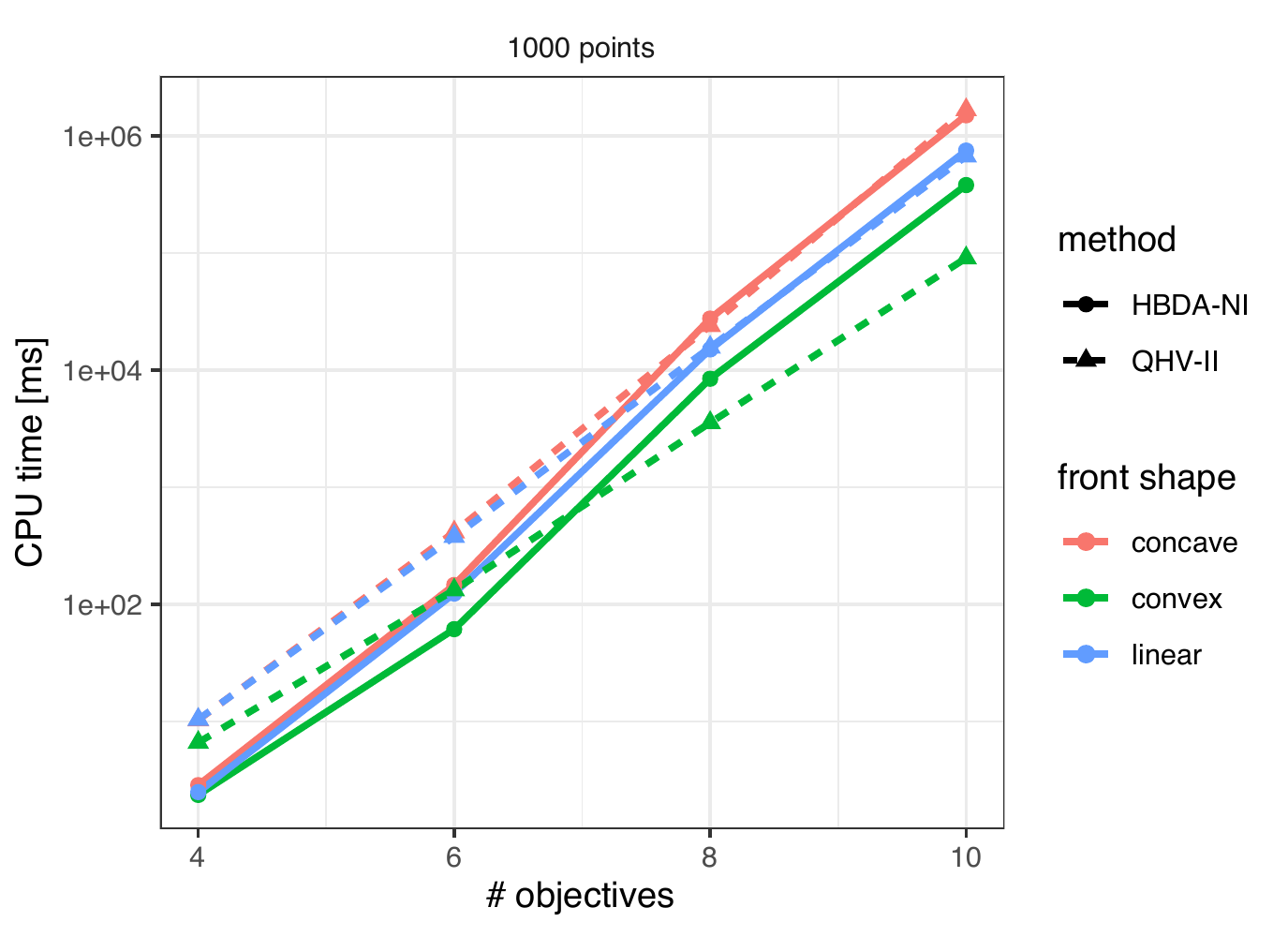}
	}
	\subfloat{
  \includegraphics[width=0.5\textwidth]{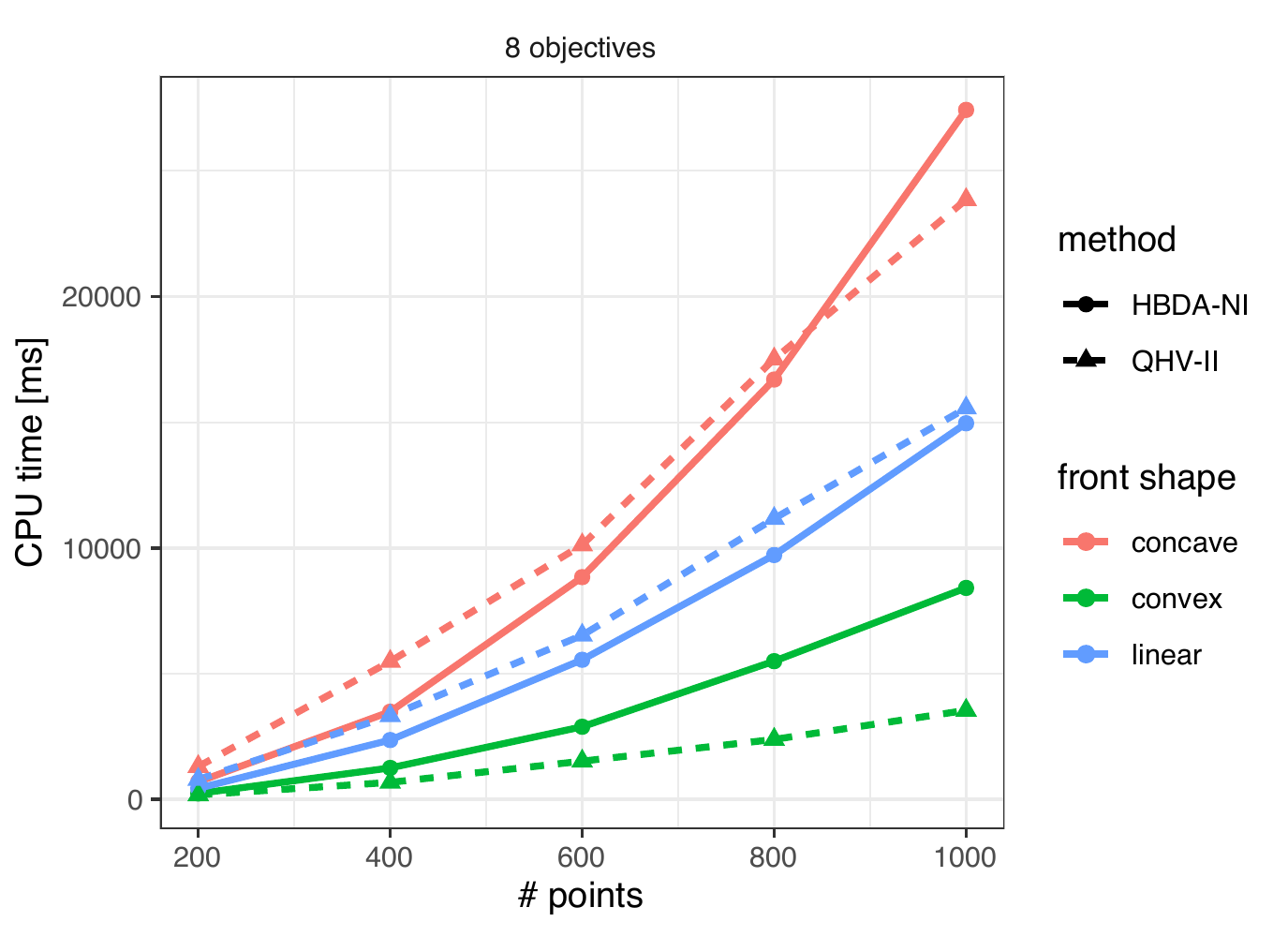}
  	} 
	\\
\caption{Running time of exact hypervolume computation methods w.r.t. the number of objectives (left, logarithmic scale) and the number of points for instances with $m=8$ objectives (right).}
\label{fig:Exact}
\end{figure}

Alternatively, the hypervolume can be approximated by Monte Carlo sampling~\citep{bader2011}. In this case, the complexity is $\Theta(s \; m \; N^b)$ 
, where $s$ is the number of sampling points 
and the term $N^b$ relates to the dominance tests. Since Monte Carlo sampling is just a sequence of $s$ independent experiments, each asking a yes/no question (dominated/non-dominated), the confidence interval can be derived from the binomial distribution and does not depend on the number of objectives. On the other hand, the question remains if the size of the confidence intervals should be reduced with growing number of objectives $m$ and/or growing $N$.

To analyze the above question, we consider first the case where the size $N$ of the set of points for which hypervolume is estimated is constant and only the number of objectives changes. The required precision of hypervolume estimation could be related to an average contribution of a single point to the hypervolume, i.e. the difference in the hypervolume with and without this single point. In other words, intuitively speaking, the precision of hypervolume estimation should be sufficient to distinguish addition/subtraction of a given number of points. It could be expected that the average contribution does not depend on the number of objectives. To analyze this experimentally, we use the same data sets as above. For each individual point its contribution to the hypervolume of the given set is calculated with an exact method. The results presented in Fig.~\ref{fig:HVContr} are average contributions for $10 \times 1\,000 = 10\,000$ individual points. The influence of the number of objectives depends mainly on the instance type and only in the case of convex instances the average contribution decreases with a growing number of objectives. However, note that for convex instances, the hypervolume approaches $0$ as the number of objectives grows (i.e. the points are very close to the nadir point), thus the contribution of an individual point also approaches $0$. On the other hand, the average contribution grows with the number of objectives for linear instances. Summarizing, no general trend for the decrease or increase of the average contribution with the growing number of objectives could be concluded from this experiment, thus we conclude that also the size of the confidence interval should not change with the number of objectives and constant $N$.

Additionally, as discussed above, it is often assumed that $N$ increases with $m$. In this case, it could be expected that the average contribution is $\Theta(1/N)$. To test this assumption, we show the results in Fig.~\ref{fig:HVContr} for the same set of instances with $200, 400, \ldots, 1\,000$ points. As it can be observed, the average contribution indeed decreases approximately linearly with $1/N$. Since we propose to treat the average contribution as the indicator of required precision, the size of the confidence interval $r$  (i.e. the difference between the upper and lower confidence bound) for the estimated hypervolume should be $\Theta(1/N)$. The confidence interval for binomial distribution could be calculated with Wilson score interval with continuity correction for binomial distribution \citep{Newcombe1998}. According to this test, $r = \Theta(1/\sqrt{s})$, where $s$ is the number of sampled points, so $s = \Theta(1/{r^2})$.
Thus, in order to achieve the required accuracy, the number of sampled points $s$ should be $\Theta(N^2)$. 

\begin{figure}[!t]
\centering
	\subfloat{
      \includegraphics[width=0.5\textwidth]{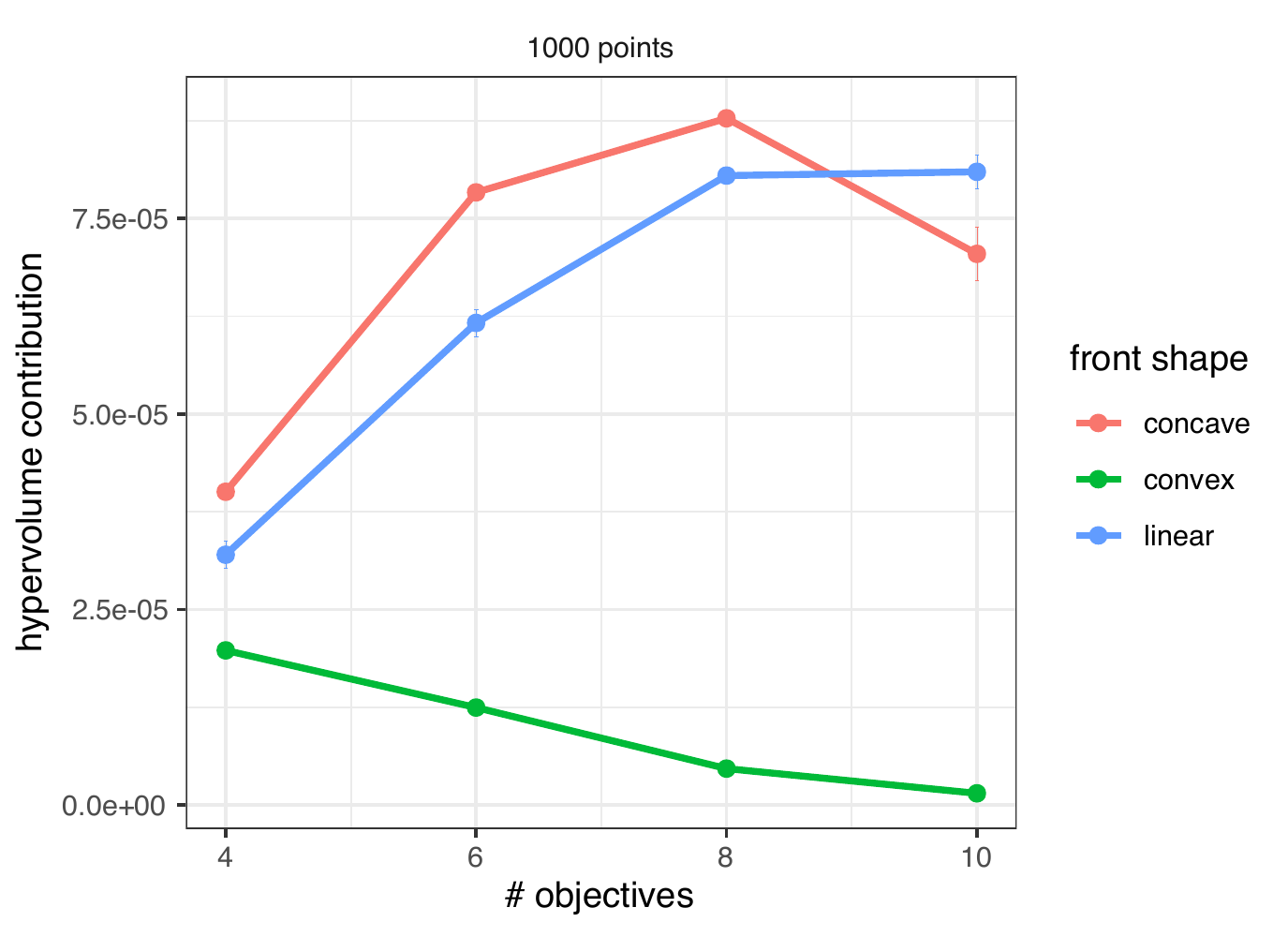}
	}
	\subfloat{
  \includegraphics[width=0.5\textwidth]{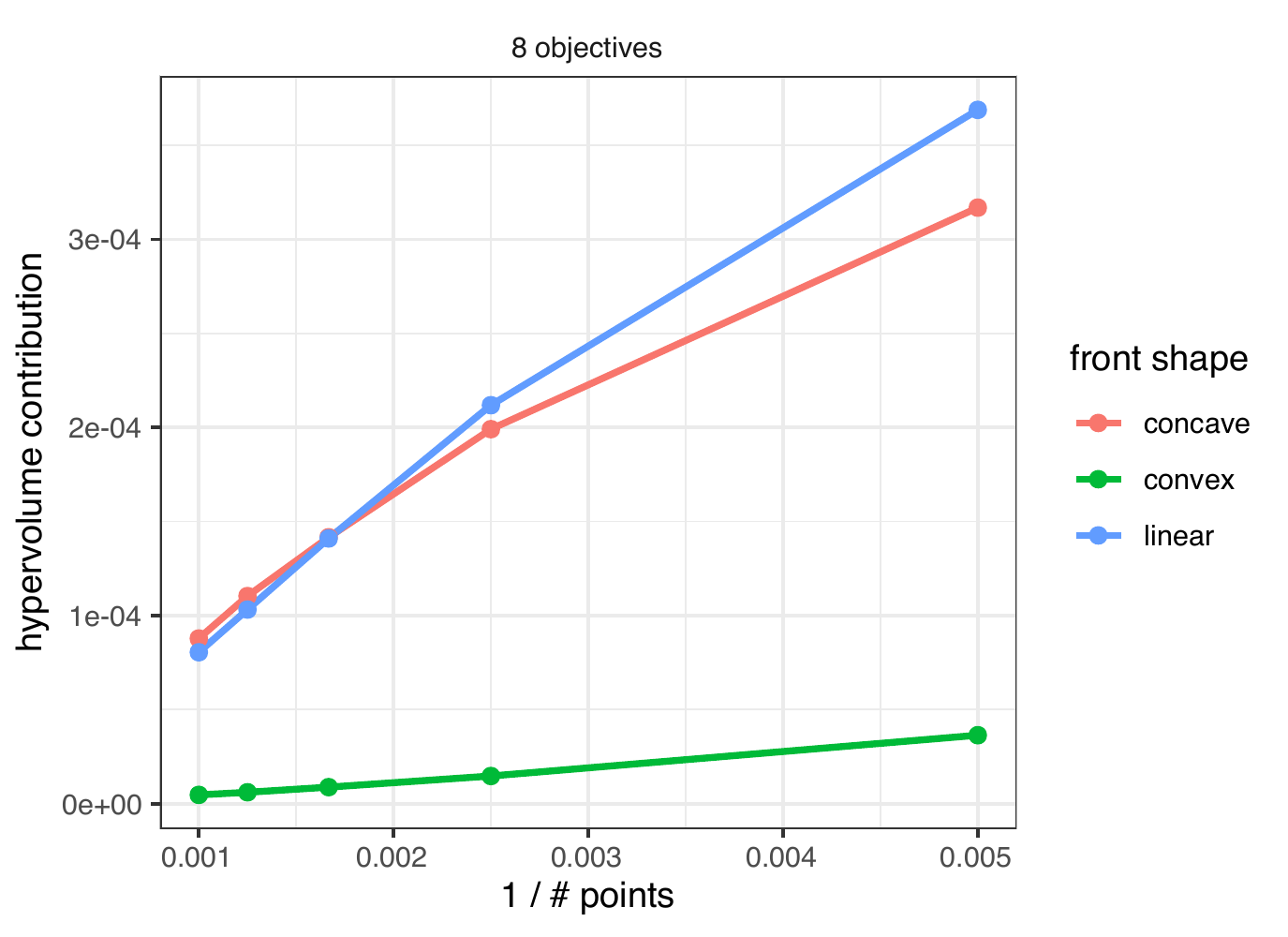}
  	} 
	\\
\caption{Average contribution of individual solutions to the hypervolume w.r.t. the number of objectives (left) and the number of points (right).}
\label{fig:HVContr}
\end{figure}


To test the above hypothesis, we perform the following experiment. Using the same data sets we run Monte Carlo sampling until it achieves size of the confidence interval lower or equal to $5/N$. We use the simplest version of Monte Carlo sampling in which the hypercuboid defined by the nadir and ideal points is uniformly sampled. ND-tree is used for the dominance test. The results are presented in Fig.~\ref{fig:MC}. As can be observed, for linear and concave instances the running times grow slowly with the number of objectives, which is due to the increasing number of comparisons of objective values. For convex instances the running time of Monte Carlo sampling decreases with the growing number of objectives, because for this set of instances with high numbers of objectives, the values of hypervolume are very close to $0$ which makes the estimation task easier. We test also the influence of the number of points on the running time for $8$ objective instances (Fig.~\ref{fig:MC}) together with quadratic trend lines which fit the data very well. Note, however, that the running times of exact methods are shorter than that of Monte Carlo sampling for $4$ objectives and comparable for $6$ objectives. Only for $8$ and $10$ objectives Monte Carlo sampling is clearly beneficial in the presented experiment.
Furthermore, since the required sample size, and thus CPU time, is $\Theta(1/{r^2})$, the running time of Monte Carlo may be much higher if a smaller confidence interval is required. For example, if the required confidence interval is $1/N$ (i.e. 5 time smaller), the running time of Monte Carlo sampling would grow about 25 times and the running times of exact methods would be larger only for 10 objectives.

\begin{figure}[!t]
\centering
	\subfloat{
      \includegraphics[width=0.5\textwidth]{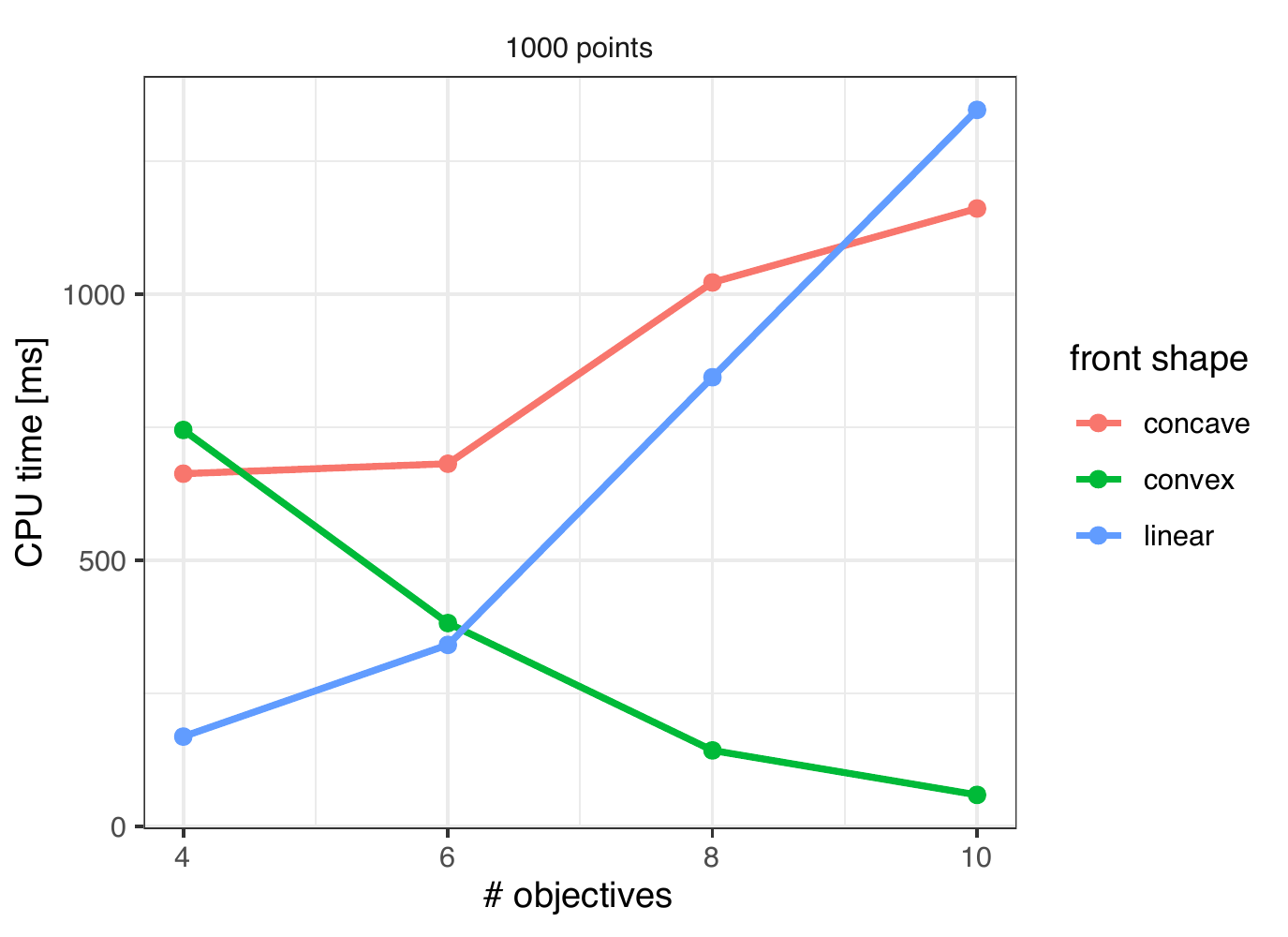}
	}
	\subfloat{
  \includegraphics[width=0.5\textwidth]{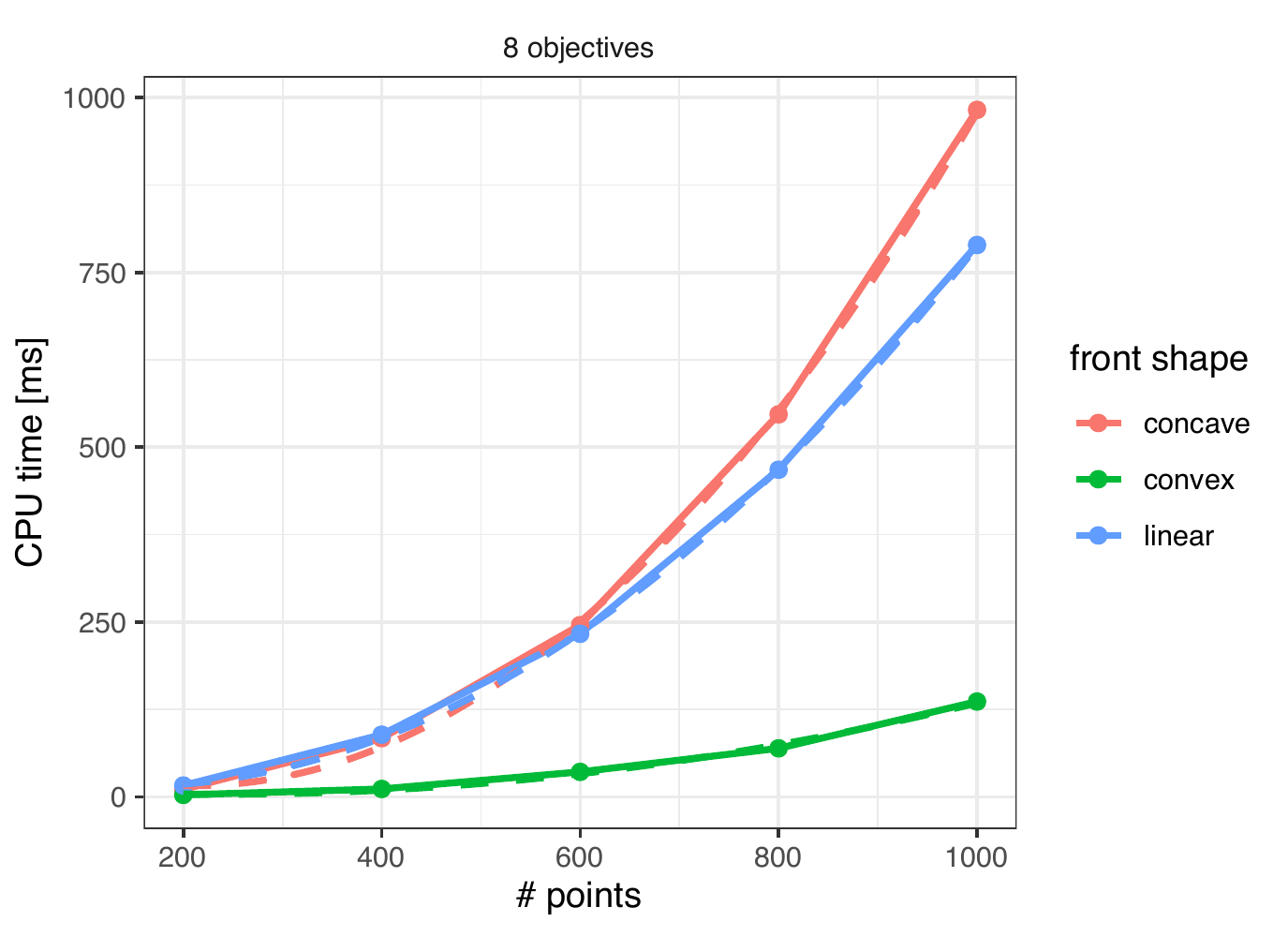}
  	} 
	\\
\caption{Running time of Monte Carlo sampling for hypervolume estimation w.r.t. the number of objectives (left) and the number of points~(right).}
\label{fig:MC}
\end{figure}

Summarizing, both the theoretical analysis and the presented experiments suggest that Monte Carlo sampling is an interesting alternative to exact methods for many-objective problems, if an approximate value of hypervolume is sufficient in a given context.

\subsection{Impact on Scalarization Methods}\label{s-furtherscal}

We consider a MO problem with the vector-valued objective function $\mathbf{\bob}$ and the feasible set $X$ (see Definition \ref{d-MO}). Here, we are looking for Pareto optimal elements with respect to $\mathbf{\bob}$ and a nontrivial, closed, convex and pointed cone $C\subset \mathbb{R}^{m }$ (see Definition \ref{d-fCmax}). 
Because of technical reasons, scalarization methods are formulated as minimization problem. Taking into account that a Pareto optimal (maximal) solution with respect to $\mathbf{\bob}$ and $C$ in the sense of Definition \ref{d-fCmax} is a Pareto optimal (minimal) solution with respect to $-\mathbf{\bob}$ and $-C$ and vice versa (see Remark \ref{r-maxmin}), we replace the MO problem by the equivalent problem to find Pareto optimal (minimal) solutions with respect to $-\mathbf{\bob}$ and $-C$. For simplicity, we put $\mathbf{\bob}:= -\mathbf{\bob}$ and denote the corresponding minimization problem as MO problem too.
 
An important question is how to characterize the sets of Pareto (weakly) optimal elements of MO problems via suitable scalarization methods. In the non-convex case, especially in combinatorial MO, it is important to provide suitable non-linear scalarization functionals. This section discusses the influence of the number of objectives on scalarization procedures. Scalarization methods are often used for deducing characterizations of solutions of MO problems and for deriving algorithms for solving these problems (see~\cite{eich08} and references therein for an overview). 

A prominent scalarization method is the well known weighted Chebyshev scalarization:
\begin{equation}\label{weighTsch} 
\min_{\bx \in X} \max_{i \in \{1, \ldots ,m\}} \lambda^*_i ( f_i (\bx) - w_i ) 
\end{equation}
with $\lambda^*_i \in \operatorname*{int} \mathbb{R}^{m }_+ $, $w_i \in  \mathbb{R}$ ($i=1, \ldots , m$). 

Another important scalarization method is the $\epsilon$-constraint problem (with $\epsilon = (\epsilon_1, \ldots , \epsilon_m  ) \in \mathbb{R}^{m }$), see \citet{Hai71}:
\begin{align}\label{eps}
\begin{split}
& \min f_j(\bx) \\
& \text{subject to } f_i(\bx) \leq \epsilon_i,~i=1, \ldots, m  ,~ i \neq j,\\
& \hspace{1.7cm}{\bx \in X\subseteq\mathbb{R}^n.}
\end{split}
\end{align}

In the following section, we will explain that the scalarization functionals involved in (\ref{weighTsch}) and (\ref{eps}) are special cases of a general scalarization functional (see (\ref{funcak0}) in the next section) such that it is possible to obtain continuity, monotonicity, convexity properties of the scalarization functionals in (\ref{weighTsch}) and (\ref{eps}) from the corresponding properties of the general non-linear functional derived in \cite[Theorem 2.3.1]{GopRiaTamZal:03}. These properties are important for a characterization of the solutions to MO problems.


\subsubsection{A General Scalarization Technique}\label{sec-general}

We will formulate a general non-linear scalarizing functional (see \cite[Section 2.3]{GopRiaTamZal:03}) and explain that well known scalarization techniques (weighted Chebyshev scalarization, weighted sum scalarization, $\epsilon$-constraint problem, scalarization by \citet{Kal94}) can be described by special cases of this general scalarizing functional.

For a nonempty closed subset $A$ of $\mathbb{R}^{m }$ and an element $k\in -0^+A \setminus \{0\}$ ($-0^+A$ denotes the negative recession cone to $A$), we study a {non-linear functional} $\varphi_{A,k} : \mathbb{R}^{m } \rightarrow \overline{\mathbb{R}}:=\mathbb{R} \cup \{-\infty\} \cup \{+ \infty \}$ defined by
\begin{equation}\label{funcak0}
\varphi_{A,k} (y):= \inf \{t\in
{\mathbb{R}} \mid y\in tk + A\}.
\end{equation}

Now, we consider the scalarized MO problem using the non-linear functional (\ref{funcak0}):
\[\tag{$\rm{MO}(\varphi_{A,k})$} \min_{\bx\in X}  \varphi_{A,k} (\mathbf{\bob}(\bx)).\]
If we consider a nontrivial, closed, convex and pointed cone $C \subset \mathbb{R}^m$ and suppose $A - C \subseteq A$, a solution of ($\rm{MO}(\varphi_{A,k})$) is a weakly optimal solution of the MO problem taking into account the monotonicity properties of the functional (\ref{funcak0}), see \cite[Theorem 2.3.1 (f) and Theorem 3.1.8]{GopRiaTamZal:03}.

\bigskip

We explain that many well known scalarization methods for multi-objective optimization problems are special cases of a scalarization by means of the non-linear functional (\ref{funcak0}). 
\begin{itemize}

\item[(a)] We show that the {$\epsilon$-constraint method} can be described by means of functionals of type (\ref{funcak0}): Let some $j \in \{1,\ldots, m  \}$ and
some real values $\epsilon_i \in \mathbb{R}$, $i=1,\ldots, m  , i\not=j$ be given. Then, the $\epsilon$-constraint
scalarization (see \citep{Hai71,ChaHai83,eich08,DurStrTam2017}), is given by the functional $\varphi_{A,k}$ (see (\ref{funcak0})) with
\begin{equation}
A :=  \bar{b} - \mathbb{R}^{m }_+ , \mbox{ with } \bar{b}=(\bar{b}_1, \ldots, \bar{b}_{m })^T, 
\bar{b}_i = \left\{ \begin{array}{ll} 0 & \text{for } i=j,\\
                                       \epsilon_i  & \text{for } i \neq j,
 \end{array}\right. \label{D}  
\end{equation}
\begin{equation}\label{k} k=(k_1, \ldots, k_{m }) \mbox{ where } k_i = \left\{ \begin{array}{ll} 1 & \text{for } i=j,\\
                                     0 & \text{for } i \neq j.
\end{array}\right. 
\end{equation}
 
We denote the functional (\ref{funcak0}) with $A$ and $k$ as in (\ref{D}), (\ref{k}) by $\varphi^\epsilon_{A,k}$.
With $A$ and $k$ given by (\ref{D}), (\ref{k}), the scalarized MO problem 
\[\tag{$\rm{MO} (\varphi^\epsilon_{A,k})$} \min_{\bx\in X} \varphi^\epsilon_{A,k} (\mathbf{\bob}(\bx))\]
describes the $\epsilon$-constraint problem (\ref{eps}).

Taking into account \cite[Theorem 2.3.1 (f) and Theorem 3.1.8]{GopRiaTamZal:03}, solutions of (\ref{eps}) generate weakly optimal solutions of the MO problem.

If the number $m$ of objective functions is increasing, then the number of parameters $\epsilon_i$ ($i=1 , \ldots , m$) involved in the constraints in the definition of the functional $\varphi^\epsilon_{A,k}$ (i.e., the number of constraints) is growing up.

\item [(b)] Let a set $B_{L}$ be given by a system of linear inequalities. We consider
\begin{equation}\label{f-bGamma} B_{L} := \{y \in \mathbb{R}^{m }\; | \;\ \langle a^i ,  y \rangle   \leq \alpha_i,\ a^i \in \mathbb{R}^{m },\ \alpha_i \in \mathbb{R},\ i \in \{1, \ldots , n \}\}. 
\end{equation}
Using $a^i$ from this formula for $B_{L}$, we define a set $A_{L}
\subset \mathbb{R}^{m }$ by
\begin{equation}
   \label{equ:MengeA}
   A_{L} := \{y \in \mathbb{R}^{m }\; | \;\ \langle a^i ,  y \rangle   \leq \alpha_i,\ i \in I\}  \end{equation}
with the index set
\begin{equation}
   \label{equ:IndexmengeI}
   I := \{i \in \{1, \ldots , n\}\; | \;\ \{y \in \mathbb{R}^{m }\mid \langle a^i ,  y \rangle   = \alpha_i\} \cap B_{L} \cap \operatorname{int} \mathbb{R}^{m }_+ \neq \emptyset\}.
\end{equation}
The set $I$ is exactly the set of indices $i \in \{1, \ldots , n \}$ for
which the hyperplanes $\langle a^i ,  y \rangle   = \alpha_i$ are active in the non-negative orthant.
Let $B_{L}$
and the corresponding set $A_{L}$ defined as in (\ref{equ:MengeA}),
let vectors $k \in -0^+A_{L} \setminus \{0\}$ and $w \in \mathbb{R}^{m }$ be given. We consider the functional $\varphi_{A,k}$ (see (\ref{funcak0})) with $A= w + A_{L}$, i.e., we study the  functional $\varphi_{w + A_{L} ,k}: \mathbb{R}^{m } \rightarrow \overline{\mathbb{R}}$ of type (\ref{funcak0}) given by
\begin{equation}
\label{equ:z1Def}
   \varphi_{w + A_{L} ,k}(y) = \inf \{t \in \mathbb{R}\; | \;\ y \in t k + w + A_{L} \}, \qquad y \in \mathbb{R}^{m }.
\end{equation}
The functional $ \varphi_{w + A_{L} ,k}$ depends on the set $A_{L}$ and the parameters $k$ and $w$.

Using the non-linear functional $\varphi_{w + A_{L} ,k}$ given by (\ref{equ:z1Def}), we consider the scalarized problem
\[\tag{$\rm{MO}( \varphi_{w + A_{L} ,k} )$} \min_{\bx\in X} \varphi_{w + A_{L} ,k} (\mathbf{\bob}(\bx)) .\]

The following assertion is shown by~\citet{TamWin03} under more restrictive assumptions.

\begin{corollary}\label{special-struc-new}
We consider the set $B_{L}$ given by (\ref{f-bGamma}) and $w \in \mathbb{R}^{m }$ arbitrarily chosen. Let the corresponding set $A_{L}$ be given by (\ref{equ:MengeA}), (\ref{equ:IndexmengeI}), $k \in -0^+A_{L} \setminus \{0\}$ and the functional $ \varphi_{w + A_{L} ,k}$  given by (\ref{equ:z1Def}). Assume that $\langle a^i ,  k \rangle  
\neq 0$ for all $i \in I$.

Then, the non-linear functional $\varphi_{w + A_{L} ,k}$ in
(\ref{equ:z1Def}) is convex and $\mathbb{R}^{m }_+$-monotone (i.e., $y^1 \in y^2 +\mathbb{R}^{m }_+$ implies $\varphi_{w + A_{L} ,k}(y^1) \geq \varphi_{w + A_{L} ,k}(y^2)$). Furthermore, $\varphi_{w + A_{L} ,k}$ in
(\ref{equ:z1Def}) has the structure
\begin{equation}\label{f-spec-scal-block-new}
\varphi_{w + A_{L} ,k}(y) = \max_{i \in I} \frac{\langle a^i ,  y \rangle   - \langle a^i ,  w \rangle   - \alpha_i}{\langle a^i ,  k \rangle  }.  
\end{equation}
\end{corollary}
\begin{proof} We get the assertions from \cite[Theorem 2.3.1]{GopRiaTamZal:03} and
\begin{eqnarray}
   \varphi_{w + A_{L} ,k}(y) & = & \inf\ \{t \in \mathbb{R}\; | \;\ y - w - t k \in A_{L}\}  \nonumber \\
          & = & \inf\ \{t \in \mathbb{R}\; | \;\ \langle a^i ,  y \rangle   - \langle a^i ,  w \rangle   - t \langle a^i ,  k \rangle   \leq \alpha_i,\ i \in I\}  \nonumber \\
          & = & \inf\ \{t \in \mathbb{R}\; | \;\ \frac{\langle a^i ,  y \rangle   - \langle a^i ,  w \rangle   - \alpha_i}{\langle a^i ,  k \rangle  } \leq t,\ i \in I\}  \nonumber \\
          & = & \max_{i \in I} \frac{\langle a^i ,  y \rangle   - \langle a^i ,  w \rangle   - \alpha_i}{\langle a^i ,  k \rangle  }.  \nonumber
\end{eqnarray}
\end{proof}

It is important to mention, if the number $m$ of objective functions is increasing, then the scalar products in (\ref{f-spec-scal-block-new}) are taken for parameters $w$ and $k$ in higher dimensional spaces $\mathbb{R}^m$.

\item [(c)] We consider the non-linear functional $\varphi_{w + A_{L} ,k}$ defined by
formula (\ref{equ:z1Def}) with the structure (\ref{f-spec-scal-block-new}). Choosing the parameters $A_{L} =  - \mathbb{R}^{m }_+$, $k \in \mathbb{R}^{m }_+ \setminus \{0\}$, $w=0$, $\alpha_i = 0$, $a^i \in \mathbb{R}^{m }_+$ such that $\langle a^i ,  k\rangle =1$  ($i= 1, \ldots , m  $), we obtain in (\ref{f-spec-scal-block-new}):
\begin{equation}\label{Tscheb}
\varphi_{- \mathbb{R}^{m }_+ ,k} (y) = \varphi_{w + A_{L} ,k}(y) = \max_{i \in I} {\langle a^i ,  y \rangle   } .
\end{equation}
This is the well known weighted Chebyshev scalarization with weights $\lambda^*_i =1$ for all $i \in I$ and the origin as reference point $w$ (see (\ref{weighTsch})). Because (\ref{Tscheb}) is a special case of (\ref{funcak0}) with $A=-\mathbb{R}^m_+$, we know that it is possible to characterize weakly optimal elements of the MO problem by solutions of the problem (\ref{weighTsch}), see \cite[Theorem 2.3.1 (f) and Theorem 3.1.8]{GopRiaTamZal:03}.

For the simplest case of the scalarizing functional in (\ref{Tscheb}), we get the problem
\[ \min_{\bx \in X} \langle a ,  \mathbf{\bob}(\bx) \rangle   , \]
with $a \in  \mathbb{R}^{m }_+ \setminus \{0\}$, i.e., the weighted sum scalarization.

If we choose $a \in \operatorname{int} \mathbb{R}^{m }_+$, we obtain the  characterization of properly optimal elements in
the sense of the weighted sum scalarization (compare
\citet{sch70}).

\item [(d)] In \cite[Theorem 3.7]{Kal94}, Pareto optimal elements of $F \subset \mathbb{R}^{m }$ with respect to a polyhedral cone $C \subset \mathbb{R}^{m }$ are  characterized. Let a
polyhedral cone 
\[ C := \{y \in \mathbb{R}^{m }\; | \; \langle -b^i ,  y \rangle   \geq 0,\; b^i \in \mathbb{R}^{m }, \;  i \in \{1, \ldots , n \}\}   \]
be given.
It is known (see \citep{Kal94}): Consider a Pareto optimal element $\bar{y}$ of $F \subset \mathbb{R}^{m }$ with respect to $C \subset \mathbb{R}^{m }$. Then, the system
\[ \langle b^i ,  y-\hat{y} \rangle   < \langle b^i ,  \bar{y}-\hat{y} \rangle  , \qquad i \in \{1, \ldots , n \} ,\]
where $F \subset \hat{y} + C$, is inconsistent for all $y \in F$.

We chose $A_L = \bar{y} - C$ in (\ref{f-spec-scal-block-new}) and consider $\hat{F} = F - \hat{y}$ such that we get
$\hat{F} \subset C$. This is a weaker condition than $\hat{F} \subset \mathbb{R}^{m }_+$. Furthermore, \cite[Theorem 3.2 (iii)]{TamWin03} yields
\[ \forall\ y \in F:  \qquad \varphi_{w + A_L , k}(y-\hat{y}) \geq \varphi_{w + A_L , k}(\bar{y}-\hat{y}) . \]
Taking into account~(\ref{f-spec-scal-block-new}) with $w = 0$, $\alpha_i = 0$, we get for all $y \in
F$
\[ \max_{i =1, \ldots , n} \frac{\langle b^i ,  y-\hat{y} \rangle  }{\langle b^i ,  k \rangle  }
   \geq \max_{i =1, \ldots , n} \frac{\langle b^i ,  \bar{y}-\hat{y} \rangle  }{\langle b^i ,  k \rangle  }. \]
So, we obtain that for each $y \in F$, there are $i^\ast \in I$ and
\[ \frac{\langle b^{i^\ast} ,  y-\hat{y} \rangle  }{\langle b^{i^\ast} ,  k \rangle  }
   \geq \max_{i =1, \ldots , n} \frac{\langle b^i ,  \bar{y}-\hat{y} \rangle  }{\langle b^i ,  k \rangle  }
   \geq  \frac{\langle b^{i^\ast} ,  \bar{y} -\hat{y}  \rangle  }{\langle b^{i^\ast} ,  k \rangle  }, \]
i.e., $\langle b^{i^\ast} ,  y-\hat{y} \rangle   \geq \langle b^{i^\ast} ,  \bar{y}-\hat{y} \rangle  $. Now, we conclude that $\langle b^i ,  y-\hat{y} \rangle   < \langle b^i ,  \bar{y}-\hat{y} \rangle  $ for all $i = 1, \ldots , n$ is not possible.

\item [(e)] In order to generate weakly optimal solutions of a MO problem, Pascoletti and Serafini (cf. \citep{eich08}, \citep{PasSer84}) considered the following scalar
surrogate problem

\begin{align}
\begin{split}\label{PS}
&  \min \; t \\
& \text{subject to } \mathbf{\bob}(\bx) \in  tk + a -  \mathbb{R}^{m }_+,\\
& \hspace{1.7cm}{\bx \in X\subseteq\mathbb{R}^n , \; t \in \mathbb{R}}
\end{split}
\end{align}

with parameters $a \in \mathbb{R}^{m }$, $k \in \operatorname*{int} \mathbb{R}^m_+$. The scalar problem in (\ref{PS}) is a scalarization of the MO problem by means of a functional of type (\ref{funcak0}) with $A= a -  \mathbb{R}^{m }_+$ and $k \in \operatorname*{int} \mathbb{R}^m_+$, namely
\[\tag{$\rm{MO}(\varphi_{a-\mathbb{R}^{m }_+,k})$} \min_{\bx\in X}  \varphi_{a-\mathbb{R}^{m }_+,k} (\mathbf{\bob}(\bx)).\]

Taking into account \cite[Theorems 2.3.1 (f) and 3.1.8]{GopRiaTamZal:03}, solutions of ($\rm{MO}(\varphi_{a-\mathbb{R}^{m }_+,k})$) generate weakly optimal solutions of the MO problem.

\end{itemize}

\subsubsection{Scalarization Methods for Multi-objective Combinatorial Optimization}

The application of scalarization methods, especially of the {$\epsilon$-constraint method}, for combinatorial MO problems is discussed by~\citet{Fig2017}. As explained in Section \ref{sec-general}, the {$\epsilon$-constraint method} is a special case of a scalarization by means of the functional $\varphi_{A,k}$ in (\ref{funcak0}) with the parameters $A$ and $k$ given by (\ref{D}), (\ref{k}). 
Scalar optimization problems $\min_{\bx\in X} \varphi_{A,k} (\mathbf{\bob}(\bx))$ with $A$ and $k$ as in (\ref{D}), (\ref{k}) are resource-constrained combinatorial problems.  These scalar combinatorial optimization problems constitute to be NP-hard. In particular, if we consider combinatorial problems, then the additional $\epsilon$-constraints can
make $\min_{\bx\in X} \varphi_{A,k} (\mathbf{\bob}(\bx))$ based on different parameters $A$ and $k$ like in (\ref{D}), (\ref{k}) considerably more difficult to solve than the suitable single-objective problem.~\citet{Fig2017} explained that the cause for this additional difficulty is the perturbation of the combinatorial structure of the polyhedron of feasible solutions by the $\epsilon$-constraints (see (\ref{eps}) and (\ref{D})).

Furthermore, it is mentioned  in \citet{Fig2017} that there are particular cases where the property of total unimodularity is compatible with the $\epsilon$-constraint method. There are combinatorial MO problems, where total unimodularity can be preserved during the scalarization procedure. For instance, the constraint matrix
of a binary knapsack problem or the binary assignment problem is totally unimodular.

In general, it seems to be preferable to use the non-linear functional (\ref{equ:z1Def}) for scalarization. Taking into account Corollary \ref{special-struc-new}, it is possible to avoid additional $\epsilon$-constraints.

Concerning the influence of the number of objective functions:
\begin{itemize}
\item {\bf $\epsilon$-constraint method:} There are more $\epsilon$-constraints (see (\ref{eps}) and (\ref{D})) in the scalarized problem in the case that we add some objective functions.
\item {\bf scalarization by $\mathbf{\varphi_{w + A_{L} ,k}}$:} If we add some more objective functions, then the problem ($\rm{MO}( \varphi_{w + A_{L} ,k} )$) is not more difficult because the number of objective functions $m $ is only involved in the scalar products $\langle a^i ,  w \rangle$ and $\langle a^i ,  k \rangle$ for $a^i, w, k \in \mathbb{R}^{m }$ taking into account Corollary \ref{special-struc-new}.

\end{itemize}


\subsection{Distances between Weight Vectors}
\label{sec:dist_weights}

\begin{figure}[!t]
\centering%
\includegraphics[width=0.55\textwidth]{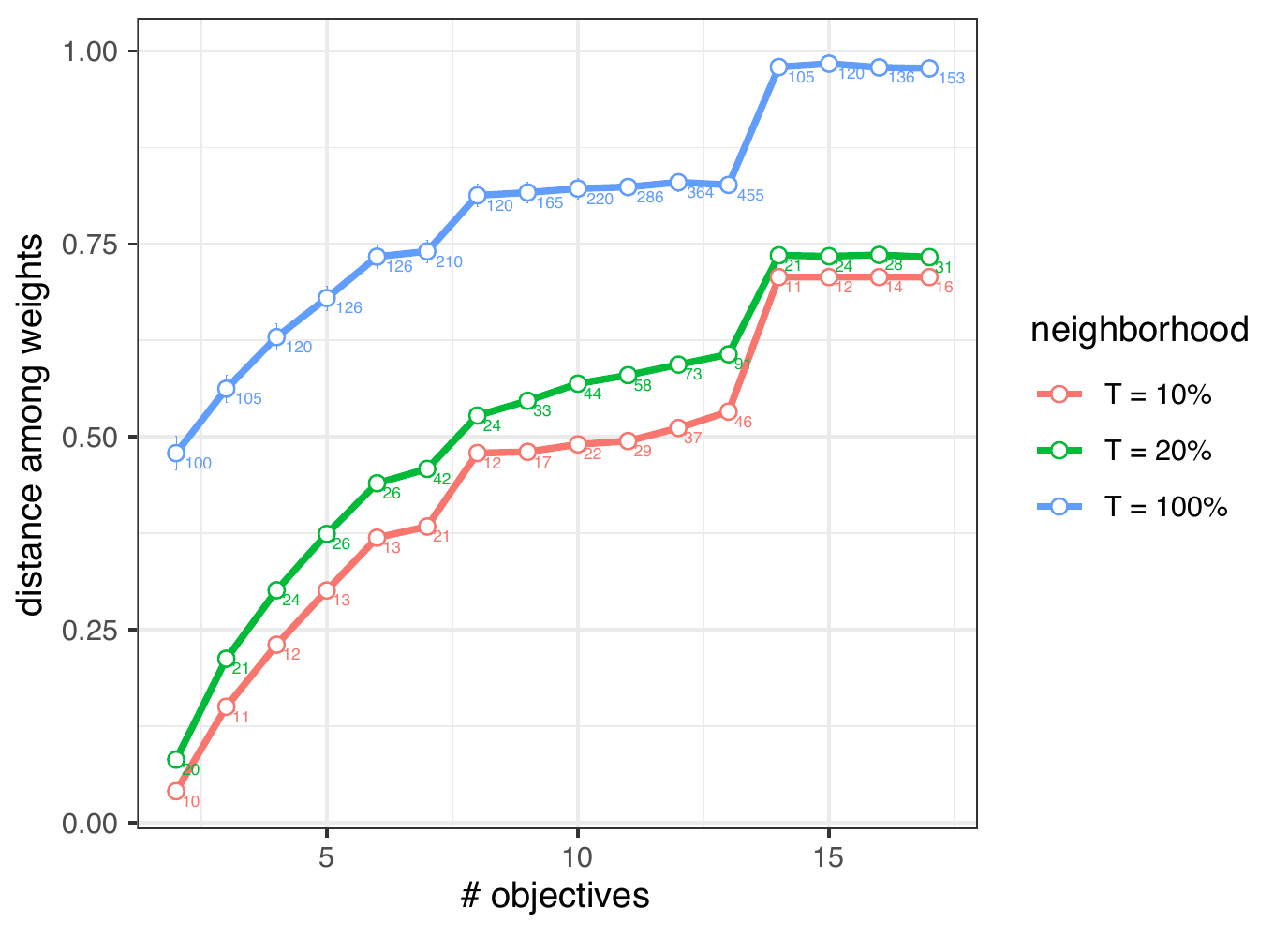}
\caption{Distance among weight vectors with respect to the number of objectives.
Below each point is reported the number of weight vectors for the considered setting.
}
\label{fig:dist_weight}
\end{figure}

In decomposition-based MOEAs, such as MOEA/D~\citep{zhang2007moea} and its variants~\citep{decompo_survey}, multiple scalar sub-problems are optimized simultaneously. Each sub-problem is defined by a particular weight vector setting for the considered scalarizing function. It is often argued that the search process actually benefits from solving the sub-problems cooperatively, given that there exists a certain locality among them, the solution from one sub-problem contributing to the solution of another sub-problem and vice-versa. In particular, a neighborhood is defined among sub-problems in order to limit the exchange of information between them.
In this section, we investigate how the distance between 
uniformly-defined weight vectors is impacted by the number of objectives.

We generate weight vectors following a simplex-lattice design \citep{simplex_lattice} as performed, e.g., in the original MOEA/D~\citep{zhang2007moea}, where the number of weights is set as the population size $\mu$.
The simplex-lattice design generates $\mu$ weight vectors 
such that $\mu = {H + m - 1 \choose m - 1}$, where $H$ is a user-defined parameter. In our experiments, for a given number of objectives~$m$, we take the smallest $H-$value such that $\mu \geq 100$.
The neighborhood of a given weight vector is defined as the set of the $T$ closest weight vectors, based on Euclidean distance. We consider three neighborhood sizes $T \in \{10\%, 20\%, 100\%\}$, given as a proportion of the total number of weight vectors $\mu$.
When $T=100\%$, this means that there is no restriction on the exchange of information between solutions; i.e. the neighborhood is made of the entire population.
For a given setting, we select $900$ pairs of weight vectors randomly (we used the same setting in Section~\ref{distanceBetweenSolutions}), each time within a given neighborhood, and we report the average Euclidean distance between them and confidence intervals in Fig.\ref{fig:dist_weight}.

We observe that the distance among weight vectors increases with the number of objectives. When no neighborhood is considered ($T=100\%$), the distance goes from about 
$0.5$ for $m=2$ to close to $1$ for $m \geq 14$. Restricting the selection of weight vectors among the $10\%$ or $20\%$ closest ones reduces the distance to an order of magnitude. However, the trend remains similar, and the distance exceeds $0.5$ for problems with more than $12$ objectives, even with a small neighborhood size of $T=10\%$.
As such, the assumption that the neighboring sub-problems share similar information becomes less accurate as the number of objective grows, and might actually affect the performance of decomposition-based MOEAs for many-objective problems.

\section{Conclusions and Future Work}\label{conclusions}

\subsection{Summary}\label{summary}
This paper has carried out a theoretical and empirical analysis of the impact of increasing the number of objectives~$m$ on (i)~the characteristics of a multi-objective (MO) problem and (ii)~the complexity of commonly used MO procedures and algorithms. For the empirical analysis, we used multi-objective NK landscapes, which allowed us to conveniently scale up the number of objectives $m$.
Table~\ref{contribution_table} provides an overview of the topics we have made contributions to, and where these can be found in the paper. The main findings of our analysis in terms of scaling efficiencies can be summarized as:
\begin{itemize}
\item \emph{Good} scaling behavior (i.e., polynomial):
	\begin{itemize}
	
	\item \textit{Impact on scalarization methods:} We have proposed a general non-linear scalarizing functional that can be used to describe several well known scalarizing techniques including weighted Chebyshev scalarization, weighted sum scalarization, Pascoletti-Serafini problem, $\epsilon$-constraint problem, and scalarization by Kaliszewski. Using the proposed non-linear functions, the complexity of solving the corresponding scalar problem grows linearly with $m$. 

	\item \textit{Approximating hypervolume:} We have shown that the complexity of computing the hypervolume exactly grows exponentially with $m$ in the typical case (not only in the worst case). When approximating the hypervolume using Monte Carlo sampling, we demonstrated empirically that the size of the confidence interval (of the approximated hypervolume value) should not change with $m$ assuming a constant Pareto archive size $N$; instead the number of sampled points should be $\Theta(N^2)$. Moreover, we showed that there is a switching point ($m\geq8$ in our study) beyond which the running time of a sampling-based approach outperforms an exact method, and the switching point moves to higher $m$ as the required confidence interval reduces in width.   
	\end{itemize}
\item \emph{Relatively good} scaling behavior:
	\begin{itemize}
	
	\item \textit{Dominance test and updating the Pareto archive:} Although, in the case of unbounded Pareto archive, the running time of updating the Pareto archive grows exponentially with $m$ for all four considered data structures -- a simple (unordered) list, ND-Tree, Quad~Tree, and MFront~II --, experiments reveal that ND-Tree and Quad~Tree are more robust. With regards to the running time of processing single solutions w.r.t to the Pareto archive size, we observed empirically a switching point in terms of $m$ beyond which the running time turns from sublinear to approximately linear dependence on the Pareto archive size. This switching point is between $5-7$ objectives for the simple list and MFront~II, between $6-8$ objectives for \mbox{ND-tree}, and between $4-6$ objectives for Quad tree.
	
	\end{itemize}
\item \emph{Poor} scaling behavior (i.e., exponential complexity, decreased quality):
	\begin{itemize}
	 
	\item \textit{Number of Pareto optimal solutions:} We have confirmed empirically that the number of Pareto optimal solutions grows exponentially with $m$, with less than 5\% of solutions being Pareto optimal for $m=2$ and this proportion growing to about 50\% for $m=7$, and more than 99\% for $m=20$ for the considered problems.
	
	\item \textit{Probability for a solution to be non-dominated:} We derived theoretically the probability that $\mu$ random pairs of solutions are mutually non-dominated, and this probability is $\left( 1-\frac{1}{2^{m-1}} \right)^{\mu}$. An empirical analysis confirmed the correctness of this probability, and showed that it becomes very likely that all solutions are mutually non-dominated for $m>15$, even for values of $\mu>1\,000$. 
	
	\item \textit{Probability of having heterogeneous objectives:} We have shown experimentally that the level of heterogeneity among objectives in a many-objective problem increases with $m$. In particular, for the case of heterogeneous evaluation durations of objectives, with durations being drawn from a Beta distribution, we observed that the difference in the minimum and maximum difference in evaluation durations decreases and increases exponentially with $m$, respectively. A unskewed (symmetric) Beta distribution was associated with the largest increase in the maximum difference. 
	
	\item \textit{Distance between solutions:} We have shown empirically that, as $m$ increases, the expected distance (in the design space) among Pareto optimal solutions becomes more similar to the one among random solutions, and for $m\geq15$ the distances were identical. Due to the spherical distribution of solutions in the objective space for the problem considered, the distance between Pareto optimal and random solutions in the objective space is very similar and increases linearly with $m$. These observations suggest that, for many-objective problems, (i)~it becomes more difficult to discover a high-quality representation of the Pareto front as the Pareto optimal solutions are distant from each other, and (ii)~that few building blocks (if any) might actually be exploited by (blind) recombination. 
	
	\item \textit{Computing hypervolume exactly:} We have shown empirically that existing theoretical results hold and that the exact hypervolume computation time growths indeed exponentially with $m$ in the typical case (not only in the worst case). 
		
	\item \textit{Distance between weight vectors:} We observed empirically that the distance among uniformly-defined weight vectors increases with $m$. 
	A smaller neighborhood size follows the same general trend but reduces the distance between weight vectors to an order of magnitude, with the distance exceeding 0.5 for $m>12$ even for a neighborhood size of 10\%. This implies that, as $m$ increases, the assumption that neighboring sub-problems share similar information becomes less accurate. 
	\end{itemize}
\end{itemize}


\subsection{Recommendations for MOEAs}\label{recommendations}
This section will make use of our theoretical and empirical findings to provide recommendations and considerations on design choices for different multi-objective optimization paradigms when applying them to problems with many objectives.  

\begin{itemize}

\item
For \emph{all classes of MOEAs}:
\begin{itemize}
\item Algorithm performance should be evaluated in terms of representation quality of the Pareto front. This indicates an algorithm's ability to deal efficiently with the large distances between Pareto optimal solutions when $m$ is large. 
\item The distance between solutions in the objective space increases with $m$, causing the quality of representation to decrease and reduce coverage.
\item The level of heterogeneity in a MO problem increases with $m$, urging the need for customized methods for coping with heterogeneity for many-objective problems. 
\item The distance between solutions in the design space increases with $m$, causing (blind) recombination to be less effective.
\item When no external Pareto archive is considered, the population size shall increase to maintain the same level of coverage of the Pareto front because the number of Pareto optimal solutions increases with $m$.  
\item When an external Pareto archive is considered, a data structure such as Quad~Tree or ND-Tree shall be used 
when $m$ is large and in need to discover a good representation quickly.
\end{itemize}

\item
For \emph{dominance-based MOEAs} (e.g. NSGA-II~\citep{deb2002fast}):
\begin{itemize}
\item The dominance relation becomes less discriminative as $m$ increases, causing a reduction in the selection pressure (when selection is done based on dominance) and affect negatively the representation quality of the Pareto front. 
\item The diversity maintenance method employed is critical to ensure adequate selection pressure as $m$ increases, since most individuals in the search space are mutually non-dominated. 
\end{itemize}

\item
For \emph{decomposition/scalarization-based MOEAs} (e.g. MOEA/D~\citep{zhang2007moea}):
\begin{itemize}
\item Assuming a constant number of weight vectors (and population size), all issues mentioned above for constant population sizes hold. Moreover, the distance between weight vectors increases with $m$.
\item Assuming that the number of weight vectors increases with $m$ (in order to maintain the same level of coverage), the algorithm complexity increases with the number of weight vectors. It remains unclear how the number of weight vectors shall increase, e.g. polynomially or exponentially.
\item The complexity of solving an individual scalar problem grows linearly with $m$, and the associated problem can be formulated conveniently using the general non-linear scalarizing functional proposed in this work. 
\end{itemize}

\item
For \emph{indicator-based MOEAs} (e.g. IBEA~\citep{zitzler2004indicator}):
\begin{itemize}
\item Computing the hypervolume of a solution set exactly has a complexity that is exponential in $m$.
\item When using Monte Carlo sampling to approximate the hypervolume combined with a constant population size, all issues mentioned above for constant population sizes hold.
\item When using a constant population size, then do not change the size of the confidence interval when approximating the hypervolume through sampling as $m$ increases. 
\item Account for both the desired accuracy of the hypervolume (width of the confidence interval) and $m$ when deciding whether to compute the hypervolume exactly or approximately. Another interesting approach is to use recently proposed methods combining exact algorithms with Monte Carlo sampling, e.g. \citep{Tang2017, Jaszkiewicz2020}.
\item When assuming a population size that increases with~$m$,
it remains unclear at this stage how the number of sampling points shall be changed (or not) to reach the same level of hypervolume approximation quality.
\end{itemize}

\end{itemize}

\subsection{Future Work}\label{futureWork}
Although this study allowed us to make headwind in terms of both theoretical and empirical contributions to understanding and dealing more efficiently with many-objective optimization problems, there is much more that we as a community can do to advance this 
research area. We identified three main directions for future work to advance our work further that we discuss in the following. 

\vspace{+2mm}
\noindent \textbf{Consider other problem setting and types to verify theoretical results:}\; Our empirical study considered the impact of varying the number of objectives $m$ using a combinatorial (binary) problem but we kept both the dimension of the design space and the level of correlation between objectives constant. It would be important to verify existing and new theoretical findings for other problem settings and other artificial and real problem types, such as continuous and mixed-integer and/or distance-based problems~\citep{koppenetal2005,distanceBasedProblem2019}. Gaining a more profound theoretical and empirical understanding of the impact of heterogeneous objectives~\cite{allmendinger2015delays,eichfelder4} on many-objective optimization is critical too. 

\vspace{+2mm}
\noindent \textbf{Investigate other multi-objective concepts and algorithms:}\; Our empirical study investigated the impact of $m$ on different problem characteristics and multi-objective concepts. Additional empirical studies could investigate if the theoretical findings hold for different multi-objective concepts and algorithms. For example, we considered the scenario where the solutions are processed in a random order, when investigating the complexity of dominance tests and updating the Pareto archive as a function of $m$, while the quality of solutions generated by MOEAs should generally improve with the running time. Thus, future work can investigate the same aspects for more realistic orders of solutions or sequences of solutions generated by real MOEAs. Moreover, it would also be important to critically evaluate the impact of $m$ on other multi-objective concepts, such as non-dominance sorting. Consequently, the findings gained from the empirical analysis could be used to develop more efficient methods for dealing with many-objective problems. This includes also more efficient scalarization methods that account for variable domination structures and special structures of the parameter sets; exploring the use of these methods for problems where the preferences of the decision maker change over time would be very timely given the uncertain environment we are living in.

\vspace{+2mm}
\noindent \textbf{Expand the theoretical study:}\; Extensive contributions can be made on advancing the theoretical grounding of multi-objective concepts and algorithms as a function $m$. For example, our theoretical work can be extended by deriving the probability for a solution to be non-dominated given a random population; the challenge here is to account for the various dependencies between solutions. Knowing this probability would allow to predict the expected number of non-dominated solutions in a population and hence facilitate the design of more efficient initialization strategies for many-objective problems. This line of research has been considered, for example, by~\cite{joshi2014empirical}. Research related to heterogeneous objectives in a many-objective setup is still in its infancy and forms another direction of future research in terms of theory and algorithms. Given a problem with heterogeneous objectives and a portfolio of algorithms, an interesting theoretical and empirical question could be to investigate whether it is possible to use information available about the heterogeneity to predict which algorithm should be selected to solve the problem. This task can also be formulated as an algorithm selection problem~\citep{rice1976algorithm}.

\vspace{+5mm}
\noindent\textbf{Acknowledgments.}
This paper is a product of discussions initiated in the Dagstuhl Seminar 20031: Scalability in Multiobjective Optimization.
This research has been partially supported by the statutory funds of the Faculty of Computing and Telecommunications, Poznan University of Technology,
%
and by the French national research agency under Project ANR-16-CE23-0013-01.

\bibliographystyle{elsarticle-harv}
\bibliography{biblio}

\begin{thebibliography}{74}
\expandafter\ifx\csname natexlab\endcsname\relax\def\natexlab#1{#1}\fi
\providecommand{\url}[1]{\texttt{#1}}
\providecommand{\href}[2]{#2}
\providecommand{\path}[1]{#1}
\providecommand{\DOIprefix}{doi:}
\providecommand{\ArXivprefix}{arXiv:}
\providecommand{\URLprefix}{URL: }
\providecommand{\Pubmedprefix}{pmid:}
\providecommand{\doi}[1]{\href{http://dx.doi.org/#1}{\path{#1}}}
\providecommand{\Pubmed}[1]{\href{pmid:#1}{\path{#1}}}
\providecommand{\bibinfo}[2]{#2}
\ifx\xfnm\relax \def\xfnm[#1]{\unskip,\space#1}\fi
\bibitem[{Aguirre and Tanaka(2007)}]{aguirre2007}
\bibinfo{author}{Aguirre, H.}, \bibinfo{author}{Tanaka, K.},
  \bibinfo{year}{2007}.
\newblock \bibinfo{title}{Working principles, behavior, and performance of
  {MOEAs} on {MNK}-landscapes}.
\newblock \bibinfo{journal}{European Journal of Operational Research}
  \bibinfo{volume}{181}, \bibinfo{pages}{1670--1690}.
\bibitem[{Aguirre(2013)}]{aguirre_many_survey}
\bibinfo{author}{Aguirre, H.E.}, \bibinfo{year}{2013}.
\newblock \bibinfo{title}{Advances on many-objective evolutionary
  optimization}, in: \bibinfo{booktitle}{Proceedings of the 15th Annual
  Conference Companion on Genetic and Evolutionary Computation},
  \bibinfo{publisher}{{ACM}}. pp. \bibinfo{pages}{641--666}.
\bibitem[{Allmendinger et~al.(2017)Allmendinger, Ehrgott, Gandibleux, Geiger,
  Klamroth and Luque}]{allmendinger2017navigation}
\bibinfo{author}{Allmendinger, R.}, \bibinfo{author}{Ehrgott, M.},
  \bibinfo{author}{Gandibleux, X.}, \bibinfo{author}{Geiger, M.J.},
  \bibinfo{author}{Klamroth, K.}, \bibinfo{author}{Luque, M.},
  \bibinfo{year}{2017}.
\newblock \bibinfo{title}{Navigation in multiobjective optimization methods}.
\newblock \bibinfo{journal}{Journal of Multi-Criteria Decision Analysis}
  \bibinfo{volume}{24}, \bibinfo{pages}{57--70}.
\bibitem[{Allmendinger et~al.(2015)Allmendinger, Handl and
  Knowles}]{allmendinger2015delays}
\bibinfo{author}{Allmendinger, R.}, \bibinfo{author}{Handl, J.},
  \bibinfo{author}{Knowles, J.}, \bibinfo{year}{2015}.
\newblock \bibinfo{title}{Multiobjective optimization: When objectives exhibit
  non-uniform latencies}.
\newblock \bibinfo{journal}{European Journal of Operational Research}
  \bibinfo{volume}{243}, \bibinfo{pages}{497--513}.
\bibitem[{Allmendinger and Knowles(2013)}]{allmendinger2013hang}
\bibinfo{author}{Allmendinger, R.}, \bibinfo{author}{Knowles, J.},
  \bibinfo{year}{2013}.
\newblock \bibinfo{title}{`{Hang} on a minute': Investigations on the effects
  of delayed objective functions in multiobjective optimization}, in:
  \bibinfo{booktitle}{International Conference on Evolutionary Multi-Criterion
  Optimization}, \bibinfo{organization}{Springer}. pp. \bibinfo{pages}{6--20}.
\bibitem[{Bader and Zitzler(2011)}]{bader2011}
\bibinfo{author}{Bader, J.}, \bibinfo{author}{Zitzler, E.},
  \bibinfo{year}{2011}.
\newblock \bibinfo{title}{{HypE}: An algorithm for fast hypervolume-based
  many-objective optimization}.
\newblock \bibinfo{journal}{Evolutionary Computation} \bibinfo{volume}{19},
  \bibinfo{pages}{45--76}.
\bibitem[{Bazgan et~al.(2013)Bazgan, Jamain and Vanderpooten}]{BAZGAN2013}
\bibinfo{author}{Bazgan, C.}, \bibinfo{author}{Jamain, F.},
  \bibinfo{author}{Vanderpooten, D.}, \bibinfo{year}{2013}.
\newblock \bibinfo{title}{On the number of non-dominated points of a
  multicriteria optimization problem}.
\newblock \bibinfo{journal}{Discrete Applied Mathematics}
  \bibinfo{volume}{161}, \bibinfo{pages}{2841 -- 2850}.
\bibitem[{Beume et~al.({2009})Beume, Fonseca, Lopez-Ibanez, Paquete and
  Vahrenhold}]{Beume2009}
\bibinfo{author}{Beume, N.}, \bibinfo{author}{Fonseca, C.M.},
  \bibinfo{author}{Lopez-Ibanez, M.}, \bibinfo{author}{Paquete, L.},
  \bibinfo{author}{Vahrenhold, J.}, \bibinfo{year}{{2009}}.
\newblock \bibinfo{title}{{On the Complexity of Computing the Hypervolume
  Indicator}}.
\newblock \bibinfo{journal}{IEEE Transactions on Evolutionary Computation}
  \bibinfo{volume}{{13}}, \bibinfo{pages}{{1075--1082}}.
\bibitem[{Bonnisseau and Cornet(1988)}]{BonCor88}
\bibinfo{author}{Bonnisseau, J.M.}, \bibinfo{author}{Cornet, B.},
  \bibinfo{year}{1988}.
\newblock \bibinfo{title}{Existence of equilibria when firms follow bounded
  losses pricing rules}.
\newblock \bibinfo{journal}{Journal of Mathematical Economics}
  \bibinfo{volume}{17}, \bibinfo{pages}{119--147}.
\bibitem[{Bonnisseau and Crettez(2007)}]{BonCre07}
\bibinfo{author}{Bonnisseau, J.M.}, \bibinfo{author}{Crettez, B.},
  \bibinfo{year}{2007}.
\newblock \bibinfo{title}{On the characterization of efficient production
  vectors}.
\newblock \bibinfo{journal}{Economic Theory} \bibinfo{volume}{31},
  \bibinfo{pages}{213--223}.
\bibitem[{{Chan}(2013)}]{Chan2013}
\bibinfo{author}{{Chan}, T.M.}, \bibinfo{year}{2013}.
\newblock \bibinfo{title}{Klee's measure problem made easy}, in:
  \bibinfo{booktitle}{2013 IEEE 54th Annual Symposium on Foundations of
  Computer Science}, pp. \bibinfo{pages}{410--419}.
\bibitem[{Chand and Wagner(2015)}]{chand2015evolutionary}
\bibinfo{author}{Chand, S.}, \bibinfo{author}{Wagner, M.},
  \bibinfo{year}{2015}.
\newblock \bibinfo{title}{Evolutionary many-objective optimization: A
  quick-start guide}.
\newblock \bibinfo{journal}{Surveys in Operations Research and Management
  Science} \bibinfo{volume}{20}, \bibinfo{pages}{35--42}.
\bibitem[{Chankong and Haimes(1983)}]{ChaHai83}
\bibinfo{author}{Chankong, V.}, \bibinfo{author}{Haimes, Y.Y.},
  \bibinfo{year}{1983}.
\newblock \bibinfo{title}{Multiobjective decision making}.
  volume~\bibinfo{volume}{8} of \textit{\bibinfo{series}{North-Holland Series
  in System Science and Engineering}}.
\newblock \bibinfo{publisher}{North-Holland Publishing Co., New York}.
\bibitem[{Chugh et~al.(2018)Chugh, Allmendinger, Ojalehto and
  Miettinen}]{chugh2018surrogate}
\bibinfo{author}{Chugh, T.}, \bibinfo{author}{Allmendinger, R.},
  \bibinfo{author}{Ojalehto, V.}, \bibinfo{author}{Miettinen, K.},
  \bibinfo{year}{2018}.
\newblock \bibinfo{title}{Surrogate-assisted evolutionary biobjective
  optimization for objectives with non-uniform latencies}, in:
  \bibinfo{booktitle}{Proceedings of the 20th Annual Conference on Genetic and
  Evolutionary Computation}, \bibinfo{publisher}{ACM}. pp.
  \bibinfo{pages}{609--616}.
\bibitem[{Coello et~al.(2007)Coello, Lamont, Van~Veldhuizen
  et~al.}]{coello2007evolutionary}
\bibinfo{author}{Coello, C.A.C.}, \bibinfo{author}{Lamont, G.B.},
  \bibinfo{author}{Van~Veldhuizen, D.A.}, et~al., \bibinfo{year}{2007}.
\newblock \bibinfo{title}{Evolutionary algorithms for solving multi-objective
  problems}. volume~\bibinfo{volume}{5}.
\newblock \bibinfo{publisher}{Springer}.
\bibitem[{Daolio et~al.(2017)Daolio, Liefooghe, Verel, Aguirre and
  Tanaka}]{daolio2017}
\bibinfo{author}{Daolio, F.}, \bibinfo{author}{Liefooghe, A.},
  \bibinfo{author}{Verel, S.}, \bibinfo{author}{Aguirre, H.},
  \bibinfo{author}{Tanaka, K.}, \bibinfo{year}{2017}.
\newblock \bibinfo{title}{Problem features versus algorithm performance on
  rugged multiobjective combinatorial fitness landscapes}.
\newblock \bibinfo{journal}{Evolutionary Computation} \bibinfo{volume}{25},
  \bibinfo{pages}{555--585}.
\bibitem[{Deb(2001)}]{deb2001multi}
\bibinfo{author}{Deb, K.}, \bibinfo{year}{2001}.
\newblock \bibinfo{title}{Multi-objective optimization using evolutionary
  algorithms}.
\newblock \bibinfo{publisher}{John Wiley \& Sons}.
\bibitem[{Deb et~al.(2002)Deb, Pratap, Agarwal and Meyarivan}]{deb2002fast}
\bibinfo{author}{Deb, K.}, \bibinfo{author}{Pratap, A.},
  \bibinfo{author}{Agarwal, S.}, \bibinfo{author}{Meyarivan, T.},
  \bibinfo{year}{2002}.
\newblock \bibinfo{title}{A fast and elitist multiobjective genetic algorithm:
  {NSGA-II}}.
\newblock \bibinfo{journal}{IEEE Transactions on Evolutionary Computation}
  \bibinfo{volume}{6}, \bibinfo{pages}{182--197}.
\bibitem[{Drozd\'ik et~al.(2015)Drozd\'ik, Akimoto, Aguirre and
  Tanaka}]{Drozdik2015}
\bibinfo{author}{Drozd\'ik, M.}, \bibinfo{author}{Akimoto, Y.},
  \bibinfo{author}{Aguirre, H.}, \bibinfo{author}{Tanaka, K.},
  \bibinfo{year}{2015}.
\newblock \bibinfo{title}{Computational cost reduction of nondominated sorting
  using the {M}-front}.
\newblock \bibinfo{journal}{IEEE Transactions on Evolutionary Computation}
  \bibinfo{volume}{19}, \bibinfo{pages}{659--678}.
\bibitem[{Durea et~al.(2017)Durea, Strugariu and Tammer}]{DurStrTam2017}
\bibinfo{author}{Durea, M.}, \bibinfo{author}{Strugariu, R.},
  \bibinfo{author}{Tammer, C.}, \bibinfo{year}{2017}.
\newblock \bibinfo{title}{On some methods to derive necessary and sufficient
  optimality conditions in vector optimization}.
\newblock \bibinfo{journal}{J. Optim. Theory Appl.} \bibinfo{volume}{175},
  \bibinfo{pages}{738--763}.
\bibitem[{Eichfelder(2008)}]{eich08}
\bibinfo{author}{Eichfelder, G.}, \bibinfo{year}{2008}.
\newblock \bibinfo{title}{Adaptive scalarization methods in multiobjective
  optimization}.
\newblock Vector Optimization, \bibinfo{publisher}{Springer-Verlag},
  \bibinfo{address}{Berlin}.
\bibitem[{Eichfelder et~al.(2015)Eichfelder, Gandibleux, Geiger, Jahn,
  Jaszkiewicz, Knowles, Shukla, Trautmann and Wessing}]{eichfelder4}
\bibinfo{author}{Eichfelder, G.}, \bibinfo{author}{Gandibleux, X.},
  \bibinfo{author}{Geiger, M.J.}, \bibinfo{author}{Jahn, J.},
  \bibinfo{author}{Jaszkiewicz, A.}, \bibinfo{author}{Knowles, J.D.},
  \bibinfo{author}{Shukla, P.K.}, \bibinfo{author}{Trautmann, H.},
  \bibinfo{author}{Wessing, S.}, \bibinfo{year}{2015}.
\newblock \bibinfo{title}{{Heterogeneous Functions (WG3)}}, in:
  \bibinfo{booktitle}{{Understanding Complexity in Multiobjective Optimization
  : report from Dagstuhl Seminar 15031 / Greco, Salvatore; Klamroth, Kathrin;
  Knowles, Joshua D.; Rudolph, Günter. --- Wadern : Schloss Dagstuhl. Vol.
  5}}. \bibinfo{publisher}{Dagstuhl Zentrum f\"ur Informatik}, pp.
  \bibinfo{pages}{121--129}.
\bibitem[{Eikelboom et~al.(2015)Eikelboom, Janssen and
  Stewart}]{eikelboom2015spatial}
\bibinfo{author}{Eikelboom, T.}, \bibinfo{author}{Janssen, R.},
  \bibinfo{author}{Stewart, T.J.}, \bibinfo{year}{2015}.
\newblock \bibinfo{title}{A spatial optimization algorithm for geodesign}.
\newblock \bibinfo{journal}{Landscape and Urban Planning}
  \bibinfo{volume}{144}, \bibinfo{pages}{10--21}.
\bibitem[{Fieldsend(2020)}]{Fieldsend2020}
\bibinfo{author}{Fieldsend, J.E.}, \bibinfo{year}{2020}.
\newblock \bibinfo{title}{Data structures for non-dominated sets:
  Implementations and empirical assessment of two decades of advances}, in:
  \bibinfo{booktitle}{Proceedings of the 22nd Annual Conference on Genetic and
  Evolutionary Computation}, \bibinfo{publisher}{AMC}. p.
  \bibinfo{pages}{489–497}.
\bibitem[{Fieldsend et~al.(2019)Fieldsend, Chugh, Allmendinger and
  Miettinen}]{distanceBasedProblem2019}
\bibinfo{author}{Fieldsend, J.E.}, \bibinfo{author}{Chugh, T.},
  \bibinfo{author}{Allmendinger, R.}, \bibinfo{author}{Miettinen, K.},
  \bibinfo{year}{2019}.
\newblock \bibinfo{title}{A feature rich distance-based many-objective
  visualisable test problem generator}, in: \bibinfo{booktitle}{Proceedings of
  the 21st Annual Conference on Genetic and Evolutionary Computation},
  \bibinfo{publisher}{ACM}. p. \bibinfo{pages}{541–549}.
\bibitem[{{Fieldsend} et~al.(2003){Fieldsend}, {Everson} and
  {Singh}}]{Fieldsend2003}
\bibinfo{author}{{Fieldsend}, J.E.}, \bibinfo{author}{{Everson}, R.M.},
  \bibinfo{author}{{Singh}, S.}, \bibinfo{year}{2003}.
\newblock \bibinfo{title}{Using unconstrained elite archives for multiobjective
  optimization}.
\newblock \bibinfo{journal}{IEEE Transactions on Evolutionary Computation}
  \bibinfo{volume}{7}, \bibinfo{pages}{305--323}.
\bibitem[{Figueira et~al.(2017)Figueira, Fonseca, Halffmann, Klamroth, Paquete,
  Ruzika, Schulze, Stiglmayr and Willems}]{Fig2017}
\bibinfo{author}{Figueira, J.R.}, \bibinfo{author}{Fonseca, C.M.},
  \bibinfo{author}{Halffmann, P.}, \bibinfo{author}{Klamroth, K.},
  \bibinfo{author}{Paquete, L.}, \bibinfo{author}{Ruzika, S.},
  \bibinfo{author}{Schulze, B.}, \bibinfo{author}{Stiglmayr, M.},
  \bibinfo{author}{Willems, D.}, \bibinfo{year}{2017}.
\newblock \bibinfo{title}{Easy to say they're hard, but hard to see they're
  easy-towards a categorization of tractable multiobjective combinatorial
  optimization problems}.
\newblock \bibinfo{journal}{Journal of Multi-Criteria Decision Analysis}
  \bibinfo{volume}{24}, \bibinfo{pages}{82--98}.
\bibitem[{Fleming et~al.(2005)Fleming, Purshouse and Lygoe}]{fleming2005many}
\bibinfo{author}{Fleming, P.J.}, \bibinfo{author}{Purshouse, R.C.},
  \bibinfo{author}{Lygoe, R.J.}, \bibinfo{year}{2005}.
\newblock \bibinfo{title}{Many-objective optimization: An engineering design
  perspective}, in: \bibinfo{booktitle}{International Conference on
  Evolutionary Multi-Criterion Optimization}, \bibinfo{organization}{Springer}.
  pp. \bibinfo{pages}{14--32}.
\bibitem[{Glasmachers(2017)}]{Glasmachers2017}
\bibinfo{author}{Glasmachers, T.}, \bibinfo{year}{2017}.
\newblock \bibinfo{title}{A fast incremental {BSP} tree archive for
  non-dominated points}, in: \bibinfo{booktitle}{International Conference on
  Evolutionary Multi-Criterion Optimization}, \bibinfo{publisher}{Springer}.
  pp. \bibinfo{pages}{252--266}.
\bibitem[{G{\"o}pfert et~al.(2003)G{\"o}pfert, Riahi, Tammer and
  Z{\u{a}}linescu}]{GopRiaTamZal:03}
\bibinfo{author}{G{\"o}pfert, A.}, \bibinfo{author}{Riahi, H.},
  \bibinfo{author}{Tammer, C.}, \bibinfo{author}{Z{\u{a}}linescu, C.},
  \bibinfo{year}{2003}.
\newblock \bibinfo{title}{Variational methods in partially ordered spaces}.
\newblock CMS Books in Mathematics/Ouvrages de Math\'ematiques de la SMC, 17,
  \bibinfo{publisher}{Springer-Verlag}, \bibinfo{address}{New York}.
\bibitem[{{Guerreiro} and {Fonseca}(2018)}]{Guerreiro2018}
\bibinfo{author}{{Guerreiro}, A.P.}, \bibinfo{author}{{Fonseca}, C.M.},
  \bibinfo{year}{2018}.
\newblock \bibinfo{title}{Computing and updating hypervolume contributions in
  up to four dimensions}.
\newblock \bibinfo{journal}{IEEE Transactions on Evolutionary Computation}
  \bibinfo{volume}{22}, \bibinfo{pages}{449--463}.
\bibitem[{Haimes et~al.(1971)Haimes, Lasdon and Wismer}]{Hai71}
\bibinfo{author}{Haimes, Y.}, \bibinfo{author}{Lasdon, L.S.},
  \bibinfo{author}{Wismer, D.A.}, \bibinfo{year}{1971}.
\newblock \bibinfo{title}{On a bicriterion formulation of the problems of
  integrated system identification and system optimization}.
\newblock \bibinfo{journal}{IEEE Transactions on Systems, Man, and Cybernetics}
  \bibinfo{volume}{1}, \bibinfo{pages}{296--297}.
\bibitem[{Hughes(2005)}]{hughes2005evolutionary}
\bibinfo{author}{Hughes, E.J.}, \bibinfo{year}{2005}.
\newblock \bibinfo{title}{Evolutionary many-objective optimisation: Many once
  or one many?}, in: \bibinfo{booktitle}{Congress on Evolutionary Computation},
  \bibinfo{organization}{IEEE}. pp. \bibinfo{pages}{222--227}.
\bibitem[{Ishibuchi et~al.(2008)Ishibuchi, Tsukamoto and
  Nojima}]{ishibuchi2008evolutionary}
\bibinfo{author}{Ishibuchi, H.}, \bibinfo{author}{Tsukamoto, N.},
  \bibinfo{author}{Nojima, Y.}, \bibinfo{year}{2008}.
\newblock \bibinfo{title}{Evolutionary many-objective optimization: A short
  review}, in: \bibinfo{booktitle}{Congress on Evolutionary Computation},
  \bibinfo{organization}{IEEE}. pp. \bibinfo{pages}{2419--2426}.
\bibitem[{Jaszkiewicz(2018)}]{JASZKIEWICZ2018b}
\bibinfo{author}{Jaszkiewicz, A.}, \bibinfo{year}{2018}.
\newblock \bibinfo{title}{Improved quick hypervolume algorithm}.
\newblock \bibinfo{journal}{Computers \& Operations Research}
  \bibinfo{volume}{90}, \bibinfo{pages}{72 -- 83}.
\bibitem[{{Jaszkiewicz} and {Lust}(2018)}]{Jaszkiewicz2018}
\bibinfo{author}{{Jaszkiewicz}, A.}, \bibinfo{author}{{Lust}, T.},
  \bibinfo{year}{2018}.
\newblock \bibinfo{title}{{ND-Tree}-based update: A fast algorithm for the
  dynamic nondominance problem}.
\newblock \bibinfo{journal}{IEEE Transactions on Evolutionary Computation}
  \bibinfo{volume}{22}, \bibinfo{pages}{778--791}.
\bibitem[{Jaszkiewicz et~al.(2020)Jaszkiewicz, Susmaga and
  Zielniewicz}]{Jaszkiewicz2020}
\bibinfo{author}{Jaszkiewicz, A.}, \bibinfo{author}{Susmaga, R.},
  \bibinfo{author}{Zielniewicz, P.}, \bibinfo{year}{2020}.
\newblock \bibinfo{title}{Approximate hypervolume calculation with guaranteed
  or confidence bounds}, in: \bibinfo{booktitle}{International Conference on
  Parallel Problem Solving from Nature}, \bibinfo{publisher}{Springer}. pp.
  \bibinfo{pages}{215--228}.
\bibitem[{Joshi and Deshpande(2014)}]{joshi2014empirical}
\bibinfo{author}{Joshi, R.}, \bibinfo{author}{Deshpande, B.},
  \bibinfo{year}{2014}.
\newblock \bibinfo{title}{Empirical and analytical study of many-objective
  optimization problems: analysing distribution of nondominated solutions and
  population size for scalability of randomized heuristics}.
\newblock \bibinfo{journal}{Memetic Computing} \bibinfo{volume}{6},
  \bibinfo{pages}{133--145}.
\bibitem[{Kaliszewski(1994)}]{Kal94}
\bibinfo{author}{Kaliszewski, I.}, \bibinfo{year}{1994}.
\newblock \bibinfo{title}{Quantitative {P}areto analysis by cone separation
  technique}.
\newblock \bibinfo{publisher}{Kluwer Academic Publishers, Boston, MA}.
\bibitem[{Kauffman(1993)}]{kauffman1993}
\bibinfo{author}{Kauffman, S.A.}, \bibinfo{year}{1993}.
\newblock \bibinfo{title}{The Origins of Order}.
\newblock \bibinfo{publisher}{Oxford University Press}.
\bibitem[{Kollat et~al.(2011)Kollat, Reed and Maxwell}]{kollat2011many}
\bibinfo{author}{Kollat, J.B.}, \bibinfo{author}{Reed, P.M.},
  \bibinfo{author}{Maxwell, R.}, \bibinfo{year}{2011}.
\newblock \bibinfo{title}{Many-objective groundwater monitoring network design
  using bias-aware ensemble kalman filtering, evolutionary optimization, and
  visual analytics}.
\newblock \bibinfo{journal}{Water Resources Research} \bibinfo{volume}{47}.
\bibitem[{K{\"o}ppen et~al.(2005)K{\"o}ppen, Vicente-Garcia and
  Nickolay}]{koppenetal2005}
\bibinfo{author}{K{\"o}ppen, M.}, \bibinfo{author}{Vicente-Garcia, R.},
  \bibinfo{author}{Nickolay, B.}, \bibinfo{year}{2005}.
\newblock \bibinfo{title}{Fuzzy-{P}areto-dominance and its application in
  evolutionary multi-objective optimization}, in:
  \bibinfo{booktitle}{International Conference on Evolutionary Multi-Criterion
  Optimization}, \bibinfo{publisher}{Springer}. pp. \bibinfo{pages}{399--412}.
\bibitem[{Lacour et~al.(2017)Lacour, Klamroth and Fonseca}]{LACOUR2017347}
\bibinfo{author}{Lacour, R.}, \bibinfo{author}{Klamroth, K.},
  \bibinfo{author}{Fonseca, C.M.}, \bibinfo{year}{2017}.
\newblock \bibinfo{title}{A box decomposition algorithm to compute the
  hypervolume indicator}.
\newblock \bibinfo{journal}{Computers \& Operations Research}
  \bibinfo{volume}{79}, \bibinfo{pages}{347 -- 360}.
\bibitem[{{Liefooghe} et~al.(2019){Liefooghe}, {Daolio}, {Verel}, {Derbel},
  {Aguirre} and {Tanaka}}]{8832171}
\bibinfo{author}{{Liefooghe}, A.}, \bibinfo{author}{{Daolio}, F.},
  \bibinfo{author}{{Verel}, S.}, \bibinfo{author}{{Derbel}, B.},
  \bibinfo{author}{{Aguirre}, H.}, \bibinfo{author}{{Tanaka}, K.},
  \bibinfo{year}{2019}.
\newblock \bibinfo{title}{Landscape-aware performance prediction for
  evolutionary multi-objective optimization}.
\newblock \bibinfo{journal}{IEEE Transactions on Evolutionary Computation} ,
  \bibinfo{pages}{(early access)}.
\bibitem[{Luenberger(1992a)}]{luen92b}
\bibinfo{author}{Luenberger, D.G.}, \bibinfo{year}{1992}a.
\newblock \bibinfo{title}{Benefit functions and duality}.
\newblock \bibinfo{journal}{Journal of Mathematical Economics}
  \bibinfo{volume}{21}, \bibinfo{pages}{461--481}.
\bibitem[{Luenberger(1992b)}]{luen92}
\bibinfo{author}{Luenberger, D.G.}, \bibinfo{year}{1992}b.
\newblock \bibinfo{title}{New optimality principles for economic efficiency and
  equilibrium}.
\newblock \bibinfo{journal}{J. Optim. Theory Appl.} \bibinfo{volume}{75},
  \bibinfo{pages}{221--264}.
\newblock \URLprefix
  \url{http://dx.doi.org.qe2a-proxy.mun.ca/10.1007/BF00941466},
  \DOIprefix\doi{10.1007/BF00941466}.
\bibitem[{Malcolm et~al.(1959)Malcolm, Roseboom, Clark and Fazar}]{PERT_1959}
\bibinfo{author}{Malcolm, D.G.}, \bibinfo{author}{Roseboom, J.H.},
  \bibinfo{author}{Clark, C.E.}, \bibinfo{author}{Fazar, W.},
  \bibinfo{year}{1959}.
\newblock \bibinfo{title}{Application of a technique for research and
  development program evaluation}.
\newblock \bibinfo{journal}{Operations Research} \bibinfo{volume}{7},
  \bibinfo{pages}{646--669}.
\bibitem[{Miettinen(2012)}]{miettinen2012nonlinear}
\bibinfo{author}{Miettinen, K.}, \bibinfo{year}{2012}.
\newblock \bibinfo{title}{Nonlinear multiobjective optimization}.
  volume~\bibinfo{volume}{12}.
\newblock \bibinfo{publisher}{Springer Science \& Business Media}.
\bibitem[{Mostaghim and Teich(2005)}]{Mostaghim2005}
\bibinfo{author}{Mostaghim, S.}, \bibinfo{author}{Teich, J.},
  \bibinfo{year}{2005}.
\newblock \bibinfo{title}{Quad-trees: A data structure for storing Pareto sets
  in multiobjective evolutionary algorithms with elitism}.
  \bibinfo{publisher}{Springer London}, \bibinfo{address}{London}.
\newblock pp. \bibinfo{pages}{81--104}.
\bibitem[{Newcombe({1998})}]{Newcombe1998}
\bibinfo{author}{Newcombe, R.}, \bibinfo{year}{{1998}}.
\newblock \bibinfo{title}{{Two-sided confidence intervals for the single
  proportion: Comparison of seven methods}}.
\newblock \bibinfo{journal}{Statistics in Medicine} \bibinfo{volume}{{17}},
  \bibinfo{pages}{{857--872}}.
\bibitem[{Pascoletti and Serafini(1984)}]{PasSer84}
\bibinfo{author}{Pascoletti, A.}, \bibinfo{author}{Serafini, P.},
  \bibinfo{year}{1984}.
\newblock \bibinfo{title}{Scalarizing vector optimization problems}.
\newblock \bibinfo{journal}{Journal of Optimization Theory and Applications}
  \bibinfo{volume}{42}, \bibinfo{pages}{499--524}.
\bibitem[{Polak(1976)}]{Pol76}
\bibinfo{author}{Polak, E.}, \bibinfo{year}{1976}.
\newblock \bibinfo{title}{On the approximation of solutions to multiple
  criteria decision making problems}, in: \bibinfo{editor}{Zeleny, M.} (Ed.),
  \bibinfo{booktitle}{Multiple Criteria Decision Making}.
  \bibinfo{publisher}{Springer}, \bibinfo{address}{Berlin}, pp.
  \bibinfo{pages}{271--282}.
\bibitem[{Purshouse and Fleming(2003)}]{purshouse2003evolutionary}
\bibinfo{author}{Purshouse, R.C.}, \bibinfo{author}{Fleming, P.J.},
  \bibinfo{year}{2003}.
\newblock \bibinfo{title}{Evolutionary many-objective optimisation: An
  exploratory analysis}, in: \bibinfo{booktitle}{Congress on Evolutionary
  Computation}, \bibinfo{organization}{IEEE}. pp. \bibinfo{pages}{2066--2073}.
\bibitem[{Rice(1976)}]{rice1976algorithm}
\bibinfo{author}{Rice, J.R.}, \bibinfo{year}{1976}.
\newblock \bibinfo{title}{The algorithm selection problem}.
\newblock \bibinfo{journal}{Advances in Computers} \bibinfo{volume}{15},
  \bibinfo{pages}{65--118}.
\bibitem[{Saaty(1987)}]{saaty1987analytic}
\bibinfo{author}{Saaty, R.W.}, \bibinfo{year}{1987}.
\newblock \bibinfo{title}{The analytic hierarchy process—what it is and how
  it is used}.
\newblock \bibinfo{journal}{Mathematical Modelling} \bibinfo{volume}{9},
  \bibinfo{pages}{161--176}.
\bibitem[{Scheffé(1958)}]{simplex_lattice}
\bibinfo{author}{Scheffé, H.}, \bibinfo{year}{1958}.
\newblock \bibinfo{title}{Experiments with mixtures}.
\newblock \bibinfo{journal}{Journal of the Royal Statistical Society: Series B
  (Methodological)} \bibinfo{volume}{20}, \bibinfo{pages}{344--360}.
\bibitem[{Sch\"onfeld(1970)}]{sch70}
\bibinfo{author}{Sch\"onfeld, P.}, \bibinfo{year}{1970}.
\newblock \bibinfo{title}{Some duality theorems for the non-linear vector
  maximum problem}.
\newblock \bibinfo{journal}{Unternehmensforschung} \bibinfo{volume}{14},
  \bibinfo{pages}{51--63}.
\bibitem[{Small et~al.(2011)Small, McColl, Allmendinger, Pahle,
  L{\'o}pez-Castej{\'o}n, Rothwell, Knowles, Mendes, Brough and
  Kell}]{small2011efficient}
\bibinfo{author}{Small, B.G.}, \bibinfo{author}{McColl, B.W.},
  \bibinfo{author}{Allmendinger, R.}, \bibinfo{author}{Pahle, J.},
  \bibinfo{author}{L{\'o}pez-Castej{\'o}n, G.}, \bibinfo{author}{Rothwell,
  N.J.}, \bibinfo{author}{Knowles, J.}, \bibinfo{author}{Mendes, P.},
  \bibinfo{author}{Brough, D.}, \bibinfo{author}{Kell, D.B.},
  \bibinfo{year}{2011}.
\newblock \bibinfo{title}{Efficient discovery of anti-inflammatory
  small-molecule combinations using evolutionary computing}.
\newblock \bibinfo{journal}{Nature Chemical Biology} \bibinfo{volume}{7},
  \bibinfo{pages}{902}.
\bibitem[{Sun({2006})}]{Sun2006}
\bibinfo{author}{Sun, M.}, \bibinfo{year}{{2006}}.
\newblock \bibinfo{title}{{A primogenitary linked quad tree data structure and
  its application to discrete multiple criteria optimization}}.
\newblock \bibinfo{journal}{Annals of Operations Research}
  \bibinfo{volume}{{147}}, \bibinfo{pages}{{87--107}}.
\bibitem[{Sun({2011})}]{Sun2011}
\bibinfo{author}{Sun, M.}, \bibinfo{year}{{2011}}.
\newblock \bibinfo{title}{{A primogenitary linked quad tree approach for
  solution storage and retrieval in heuristic binary optimization}}.
\newblock \bibinfo{journal}{European Journal of Operations Research}
  \bibinfo{volume}{{209}}, \bibinfo{pages}{{228--240}}.
\bibitem[{Sun and Steuer(1996)}]{Sun1996}
\bibinfo{author}{Sun, M.}, \bibinfo{author}{Steuer, R.E.},
  \bibinfo{year}{1996}.
\newblock \bibinfo{title}{Quad-trees and linear lists for identifying
  nondominated criterion vectors}.
\newblock \bibinfo{journal}{INFORMS Journal on Computing} \bibinfo{volume}{8},
  \bibinfo{pages}{367--375}.
\bibitem[{Tammer and Winkler(2003)}]{TamWin03}
\bibinfo{author}{Tammer, C.}, \bibinfo{author}{Winkler, K.},
  \bibinfo{year}{2003}.
\newblock \bibinfo{title}{A new scalarization approach and applications in
  multicriteria d.c.\ optimization}.
\newblock \bibinfo{journal}{Journal of Nonlinear and Convex Analysis}
  \bibinfo{volume}{4}, \bibinfo{pages}{365--380}.
\bibitem[{Tang et~al.(2017)Tang, Liu and Chen}]{Tang2017}
\bibinfo{author}{Tang, W.}, \bibinfo{author}{Liu, H.}, \bibinfo{author}{Chen,
  L.}, \bibinfo{year}{2017}.
\newblock \bibinfo{title}{A fast approximate hypervolume calculation method by
  a novel decomposition strategy}, in: \bibinfo{editor}{Huang, D.S.},
  \bibinfo{editor}{Bevilacqua, V.}, \bibinfo{editor}{Premaratne, P.},
  \bibinfo{editor}{Gupta, P.} (Eds.), \bibinfo{booktitle}{Intelligent Computing
  Theories and Application}, \bibinfo{publisher}{Springer International
  Publishing}, \bibinfo{address}{Cham}. pp. \bibinfo{pages}{14--25}.
\bibitem[{Terzijska et~al.(2014)Terzijska, Porcelli and
  Eichfelder}]{terzijska2014}
\bibinfo{author}{Terzijska, D.}, \bibinfo{author}{Porcelli, M.},
  \bibinfo{author}{Eichfelder, G.}, \bibinfo{year}{2014}.
\newblock \bibinfo{title}{Multi-objective optimization in the lorentz force
  velocimetry framework}, in: \bibinfo{booktitle}{International Workshop on
  Optimization and Inverse Problems in Electromagnetism}, pp.
  \bibinfo{pages}{81--82}.
\bibitem[{Thomann(2019)}]{thomann2019trust_PhD}
\bibinfo{author}{Thomann, J.}, \bibinfo{year}{2019}.
\newblock \bibinfo{title}{A trust region approach for multi-objective
  heterogeneous optimization}.
\newblock Ph.D. thesis. Technische Universit\"{a}t Ilmenau.
  \bibinfo{address}{Ilmenau, Germany}.
\bibitem[{Trivedi et~al.(2017)Trivedi, Srinivasan, Sanyal and
  Ghosh}]{decompo_survey}
\bibinfo{author}{Trivedi, A.}, \bibinfo{author}{Srinivasan, D.},
  \bibinfo{author}{Sanyal, K.}, \bibinfo{author}{Ghosh, A.},
  \bibinfo{year}{2017}.
\newblock \bibinfo{title}{A survey of multiobjective evolutionary algorithms
  based on decomposition}.
\newblock \bibinfo{journal}{IEEE Transactions on Evolutionary Computation}
  \bibinfo{volume}{21}, \bibinfo{pages}{440--462}.
\bibitem[{Tu{\v{s}}ar and Filipi{\v{c}}(2015)}]{tuvsar2015visualization}
\bibinfo{author}{Tu{\v{s}}ar, T.}, \bibinfo{author}{Filipi{\v{c}}, B.},
  \bibinfo{year}{2015}.
\newblock \bibinfo{title}{Visualization of {P}areto front approximations in
  evolutionary multiobjective optimization: A critical review and the
  prosection method}.
\newblock \bibinfo{journal}{IEEE Transactions on Evolutionary Computation}
  \bibinfo{volume}{19}, \bibinfo{pages}{225--245}.
\bibitem[{Verel et~al.(2013)Verel, Liefooghe, Jourdan and Dhaenens}]{verel2013}
\bibinfo{author}{Verel, S.}, \bibinfo{author}{Liefooghe, A.},
  \bibinfo{author}{Jourdan, L.}, \bibinfo{author}{Dhaenens, C.},
  \bibinfo{year}{2013}.
\newblock \bibinfo{title}{On the structure of multiobjective combinatorial
  search space: {MNK}-landscapes with correlated objectives}.
\newblock \bibinfo{journal}{European Journal of Operational Research}
  \bibinfo{volume}{227}, \bibinfo{pages}{331--342}.
\newblock \DOIprefix\doi{10.1016/j.ejor.2012.12.019}.
\bibitem[{Wagner et~al.(2007)Wagner, Beume and Naujoks}]{wagner2007pareto}
\bibinfo{author}{Wagner, T.}, \bibinfo{author}{Beume, N.},
  \bibinfo{author}{Naujoks, B.}, \bibinfo{year}{2007}.
\newblock \bibinfo{title}{Pareto-, aggregation-, and indicator-based methods in
  many-objective optimization}, in: \bibinfo{booktitle}{International
  Conference on Evolutionary Multi-Criterion Optimization},
  \bibinfo{organization}{Springer}. pp. \bibinfo{pages}{742--756}.
\bibitem[{Wang et~al.(2020)Wang, Jin, Schmitt and Olhofer}]{xilu2020}
\bibinfo{author}{Wang, X.}, \bibinfo{author}{Jin, Y.},
  \bibinfo{author}{Schmitt, S.}, \bibinfo{author}{Olhofer, M.},
  \bibinfo{year}{2020}.
\newblock \bibinfo{title}{Transfer learning for gaussian process assisted
  evolutionary bi-objective optimization for objectives with different
  evaluation times}, in: \bibinfo{booktitle}{Proceedings of the 22nd Annual
  Conference on Genetic and Evolutionary Computation},
  \bibinfo{publisher}{ACM}. p. \bibinfo{pages}{TBC}.
\bibitem[{Wolpert and Macready(1997)}]{wolpert1997no}
\bibinfo{author}{Wolpert, D.H.}, \bibinfo{author}{Macready, W.G.},
  \bibinfo{year}{1997}.
\newblock \bibinfo{title}{No free lunch theorems for optimization}.
\newblock \bibinfo{journal}{IEEE Transactions on Evolutionary Computation}
  \bibinfo{volume}{1}, \bibinfo{pages}{67--82}.
\bibitem[{Zhang and Li(2007)}]{zhang2007moea}
\bibinfo{author}{Zhang, Q.}, \bibinfo{author}{Li, H.}, \bibinfo{year}{2007}.
\newblock \bibinfo{title}{{MOEA/D}: A multiobjective evolutionary algorithm
  based on decomposition}.
\newblock \bibinfo{journal}{IEEE Transactions on Evolutionary Computation}
  \bibinfo{volume}{11}, \bibinfo{pages}{712--731}.
\bibitem[{Zitzler and K{\"u}nzli(2004)}]{zitzler2004indicator}
\bibinfo{author}{Zitzler, E.}, \bibinfo{author}{K{\"u}nzli, S.},
  \bibinfo{year}{2004}.
\newblock \bibinfo{title}{Indicator-based selection in multiobjective search},
  in: \bibinfo{booktitle}{International Conference on Parallel Problem Solving
  from Nature}, \bibinfo{organization}{Springer}. pp.
  \bibinfo{pages}{832--842}.
\bibitem[{Zitzler et~al.(2003)Zitzler, Thiele, Laumanns, Fonseca and {Grunert
  da Fonseca}}]{Zitzler2003}
\bibinfo{author}{Zitzler, E.}, \bibinfo{author}{Thiele, L.},
  \bibinfo{author}{Laumanns, M.}, \bibinfo{author}{Fonseca, C.M.},
  \bibinfo{author}{{Grunert da Fonseca}, V.}, \bibinfo{year}{2003}.
\newblock \bibinfo{title}{Performance assessment of multiobjective optimizers:
  An analysis and review}.
\newblock \bibinfo{journal}{IEEE Transactions on Evolutionary Computation}
  \bibinfo{volume}{7}, \bibinfo{pages}{117--132}.

\end{thebibliography}

\end{document}